\documentclass[11pt]{article}

\usepackage{times}

\newcommand{\vs}{}
\newcommand{\cnt}{k}
\newcommand{\pr}{{\rm prox}}
\newcommand{\eg}{\emph{e.g.}}

\renewcommand\epsilon\varepsilon
\newcommand{\smthpara}{\kappa} 
\newcommand{\smthparacvx}{\kappa_{\text{\rm cvx}}} 
\newcommand{\smthparainit}{\kappa_0} 
\newcommand{\env}{f_{\smthpara}}
\newcommand{\accx}{\tilde{x}} 
\newcommand{\proxx}{\bar{x}} 
                         
\newcommand{\weakcnx}{\rho} 
\newcommand{\argmin}{\operatornamewithlimits{argmin}} 

\newcommand{\ip}[1] {\langle #1 \rangle } 
\newcommand{\norm}[1] {\left \| #1 \right \|} 
\newcommand{\dist}{{\rm dist}} 
\newcommand{\R}{{\mathbb R}} 
\newcommand{\mtd}{{\mathcal M}} 
\newcommand{\proxi}{\text{prox}} 
\newcommand{\oR}{\overline \R} 

\newcommand{{\newalgosp}}{Basic 4WD-Catalyst~}
\newcommand{{\newalgo}}{Basic 4WD-Catalyst}
\newcommand{{\autonewalgosp}}{4WD-Catalyst~}
\newcommand{{\autonewalgo}}{4WD-Catalyst}
\newcommand{\linesearch}{{Auto-adapt}}

\newcommand{\mylabel}[2]{#2\def\@currentlabel{#2}\label{#1}}

\def\defin{\triangleq}
\usepackage{color}
\usepackage{graphicx} 
\usepackage{subfigure} 
\usepackage{enumitem}

\usepackage{amsmath}
\usepackage{amsthm}
\usepackage{amssymb}
\usepackage{array}
\usepackage{multirow}
\usepackage{bm}

\usepackage{tablefootnote}

\newtheorem{theorem}{Theorem}[section]
\newtheorem{proposition}[theorem]{Proposition}
\newtheorem{lemma}[theorem]{Lemma}
\newtheorem{remark}[theorem]{Remark}
\newtheorem{coro}[theorem]{Corollary}
\newtheorem{defn}[theorem]{Definition}

\usepackage[sort, numbers]{natbib}

\usepackage[letterpaper, margin=1.0in]{geometry}

\usepackage{algorithm}
\usepackage{algorithmic}
\usepackage{amsfonts}
\usepackage{hyperref}

\usepackage{tikz}
\usetikzlibrary{shapes}
\usetikzlibrary{decorations.markings}
\usepackage{pgfplots}

\newcommand{\TheTitle}{{Catalyst Acceleration \\ for Gradient-Based Non-Convex
Optimization}} 

\title{{\TheTitle}
}
\small{
\author{
   Courtney Paquette \\
   University of Waterloo \\
   \texttt{c2paquette@uwaterloo.ca} \\
   \and
   Hongzhou Lin \\
   MIT \\
   \texttt{hongzhou@mit.edu} \\
   \and
   Dmitriy Drusvyatskiy \\
   University of Washington \\
   \texttt{ddrusv@uw.edu} \\
   \and
   Julien Mairal \\ 
Inria\thanks{Univ. Grenoble Alpes, Inria, CNRS, Grenoble INP, LJK, 38000 Grenoble, France} \\
   \texttt{julien.mairal@inria.fr} \\
   \and
   Zaid Harchaoui \\
   University of Washington \\
   \texttt{zaid@uw.edu} \\
}
}

\begin{document} 
\maketitle

\begin{abstract}
We introduce a generic scheme to solve nonconvex optimization problems using
gradient-based algorithms originally designed for minimizing convex functions.
Even though these methods may originally require convexity to operate, the
proposed approach allows one to use them on weakly convex objectives, which covers
a large class of non-convex functions typically appearing in machine learning
and signal processing.
In general, the scheme is guaranteed to produce a stationary point with 
a worst-case efficiency typical of first-order methods, and when the objective
turns out to be convex, it automatically accelerates in the sense of Nesterov
and achieves near-optimal convergence rate in function values.
These properties are achieved without assuming any knowledge about the
convexity of the objective, by automatically adapting to the unknown weak
convexity constant.  We conclude the paper by showing promising experimental
results obtained by applying our approach to incremental algorithms such as
SVRG and SAGA for sparse matrix factorization and for learning neural networks.
\end{abstract}

\section{Introduction}
We consider optimization problems of the form
\begin{equation}
\label{eq:main-obj}
   \min_{x \in \R^p} \: \left \{ f(x) \defin \frac{1}{n}\sum^n_{i=1}
     f_i(x) +\psi(x) \right \} \,.
\end{equation}
Here each function $f_i\colon\R^p\to \R$ is smooth, the regularization
$\psi\colon\R^p\to \oR$ may be nonsmooth, and we consider the extended
real number system $\oR := \R \cup \{\infty\}$. By considering
extended real-valued functions, this
composite setting also encompasses constrained minimization by letting
$\psi$ be the indicator function of the constraints on~$x$. 
Minimization of regularized empirical risk objectives of form~(\ref{eq:main-obj}) is central in machine learning.
Whereas a significant amount of work has been devoted to this composite setting
for convex problems, leading in particular to fast incremental
algorithms~\citep[see, \eg,][]{saga,conjugategradient,miso,sag,woodworth:srebro:2016,proxsvrg},
the question of minimizing efficiently~(\ref{eq:main-obj}) when the
functions~$f_i$ and~$\psi$ may be nonconvex is still largely open today.

Yet, nonconvex problems in machine learning are of high interest.  For
instance, the variable $x$ may represent the parameters of a neural network,
where each term $f_i(x)$ measures the fit between~$x$ and a data point indexed
by~$i$, or~(\ref{eq:main-obj}) may correspond to a nonconvex matrix
factorization problem (see Section~\ref{sec:exp}).
Besides, even when the data-fitting functions~$f_i$ are convex, it is also typical
to consider nonconvex regularization functions~$\psi$, for example for
feature selection in signal processing~\citep{htw:2015}. In this work,
we address two questions from nonconvex optimization:
\begin{enumerate}
\item How to apply a method for convex optimization to a nonconvex
  problem?
\item How to design an algorithm which does not need to know whether
  the objective function is convex while obtaining the optimal
  convergence guarantee if the function is convex?
\end{enumerate}

Several works attempted to transfer ideas from 
the convex world to the nonconvex one, see, \eg, \cite{GL1,GL2}.
Our paper has a similar goal and studies the extension of Nesterov's
acceleration for convex problems~\citep{nesterov1983} to nonconvex
composite ones. For $C^1$-smooth and nonconvex problems, gradient
descent is optimal among first-order methods in terms of information
based complexity to find an $\varepsilon$-stationary point \cite[Theorem 2, Sec. 5]{GD_optimal}. Without additional assumptions, worst case complexity for
first-order methods can not achieve better than
$\mathcal{O}(\varepsilon^{-2})$ oracle queries \cite{Cartis2010,Cartis2014}.
Under a stronger assumption that the objective function is $C^2$-smooth, 
state-of-the-art methods~\citep[\eg,][]{CDHS} 
achieve a marginal gain with complexity $O(\varepsilon^{-7/4}\log(1/\varepsilon))$, and do not appear to generalize to composite 
or finite-sum settings.
For this reason, our work fits within a broader stream of recent research on
methods that \emph{do not perform worse than gradient descent in the nonconvex case}
(in terms of worst-case complexity), while \emph{automatically accelerating for
minimizing convex functions}. The hope when applying such methods to nonconvex problems
is to see acceleration in practice, by heuristically exploiting convexity that is ``hidden''
in the objective (for instance, local convexity near the optimum, or convexity along the trajectory of iterates).

The main contribution of this paper is a \emph{generic} meta-algorithm, dubbed \autonewalgo, which is able to use a \emph{gradient-based}
optimization method~$\mtd$, originally designed for convex problems, and turn it
into an accelerated scheme that also applies to nonconvex objective functions.
The proposed \autonewalgosp can be seen as a $\bf{4}$-{\textbf{W}}heel-{\textbf{D}}rive extension of Catalyst~\cite{catalyst,catalyst_new} to all optimization ``terrains''
(convex and nonconvex). 
Specifically, without knowing whether the objective function is convex or not, our algorithm takes a method~$\mtd$ designed for convex optimization problems with the same structure as~(\ref{eq:main-obj}), 
\eg, SAGA~\cite{saga}, SVRG~\cite{proxsvrg}, and apply $\mtd$ to a sequence of sub-problems such that it asymptotically provides a stationary point of the nonconvex objective.
Overall, the number of iterations of~$\mtd$ to obtain a gradient norm smaller than $\varepsilon$ is $\tilde{O}(\varepsilon^{-2})$ in the worst case, while automatically reducing to
$\tilde{O}(\varepsilon^{-2/3})$ if the function is convex.\footnote{In this section, the notation $\tilde{O}$ only displays the polynomial dependency with respect to $\varepsilon$ for the clarity of exposition.}

\vspace*{-0.2cm}
\paragraph{Related work.}
Inspired by Nesterov's acceleration method for convex optimization~\citep[][]{nesterov}, 
the first accelerated method performing universally well for nonconvex and convex problems was introduced in~\cite{GL1}. 
Specifically, this work addresses composite
problems such as~(\ref{eq:main-obj}) with $n=1$, 
and, provided the iterates are bounded, it performs no worse than gradient
descent on nonconvex instances with complexity $O(\epsilon^{-2})$ on the gradient norm.
When the problem is convex, it accelerates with complexity
$O(\epsilon^{-2/3})$.  Extensions to accelerated Gauss-Newton type methods were also recently developed in~\cite{accel_prox_comp}.
In a follow-up work, the authors of~\cite{GL2} proposed a new scheme that 
monotonically interlaces proximal gradient descent steps and
Nesterov's extrapolation; thereby achieving similar guarantees as~\cite{GL1}
but without the need to assume the iterates to be bounded.
Extensions when the gradient of $\psi$ is only H\"{o}lder continuous
can also be devised. Whether accelerated methods are superior to
gradient descent remains an open question in the nonconvex setting; their performance escaping
saddle points faster than gradient descent has been observed \cite{AGD_saddle_points_Jordan,AG_near_crit_pts}. 

In~\cite{NIPS2015_5728}, a similar strategy is proposed, focusing instead on convergence guarantees 
under the so-called Kurdyka-{\L}ojasiewicz inequality---a property corresponding to polynomial-like growth of the function, as shown by \cite{error_KL}.
Our scheme is in the same spirit as these previous papers, since it
monotonically interlaces proximal-point steps (instead of proximal-gradient as in~\cite{GL2}) and
extrapolation/acceleration steps.  A fundamental difference is that our method is generic and accommodates inexact computations, since we
allow the subproblems to be approximately solved by any method we wish to accelerate.

By considering $C^2$-smooth nonconvex objective functions~$f$ with Lipschitz continuous gradient $\nabla f$ and Hessian $\nabla^2 f$, the authors of~\cite{CDHS} 
propose an algorithm with complexity
$O(\epsilon^{-7/4}\log(1/\epsilon))$, based on iteratively solving convex subproblems closely related to the original
problem. It is not clear if the complexity of their algorithm improves in the convex setting. 
Note also that the algorithm proposed in~\cite{CDHS} is inherently for $C^2$-smooth minimization and requires exact gradient evaluations.
This implies that the scheme does not allow incorporating nonsmooth regularizers and can not exploit finite sum structure. 

In \cite{reddi2016proximal}, stochastic methods for
minimizing \eqref{eq:main-obj} using variants of SVRG \cite{JT_SVRG}
and SAGA \cite{saga}. Their scheme works for both convex
and nonconvex settings and achieves convergence
guarantees of $O(Ln/\varepsilon)$ (convex) and $O(n^{2/3}
L/\varepsilon^{2})$ (nonconvex). Although for nonconvex problems 
our scheme guarantees a rate of $\tilde{O} \left (
  \frac{nL}{\varepsilon^2} \right )$, it enjoys the optimal
accelerated rate in the convex setting (See
Table~\ref{ncvx_catalyst:tbl: exist_alg_conv_rates_main}). The
empirical results of \cite{reddi2016proximal} used a step size of
$1/L$, but their theoretical analysis without mini-batching required a much smaller step-size,
$1/(n^{2/3}L)$, whereas our analysis is able to use the larger $1/L$ stepsize. 

\begin{table}[hbtp]\label{ncvx_catalyst:tbl: exist_alg_conv_rates_main}
   \renewcommand{\arraystretch}{2.2}
   \vspace*{-0.1cm}
\centering
\setlength\tabcolsep{2pt}
   \begin{tabular}{|p{2.2cm}| c | c  | c |}
	\hline
& \begin{minipage}{0.1 \textwidth} \begin{center}
    Theoretical \\ stepsize \end{center} \end{minipage} & {Nonconvex}  & {Convex}\\
\hline
SVRG \cite{proxsvrg} & $\displaystyle O\left (\frac{1}{L} \right )$ & not avail.  & $\displaystyle O \left (n \frac{L}{\varepsilon} \right )$ \\
\hline
\begin{minipage}{0.2\textwidth} ncvx-SVRG \\ 
     \cite{allen2016variance,reddi2016stochastic,reddi2016proximal}\end{minipage}&$\displaystyle O\left (\frac{1}{n^{2/3} L} \right
       )$& $\displaystyle O\left (  \frac{n^{2/3} L}{\varepsilon^2} \right )$ & $\displaystyle O\left ( \sqrt{n} \frac{ L}{\varepsilon} \right )$  \\
\hline
\begin{minipage}{0.2\textwidth} 
4WD-Catalyst \\\quad-SVRG
\end{minipage}
&$\displaystyle O\left (\frac{1}{L} \right )$ &$\displaystyle
                                                \tilde{O} \left (
                                                \frac{n
                                                L}{\varepsilon^2}
                                                \right )$ &
                                                            $\displaystyle
                                                            \tilde{O}
                                                            \left (
                                                            \sqrt{
                                                            \frac{n
                                                            L}{\varepsilon}
                                                            }\right )$ \\
\hline
\end{tabular}
\small
\caption{Comparison of rates of convergence when applying
  \autonewalgo~ to SVRG. In the convex case, we present the
  complexity in terms of number of iterations to obtain a point $x$ satisfying
  $f(x) - f^* < \varepsilon$. In the nonconvex case, we consider instead the guarantee $\text{dist}(0, \partial
 f(x)) <\varepsilon$. Note that the theoretical stepsize of ncvx-SVRG
 is much smaller than that of our algorithm and of the original SVRG. In
 practice, the choice of a small stepsize significantly slows down the
 performance (see Section~\ref{sec:exp}), and ncvx-SVRG is often
 heuristically used with a larger stepsize in practice, which is not
 allowed by theory, see~\cite{reddi2016proximal}. A mini-batch version
 of SVRG is also proposed there, allowing large stepsizes of $O(1/L)$,
 but without changing the global complexity. A similar table for
 SAGA~\cite{saga}, gradient descent, and randomized coordinate descent
 is provided in Table~\ref{ncvx_catalyst:tbl: exist} of Section~\ref{sec:existing_alg}. }
   \vspace*{0.1cm}
\end{table}

A stochastic scheme for minimizing \eqref{eq:main-obj} under
the nonconvex but \emph{smooth} setting were recently considered in
\cite{SCSG_noncvx}. The method can be seen as a nonconvex variant of the
stochastically controlled stochastic gradient (SCSG) methods \cite{SCSG}. 
If the target accuracy is
small, then the method performs no worse than nonconvex SVRG
\cite{reddi2016proximal}. If the target accuracy is large, the
method achieves a rate better than SGD. The proposed scheme does not
incorporate nonsmooth regularizers and it is unclear whether numerically
the scheme performs as well as SVRG. 

Finally, a stochastic method related to SVRG~\cite{JT_SVRG} for
minimizing large sums while automatically adapting to the weak
convexity constant of the objective function is proposed in~\cite{natasha}. When the weak convexity constant is small (\emph{i.e.}, the function is
nearly convex), the proposed method enjoys an improved efficiency estimate.
This algorithm, however, does not automatically accelerate for convex problems,
in the sense that the overall rate is slower than $O(\epsilon^{-3/2})$ in terms
of target accuracy $\epsilon$ on the gradient norm. 

\paragraph{Organization of the paper.}
Section~\ref{sec:weak_conv} presents mathematical tools for non-convex
and non-smooth analysis, which are used throughout the paper. We
provide a discussion of related works for solving the nonconvex
and nonsmooth problem \eqref{eq:main-obj} in Section~\ref{sec:related_works_wc}. 
In Sections~\ref{sec:algo} and~\ref{sec:autoalgo}, we introduce the
main algorithm and important extensions,
respectively. Section~\ref{sec:existing_alg} presents global
convergence guarantees of the scheme and convergence guarantees when
the algorithm wraps specific algorithms such as SAGA, SVRG, and
randomized coordinate descent. Finally, we present experimental results on matrix factorization and training
of neural networks in Section~\ref{sec:exp}.
\section{Tools for nonconvex 
  optimization}\label{sec:weak_conv} 
 In this paper, we focus on a broad class of nonconvex functions known as \emph{weakly convex} functions, which covers most of the cases of interest in machine learning and signal processing. 

 \subsection{Weakly-convex functions}
Weakly convex functions have appeared in a wide variety of contexts, and under
different names. Some notable examples are globally
lower-$C^2$ \cite{rock_subsmooth}, prox-regular \cite{prox_reg_var_anal},
proximally smooth functions \cite{prox_smooth_equiv_rock}, and those functions
whose epigraph has positive reach \cite{pos_reach}. 
 We recall here basic definitions and classical results.
\begin{defn}[Weak convexity]
	{\rm 
	A function $f\colon\R^p\to \oR$ is {\em  $\weakcnx-$weakly convex} if for
	any points~$x,y$ in~$\R^p$ and for any $\lambda$ in $[0,1]$, the approximate secant inequality holds:}
\begin{equation}	
\begin{aligned} \label{eq: weak_cnx_1}
	f(\lambda x + (1-\lambda) y) \le \lambda f(x) + (1-\lambda) f(y) 
	+\tfrac{\weakcnx \lambda (1-\lambda)}{2} \norm{x-y}^2. 
	\end{aligned}
\end{equation}
\end{defn}
\begin{remark}
	When $\weakcnx =0$, the above definition reduces to the classical definition of convex functions.
\end{remark}

\begin{figure}[h!] \centering{ \begin{minipage}{0.45\textwidth}
\begin{tikzpicture}
    \begin{axis}[
        domain=-25:25,
        xmin=-20, xmax=20,
        ymin=-30, ymax=150,
        samples=1000,
        axis y line=middle,
        axis x line=middle,
    ]
        \addplot+[mark=none, ultra thick]
        {-1*x*x-x-0.25+10*abs(x)+10*abs(x-8)+10*abs(x+8)-160}; 
    \end{axis}
\end{tikzpicture}
\end{minipage} \qquad \begin{minipage}{0.45\textwidth}
\begin{tikzpicture}
    \begin{axis}[
        domain=-25:25,
        xmin=-20, xmax=20,
        ymin=-30, ymax=150,
        samples=1000,
        axis y line=middle,
        axis x line=middle,
    ]
        \addplot+[mark=none, ultra thick]
        {10*abs(x)+10*abs(x-8)+10*abs(x+8)-160}; 
    \end{axis}
\end{tikzpicture}
\end{minipage}
\caption{Example of a weakly convex function. By adding an appropriate quadratic to the
  weakly convex function (left), we get the convex function on the
  right. }
}
\end{figure}
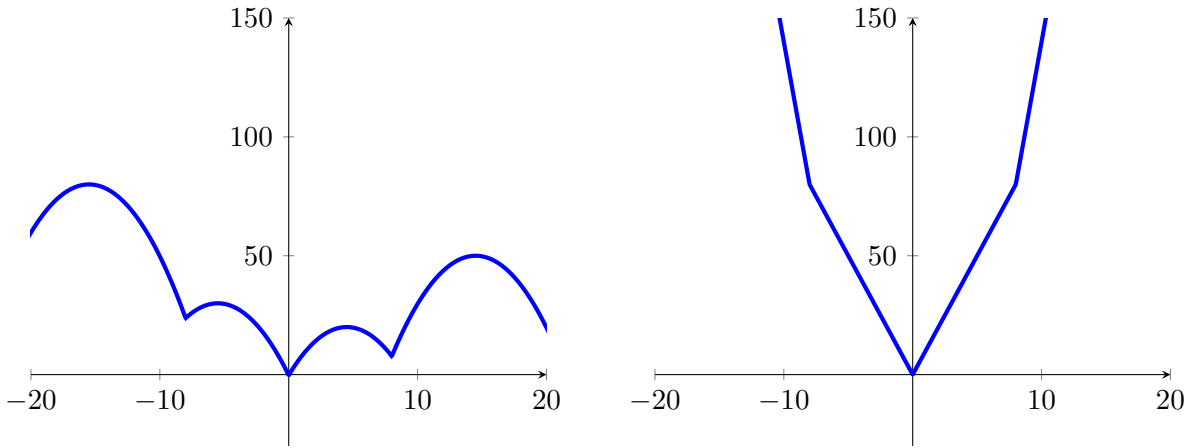

\begin{proposition} \label{prop: weak_cnx_char}
	A function $f$ is $\weakcnx-$weakly convex if and only if the function $f_\weakcnx$ is convex, where
	$$ f_\weakcnx(x) \defin f(x) + \tfrac{\weakcnx}{2} \norm{x}^2.$$
\end{proposition}

\begin{proof} A simple computation shows
\begin{equation}
\begin{aligned} \label{eq: weak_cnx_exp}
\tfrac{\weakcnx \lambda(1-\lambda)}{2} \|x-y\|^2 &=
                                                   \tfrac{\weakcnx}{2}
                                                   \lambda(1-\lambda)
                                                   \big ( \|x\|^2 +
                                                   \|y\|^2 - 2\ip{x,y}
                                                   \big )\\
&= \tfrac{\weakcnx \lambda}{2} \|x\|^2 + \tfrac{\weakcnx
  (1-\lambda)}{2} \|y\|^2 - \tfrac{\weakcnx \lambda^2}{2} \|x\|^2 -
  \tfrac{\weakcnx (1-\lambda)^2}{2} \|y\|^2 - \weakcnx \lambda
  (1-\lambda) \ip{x,y}\\
&= \tfrac{\weakcnx \lambda}{2} \|x\|^2 + \tfrac{\weakcnx
  (1-\lambda)}{2} \|y\|^2 - \tfrac{\weakcnx}{2} \|\lambda x +
  (1-\lambda) y\|^2.
\end{aligned}
\end{equation}
Suppose $f_{\weakcnx}(x)$ is convex. Then for any
  $\lambda \in [0,1]$ and $x,y \in \R^p$, we have
\begin{align*}
f_{\weakcnx}(\lambda x + (1-\lambda) y) &= f(\lambda x + (1-\lambda)y)
  + \tfrac{\weakcnx}{2} \| \lambda x + (1-\lambda) y \|^2\\
&\le \lambda f(x) + (1-\lambda
  ) f(y) + \tfrac{\weakcnx \lambda}{2} \| x\|^2 + \tfrac{\weakcnx  (1-\lambda)}{2} \|y\|^2.
\end{align*}
In order to prove the result, it suffices to show that 
\[ \tfrac{\weakcnx \lambda}{2} \|x\|^2 + \tfrac{\weakcnx
    (1-\lambda)}{2} \|y\|^2 -\tfrac{\weakcnx}{2} \|\lambda x +
  (1-\lambda) y\|^2 \le \tfrac{\weakcnx \lambda(1-\lambda)}{2}
  \|x-y\|^2.\]
This follows by rearranging the terms in Equation~\eqref{eq: weak_cnx_exp}. Next, we suppose
$f$ is $\rho$-weakly convex; hence Equation~\eqref{eq: weak_cnx_1}
holds. We observe that $\tfrac{\weakcnx \lambda (1-\lambda)}{2}
\|x-y\|^2$ can be rewritten as Equation~\eqref{eq: weak_cnx_exp}. As a
result, we conclude 
\begin{align*}
	f(\lambda x + (1-\lambda) y) \le \lambda f(x) + (1-\lambda) f(y) 
	+  \tfrac{\weakcnx \lambda}{2} \|x\|^2 + \tfrac{\weakcnx
  (1-\lambda)}{2} \|y\|^2 - \tfrac{\weakcnx}{2} \|\lambda x +
  (1-\lambda) y\|^2.
\end{align*}
Rearranging the terms, we get the desired result. 
\end{proof}

\begin{coro}
	If $f$ is twice differentiable, then $f$ is $\weakcnx$-weakly convex if and only if
$\nabla^2 f(x)\succeq -\weakcnx I$ for all $x$.
\end{coro}

\begin{proof} This follows from the observations that a twice
  differentiable function is convex if and only if $\nabla^2 f(x)
  \succeq 0$ and Proposition~\ref{prop: weak_cnx_char}.
\end{proof}

Intuitively, a function is weakly convex when it is ``nearly convex'' up to a quadratic function. This represents a complementary notion to strong convexity. 

\begin{proposition} \label{prop: weak_convex_Lip}
	If a function $f$ is differentiable and its gradient is Lipschitz continuous with Lipschitz parameter $L$, then $f$ is $L$-weakly convex.
\end{proposition}

\begin{proof} Since $f$ is differentiable and its gradient is
  $L$-Lipschitz, we observe for all $x,y \in \R^p$
\begin{align*}
f(y) &\ge f(x) + \ip{\nabla f(x), y-x} - \tfrac{L}{2} \norm{x-y}^2\\
&= f(x) + \ip{\nabla f(x), y-x} - \tfrac{L}{2} \norm{x}^2 + L \ip{x,y}
  - \tfrac{L}{2} \norm{y}^2 + L \norm{x}^2 - L \norm{x}^2\\
&= f(x) + \tfrac{L}{2} \norm{x}^2 + \ip{\nabla f(x) + Lx, y-x} -
  \tfrac{L}{2} \norm{y}^2.
\end{align*}
By rearranging the terms, we deduce
\[f_L(y) \ge f_L(x) + \ip{\nabla f_L(x), y-x}, \quad \text{for all
    $x, y \in \R^p$.}\]
Hence, we see the function $f_{\weakcnx}(x)$ is convex for $\weakcnx =
L$ and the result follows by
applying Proposition~\ref{prop: weak_cnx_char}.
\end{proof}

We remark that for most of the interesting machine learning problems, the smooth part of the objective function admits Lipchitz gradients, meaning that the function is in fact weakly convex. 

\subsection{Subdifferential}\label{sec:subgradient}
Convergence results for nonsmooth optimization typically rely on the concept of subdifferential. However, the generalization of the subdifferential to nonconvex nonsmooth function is not unique~\cite{borwein:lewis:2006}. With the weak convexity in hand, all these constructions coincide, and therefore we slightly abuse standard notation as set out in Rockafellar and Wets \citep{rock_wets}.  

\begin{defn}[Subdifferential]{\rm
		Consider a function $f\colon\R^p\to\oR$ and a point $x$ with $f(x)$ finite. The {\em subdifferential} of $f$ at $x$ is the set}
		\begin{equation*}
			\partial f(x)\!:=\!\{\xi \in \R^p: f(y)\!\geq \!f(
                x)+\xi^T(y-x) 
                 + o(\|y- x\|), \,\, \forall y \in \R^p \}.
		\end{equation*}
\end{defn}

Thus, a vector $\xi$ lies in $\partial f( x)$ whenever the linear
function $y\mapsto f( x)+\xi^T(y- x)$ is a lower-model of $f$, up to
first-order around $ x$. In particular, the subdifferential~$\partial
f(x)$ of a differentiable function~$f$ is the singleton~$\{\nabla
f(x)\}$; while for a convex function $f$ it coincides with the
subdifferential in the sense of convex analysis, see \cite[Exercise 8.8]{rock_wets}. Moreover, the following sum rule, 
$$\partial (f+g)(x)=\partial f(x)+\nabla g(x),$$ 
holds for any differentiable function $g$. 

In non-convex optimization, standard complexity bounds are derived to guarantee
\begin{equation*}
\text{\rm
  dist} \big (0, \partial f(x) \big ) \leq \varepsilon \; .
  \end{equation*}
Recall when $\varepsilon =0$, we are at a stationary point and
satisfy first-order optimality conditions. For functions that
are nonconvex, first-order methods search for points with small subgradients,
which does not necessarily imply small function values, in contrast to convex functions where the two criteria are much closer related. In our convergence analysis, we will use the following differential characterization of $\weakcnx$-weakly convex functions, which generalize classical properties of convex functions. 
\begin{theorem}[Differential characterization of $\weakcnx$-weakly
  convex functions] {\hfill \\ }
For any lower-semicontinuous function $f\colon\R^p\to\oR$, the following properties are equivalent:
\begin{enumerate}
	\item $f$ is $\weakcnx$-weakly convex.
	\item\label{it:subgrad} {\bf (subgradient inequality)} The inequality 	
	\begin{equation*} 
	f(y) \ge f(x) + v^T(y-x) - \frac{\weakcnx}{2}
	\norm{y-x}^2
	\end{equation*}
	holds for all $x,y \in \R^p$ and $v\in \partial f(x)$.
	\item\label{it:hypo_mon} {\bf (hypo-monotonicity)} The inequality $$(v-w)^T(x-y)\geq -\weakcnx\|x-y\|^2$$
	holds for all $x,y\in \R^p$ and $v\in \partial f(x)$, $w\in \partial f(y)$. 
\end{enumerate} 
\end{theorem}

\section{Related work on weakly convex functions} \label{sec:related_works_wc}
For many machine learning problems, the objective functions includes a
smooth component which is often assumed to have an $L$-Lipschitz
gradient. The precise relationship between the weak-convexity constant
$\weakcnx$ and the Lipschitz constant $L$ is given in
Proposition~\ref{prop: weak_convex_Lip}:
\begin{center}
\emph{If $f$ is differentiable and $\nabla f(x)$ is $L$-Lipschitz,
  then $f$ is $\weakcnx$-weakly convex
  for some $\weakcnx \le L$.}
\end{center}
Many functions with $L$-Lipschitz gradients
have weak-convexity constants $\weakcnx$ which are 
  smaller than $L$. Our goal is to develop a
method that exploits this property of the weak convexity constant for
nonconvex functions while obtaining optimal convergence rates for
convex problems.  
Up until now, nearly all the research for methods to solve the large finite sum
problem \eqref{eq:main-obj} have assumed $\weakcnx = 0$ (i.e. convex) or
$\weakcnx = L$. We provide a short, selective list of convergence guarantees for a few popular approaches.

\begin{itemize}
\item When $\weakcnx = 0$, \emph{Accelerated SVRG} \cite{Acc_SVRG_1,katyusha,Acc_SVRG_2} finds a point satisfying $f(x)-f^* \le
  \varepsilon$ after $\tilde{O} \left (
    \sqrt{nL/\varepsilon} \right )$ gradient computations.
\item When $\weakcnx = L$, \emph{SVRG} \cite{allen2016variance,reddi2016stochastic,reddi2016proximal} finds a point satisfying
  $\mathbb{E}(\text{dist}(0, \partial f(x) )) \le \varepsilon$ in
  $O(n^{2/3} L/\varepsilon^2)$ gradient computations.
\item When $\weakcnx = L$, \emph{Full Gradient Descent (FG)} finds a point satisfying
  $\norm{\nabla f(x)} \le \varepsilon$ after $O(nL/\varepsilon^2)$
  number of gradient computations.
\item When $\weakcnx = L$, \emph{Stochastic Gradient Descent (SGD)} finds a point
  satisfying $\mathbb{E} ( \norm{\nabla f(x)}) \le \varepsilon$ after
  $O(L/\varepsilon^2 + L C/\varepsilon^{4})$ number of gradient
  computations where $C$ is the variance of the stochastic
  gradient. This is under the assumption that $\varepsilon$ is small. 
\item When $\weakcnx = 0$, \emph{AdaGrad} \cite{adaGrad} uses regret guarantees in an online convex optimization setting. We
  are not aware of guarantees for convex optimization with finite-sum
  structure nor for non-convex optimization with finite-sum
  structure. 
\end{itemize}
To the best of our knowledge when $0 < \weakcnx \ll L$, it is unclear
whether FG, SGD, and SVRG~\cite{JT_SVRG} can take advantage of the weak
convexity constant. For notational convenience, we state all the
convergence results based on $\mathbb{E}(\text{dist}(0,\partial
f(x))) \le \varepsilon$. 

\subsection{Behavior of finite-sum optimization methods when the objective is nonconvex.} Stochastic
methods based on variance-reduced stochastic gradients have recently
been applied to nonconvex problems. The authors of \cite{reddi2016proximal} propose for instance stochastic methods for
minimizing \eqref{eq:main-obj} using variants of SVRG \cite{JT_SVRG}
and SAGA \cite{saga} under the assumption that $\weakcnx = L$. Their scheme works for both convex
and nonconvex settings and achieves convergence
guarantees of $O(Ln/\varepsilon)$ (convex) and $O(n^{2/3}
L/\varepsilon^{2})$ (nonconvex) and includes a minibatch variant. 

A stochastic scheme for minimizing large finite sum structure under
the nonconvex but \emph{smooth} setting were recently considered in
\cite{SCSG_noncvx}. In particular, they examine the problem setting
where the $f_i$ are differentiable, the function $\psi \equiv 0$, and
$\weakcnx = L$. Their observation was for low target accuracies
$\varepsilon$ (i.e., when $\varepsilon$ is \emph{not} $\varepsilon \ll n^{-1/2}$), SGD has similar or even better theoretical complexity
than FG and existing variance reduction methods. Hence, they developed
an algorithm that for low accuracy behaves better than SGD and for
high accuracy no worse than nonconvex SVRG
\cite{reddi2016proximal}. The method is a nonconvex variant of the
stochastically controlled stochastic gradient (SCSG) methods
\cite{SCSG}, attaining a convergence rate of $\tilde{O} \left (\min\{\varepsilon^{-10/3},
  n^{2/3} \varepsilon^{-2} \right )$ in gradient computations. 

Both the methods above assumed $\weakcnx =L$; however recently in \cite{natasha}, a stochastic method that
automatically adapts to the weak convexity constant of the objective
function was proposed. The method is related to SVRG~\cite{JT_SVRG}
and includes variants that use minibatching. The proposed stochastic
method finds a point~$x$ satisfying
$\mathbb{E}(\text{dist}(0, \partial f(x))) \le \varepsilon$ in $\tilde{O} \left ( \min \{
  n^{3/4}\sqrt{L\weakcnx}/\varepsilon^2, n^{2/3}(L^2
  \weakcnx)^{1/3}/\varepsilon^2  \} \right )$ stochastic gradient
computations $\nabla f_i(x)$. The author showed a dichotomy for the
weak convexity constant
$\weakcnx$: if $\weakcnx$ is
small, i.e. $\weakcnx < 
L/\sqrt{n}$, then the first term in the convergence guarantee
is smaller and if the $\weakcnx$ is large ($\weakcnx > L/\sqrt{n}$), the second term is smaller.  Up to logarithmic factors and
$\weakcnx = L$, it matches the best known rate established by
nonconvex SVRG~\cite{reddi2016proximal}. 

\subsection{Behavior of finite-sum optimization methods when the objective is convex.}\label{sec:behavior_cvx} 
For the stochastic methods previously considered in the nonconvex
setting, we note their \textit{convex} rates. In~\cite{reddi2016proximal}, for 
convex objectives, the methods attain convergence guaranties of
$O(nL/\varepsilon)$. In both \cite{natasha,SCSG_noncvx}, they
only focus on nonconvex problems, as such their convex rates are the
same as the nonconvex setting, that is, $O(\varepsilon^{-2})$. 
When the objective is assumed to be convex, methods often achieve
faster rates of convergence. Accelerated gradient methods designed by Nesterov \cite{nesterov} are known
to require $O \left ( nL^{2/3}/\varepsilon^{2/3} \right )$ number of
gradient computations to obtain near stationary point, $\norm{\nabla
  f(x)} \le \varepsilon$, but only $O(n
\sqrt{nL/\varepsilon})$ number of gradient computations to obtain a
near optimal point, $f(x)-f^* \le \varepsilon$ \cite{nest_optima}. This gap between the
two guarantees was resolved in \cite{nest_optima}. By adding a
regularization term and an \emph{additional known bound on
  $\norm{x_0-x^*}^2$}, one can improve the gradient complexity to $O(n
\sqrt{L/\varepsilon})$. Without such a bound on the distance to the
optimal solution set, it is unclear if one can improve the convergence
rate. We will assume throughout this paper that we do not know a bound on
$\norm{x^*-x_0}^2$ for \eqref{eq:main-obj}. 

The authors of \cite{SCSG_noncvx} based their work off a class of
algorithms called stochastically
controlled stochastic gradient (SCSG) methods \cite{SCSG}. In these methods, the 
functions $f_i$ are \emph{smooth} and \emph{convex}. The SCSG method satisfies: when the target accuracy is low, the method has the
same $O(\varepsilon^{-2})$ rate as SGD but with a small data-dependent
constant factor and when the target accuracy is high, the method has
the same rate as the best non-accelerated methods,
$O(nL/\varepsilon)$. 

\subsection{Our results} \label{sec:our_results} In this paper, we
design \emph{a generic method that performs no worse than gradient descent in the nonconvex case,
  while automatically accelerating for minimizing convex functions.}
In particular, we devise a \emph{single} algorithm which
adapts to the weak convexity constant if the objective is nonconvex,
while also obtaining the accelerated rate of
$O(\varepsilon^{-2/3})$ when the objective is convex. The hope is
that by applying such methods to nonconvex problems we 
see acceleration in practice by heuristically exploiting convexity
that is ``hidden'' in the objective function. Moreover, our
algorithm applies to incremental methods SVRG/SAGA \emph{and} randomized
coordinate descent. Designing such an acceleration scheme for possibly
nonconvex optimization problems is challenging. Whether convergence
guarantees for optimization
algorithms accelerated naively with classical Nesterov or momentum
acceleration match gradient descent on nonconvex problems remains an
open problem; yet in the vicinity of saddle points, accelerated
gradient methods escape faster than
gradient descent \cite{AGD_saddle_points_Jordan,AG_near_crit_pts}. Our scheme capitalizes on this
valuable observation.

First, we consider the
situation where the weak convexity constant $\weakcnx$ is known. By interlacing
incremental methods such as SVRG and SAGA, our proposed algorithm,
\newalgo -SVRG/SAGA, where~$\weakcnx$ is known, finds an $\varepsilon$-approximate stationarity point of
$f(x)$ in gradient complexity
\[\text{if $f$ is nonconvex,} \, \, \, \, \tilde{O} \left
    ( \frac{n\weakcnx}{\varepsilon^2} \right ) \qquad \text{and}
  \qquad \text{if $f$ is convex,} \, \, \, \, \displaystyle \tilde{O} \left
    ( \frac{n^{2/3} L^{2/3}}{ \varepsilon^{2/3}} \right ).\]
Moreover if the objective is convex, \newalgo -SVRG/SAGA, finds a point
satisfying $\mathbb{E}(f(x)) - f^* < \varepsilon$ in at most $\displaystyle
\tilde{O} \left (\sqrt{nL / \varepsilon} \right )$. Despite a worse
dependence on $n$ than \cite{natasha, SCSG_noncvx,reddi2016proximal}, our scheme, like that of~\cite{natasha},
highlights the dependence on $\weakcnx$, which one does not see from
the convergence guarantees of FG
or its proximal variant \cite{Cartis2014}. Moreover, \newalgo -SVRG/SAGA obtains
convergence guarantee in the convex setting rivaling accelerated SVRG
(convex) methods \cite{Acc_SVRG_1,katyusha,Acc_SVRG_2}. 

It is common in machine learning problems for the weak convexity
constant to be unknown. Previous work in the area, namely
\cite{natasha}, required the parameter $\weakcnx$ to be specified (see Line 8 in
Algorithms 1 and 2 of \cite{natasha}). Our second method, \autonewalgo
-SVRG/SAGA, \emph{incorporates
a procedure that eliminates the need to specify the weak convexity
constant $\weakcnx$}. The resulting method, \autonewalgo -SVRG/SAGA,  finds an $\varepsilon$-approximate stationarity point of
$f(x)$ in gradient complexity
\[\text{if $f$ is nonconvex,} \, \, \, \, \tilde{O} \left
    ( \frac{n L}{\varepsilon^2} \right ) \qquad \text{and}
  \qquad \text{if $f$ is convex,} \, \, \, \, \displaystyle \tilde{O} \left
    ( \frac{n^{2/3} L^{2/3}}{ \varepsilon^{2/3}} \right ).\]
The scheme, \autonewalgo -SVRG/SAGA, also finds such a
solution, in the convex regime, in $\displaystyle \tilde{O} \left
  (\sqrt{nL / \varepsilon} \right )$. 

We also apply \newalgosp to randomized coordinate descent, denoted
\autonewalgo -Rand. CD. Here we assume the objective function $f$ is
smooth and its gradient satisfies $|\nabla f(x+te_i)_i - (\nabla f(x))_i| <
L_i|t|$. We denote $L_{\max}$ the max. of the coordinate Lipschitz
       constants for $\nabla f(x)$ and $p$ is the dimension of
       the domain of $f$. We show that \autonewalgo -Rand. CD attains an $\varepsilon$-near optimal
solution, in the convex regime, in $\displaystyle \tilde{O} \left
  (p \sqrt{L_{\max} / \varepsilon} \right )$. This agrees with the
results for accelerated randomized CD \cite{Wright_CD}. 
\section{The \newalgosp algorithm for non-convex
  optimization}\label{sec:algo}

We now present a generic scheme (Algorithm~\ref{alg: uniform_cat_nols}) for applying a convex optimization method to minimize 
\begin{equation}
\min_{x\in \R^p}~ f(x),
\end{equation}
 where $f$ is only $\weakcnx$-weakly convex and $f$ is lower bounded. Our goal is to develop a unified framework that automatically accelerates in convex settings. Consequently, the scheme must be agnostic to the constant $\rho$.

\subsection{\newalgosp: a meta algorithm}
At the center of our meta algorithm (Algorithm~\ref{alg: uniform_cat_nols}) are two sequences of subproblems obtained by
adding simple quadratics to $f$. The proposed approach extends the Catalyst acceleration of~\cite{catalyst,catalyst_new}
and comes with a simplified convergence analysis. We next describe in detail each step of the scheme.

\vs
\paragraph{Two-step subproblems.}
The proposed acceleration scheme builds two main sequences of iterates 
$(\proxx_\cnt)_{\cnt}$ and $(\accx_\cnt)_{\cnt}$,
obtained from approximately solving two subproblems. These subproblems 
are simple quadratic perturbations of the original problem $f$ having the form:
$$ \min_{x} \left \{ f_ {\smthpara} (x;y) : = f(x) + \frac{\kappa}{2} \norm{x-y}^2  \right \}.$$
Here, $\kappa$ is a regularization parameter and $y$ is called the {\em prox-center}. By adding the quadratic, 
we make the problem more ``convex'':  when $f$ is non convex, with a large enough $\smthpara$, the 
subproblem will be convex; when~$f$ is convex, we improve the conditioning of the problem.

At the $\cnt$-th iteration, given a previous iterate $x_{\cnt-1}$ and the extrapolation term $v_{\cnt-1}$, we 
construct the two following subproblems.
\begin{enumerate}
\item {\bf Proximal point step.}
We first perform an inexact proximal point step with prox-center $x_{\cnt-1}$: 
\begin{align*}
\proxx_\cnt \approx \argmin_x \env(x; x_{\cnt-1}) &\quad \text{[Proximal-point step]}
\end{align*}

\item 
 {\bf Accelerated proximal point step.}
Then we build the next prox-center $y_{\cnt}$ as the convex combination
\begin{equation}
   y_\cnt = \alpha_\cnt v_{\cnt-1} + (1-\alpha_\cnt) x_{\cnt-1}.  \label{eq:yk}
\end{equation}
Next, we use $y_{\cnt}$ as a prox-center and update the next extrapolation term: 
\begin{align} 
\accx_\cnt &\approx \argmin_x \env(x; y_\cnt) &
                                                \!\!\!\!\!\!\!\!\!\!\!\!\!\!\text{[Accelerated
                                                proximal-point step]} \nonumber \\
   v_\cnt &= x_{\cnt-1} + \tfrac{1}{\alpha_\cnt} (\accx_\cnt-x_{\cnt-1}) &\quad \text{[Extrapolation]}   \label{eq:vk}
\end{align}
where $\alpha_{\cnt+1} \in (0,1)$ is a sequence of coefficients satisfying $\small{(1-\alpha_{\cnt+1})/\alpha^2_{\cnt+1} = 1/{\alpha_{\cnt}}^2}$. 
Essentially, the sequences $(\alpha_{\cnt})_{\cnt}, (y_{\cnt})_{\cnt}, (v_{\cnt})_{\cnt}$ are built upon the extrapolation principles of Nesterov~\cite{nesterov}.

\end{enumerate}

\vs
\paragraph{Picking the best.}
At the end of iteration $k$, we have at hand two iterates, resp. $\proxx_\cnt$ and $\accx_\cnt$. Following~\cite{GL1}, we 
simply choose the best of the two in terms of their objective values, that is we choose $x_\cnt$ such that
\begin{equation*}
f(x_\cnt) \leq \min \, \{ f(\proxx_\cnt), f(\accx_\cnt) \} \; .
\end{equation*}
The proposed scheme blends the two steps in a synergistic way,
allowing us to recover the near-optimal rates of convergence 
in both worlds: convex and non-convex. 
Intuitively, when $\proxx_{\cnt}$ is chosen, it means that Nesterov's
extrapolation step ``fails'' to accelerate convergence.

\begin{algorithm}[!t]
	\caption{\newalgo}
	\label{alg: uniform_cat_nols}
	\begin{algorithmic}[1]
		\INPUT  Fix a point $x_0 \in \text{dom}
		\, f$, real numbers $\smthpara > 0$,
		and an optimization method $\mathcal{M}$.\\
		\textbf{initialization:} $\alpha_1 \equiv 1$, $v_0 \equiv x_0$.\\
		\textbf{repeat} for $\cnt =1, 2, \dots$
\begin{enumerate}
	 \item Choose $\proxx_\cnt$ using $\mathcal{M}$ such that
		\begin{equation}\label{eqn:prox_1}
		\proxx_\cnt \approx \argmin_x \env(x; x_{\cnt-1})
		\end{equation}
where $\text{\rm dist} \big (0, \partial
                        \env(\proxx_\cnt; x_{\cnt-1}) \big )<
			\smthpara \norm{\proxx_\cnt-x_{\cnt-1}}$ and
                        $\env(\proxx_\cnt; x_{\cnt-1}) \le
                        \env(x_{\cnt-1}; x_{\cnt-1})$.
			
				 \item Set 	
				\begin{equation}\label{eqn:accel_1}y_\cnt = \alpha_\cnt v_{\cnt-1} + (1-\alpha_\cnt) x_{\cnt-1}.		
		\end{equation}	
		\item Choose $\accx_\cnt$ using $\mathcal{M}$ such that
		\begin{equation}\label{eqn:accel_2}
		 \accx_\cnt \approx \argmin_x f_{\smthpara}(x; y_\cnt) \qquad\text{ where \qquad
			$\text{\rm dist}\big (0, \partial
                        f_{\smthpara}(\accx_\cnt; y_\cnt) \big ) <
                       \frac{\smthpara}{\cnt+1}
                       \norm{\accx_{\cnt}-y_{\cnt}} $ } .
		\end{equation}
		
		 \item  Set 
		\begin{equation}\label{eqn:v_try} v_\cnt = x_{\cnt-1} + \frac{1}{\alpha_\cnt} (\accx_\cnt-x_{\cnt-1}).
		\end{equation}
		
		\item  Pick $\alpha_{\cnt+1} \in (0,1)$ satisfying
		\begin{equation}\label{eqn:accel_end} \frac{1-\alpha_{\cnt+1}}{\alpha^2_{\cnt+1}} = \frac{1}{\alpha_{\cnt}^2}.
		\end{equation} 
		 \item Choose $x_\cnt$ to be any point satisfying
		\begin{equation}\label{eqn:min_better}
		f(x_\cnt) \leq \min \, \{ f(\proxx_\cnt), f(\accx_\cnt) \} .
		\end{equation}
\end{enumerate}
\textbf{until} the stopping criterion $\text{\rm dist}\big(0, \partial
f(\proxx_\cnt) \big ) < \varepsilon$
	\end{algorithmic}
\end{algorithm}

\paragraph{Stopping criterion for the subproblems.} 
In order to derive complexity bounds, it is important to properly define the stopping criterion for the proximal subproblems. 
When the subproblem is convex, a functional gap like
$f_\smthpara (z; x) - \inf_z f_\smthpara(z; x)$ may be
used as a control of the inexactness, as in \cite{catalyst,catalyst_new}. Without
convexity, this criterion cannot be used since such quantities can
not be easily bounded. In particular, first order methods seek points whose
subgradient is small. Since small subgradients do
not necessarily imply small function values in a non-convex setting, first order methods
only test for near-stationarity is small subgradients. In
contrast, in the convex setting, small subgradients imply small function
values; thus a first order method in the convex setting can ``test''
for small function values. Hence, we cannot use a direct application of
Catalyst \cite{catalyst,catalyst_new} which uses the functional gap as a stopping criteria. Because we are working in the nonconvex
setting, we include a stationarity stopping criteria. 

We propose to use jointly the following two types of stopping criteria: 
\begin{enumerate}
\item Descent condition: $f_\smthpara(z;y) \leq f_\smthpara(y;y)$;
\item Adaptive stationary condition:  $\text{\rm dist} \big (0, \partial
                        \env(z; y) \big )< \smthpara \norm{z-y}$.
\end{enumerate}
Without the descent condition, the stationarity condition is
insufficient for defining a good stopping criterion
because of the existence of local maxima in nonconvex
problems. In the nonconvex setting, local maxima and local minima
satisfy this stationarity condition. The descent condition ensures the
iterates generated by the algorithm always decrease the value of the
objective function $f$; thus ensuring we move away from local maxima. The second criterion, adaptive stationary
condition, provides a flexible relative tolerance on termination of algorithm used for solving the subproblems; a detailed analysis is forthcoming. 

In \newalgo, we use both the stationary condition and the
descent condition as a stopping criteria to produce the point
$\proxx$: 
\begin{equation}
\text{\rm dist} \big (0, \partial
                        \env(\proxx_\cnt; x_{\cnt-1}) \big )<
			\smthpara \norm{\proxx_\cnt-x_{\cnt-1}} \text{
                          and }
                        \env(\proxx_\cnt; x_{\cnt-1}) \le
                        \env(x_{\cnt-1}; x_{\cnt-1}). \label{eq:prox_stop_criteria}
\end{equation}
For the point $\accx$, our ``acceleration'' point, we use a
modified stationary condition:
\begin{equation}
\text{\rm dist}\big (0, \partial
                        f_{\smthpara}(\accx_\cnt; y_\cnt) \big ) <
                       \frac{\smthpara}{\cnt+1}
                       \norm{\accx_{\cnt}-y_{\cnt}}.\label{eq:accx_stop_criteria}
\end{equation}
The $\cnt+1$ factor guarantees \newalgosp accelerates for the
convex setting. To be precise, Equation~\eqref{eq:important_sum} in
the proofs of Theorem \ref{theo:outerloop-ncvx-basic} and
  Theorem~\ref{theo:outerloop-cvx-basic} uses the factor $\cnt +1$ to ensure convergence. Note, we do not need the descent condition for
$\accx$, as the functional decrease in $\proxx$ is enough to ensure
the sequence $\{f(x_\cnt)\}_{\cnt \ge 1}$ is monotonically
decreasing.  

\vs
\subsection{Convergence analysis.}
We present here the theoretical properties of Algorithm~\ref{alg: uniform_cat_nols}. In this first stage, 
we do not take into account the complexity of solving the subproblems
\eqref{eqn:prox_1} and \eqref{eqn:accel_2}. For the next two theorems, we assume that  the
 stopping criteria for the proximal subproblems are satisfied at each iteration of Algorithm~\ref{alg: uniform_cat_nols}.

\begin{theorem}[Outer-loop complexity for \newalgo; non-convex
case] \label{theo:outerloop-ncvx-basic} 
Suppose that the function $f$ is lower bounded. For any $\kappa >0$ and $N \ge 1$, the iterates generated by
 Algorithm~\ref{alg: uniform_cat_nols}  
satisfy
\begin{align*}
\min_{j = 1, \hdots, N}  ~\text{\rm dist}^2 \big (0, \partial
  f(\proxx_j) \big ) \le \frac{8 \smthpara}{N} (f(x_0)-f^*).
\end{align*}
\end{theorem}

It is important to notice that this convergence result is valid for any $\smthpara$ and does not 
require it to be larger than the weak convexity parameter. As long as the stopping criteria for the proximal subproblems are 
satisfied, the quantities $\dist(0,\partial f(\bar{x_j}))$ tend to zero. The proof is inspired 
by that of inexact proximal algorithms~\cite{bertsekas:2015,guler:1991,catalyst} and appears in Appendix~\ref{sec:first}.

If the function $f$ turns out to be convex, the scheme achieves a faster convergence rate
both in function values and in stationarity:

\begin{theorem}[Outer-loop complexity, convex case] \label{theo:outerloop-cvx-basic} 
If the function $f$ is convex, then for any $\kappa>0$ and $N\geq 1$, the iterates generated by Algorithm~\ref{alg: uniform_cat_nols}
satisfy
\begin{equation} f(x_N)-f(x^*) \le \frac{4 \smthpara}{(N+1)^2}
\norm{x^*-x_0}^2,
\end{equation}
and 
\begin{align*}
\min_{j =1, \hdots, 2N}~ \text{\rm dist}^2 \big (0, \partial f(\proxx_j)
  \big ) \le
  \frac{32\smthpara^2}{N(N+1)^2} \norm{x^*-x_0}^2,
\end{align*}
where $x^*$ is any minimizer of the function $f$. 
\end{theorem}

The proof of Theorem \ref{theo:outerloop-cvx-basic} appears in Appendix~\ref{sec:first}. This theorem establishes a rate of $O(N^{-2})$
for suboptimality in function value and convergence in $O(N^{-3/2})$ for the
minimal norm of subgradients. 
The first rate is optimal in terms of information-based complexity
for the minimization of a convex composite
function~\cite{nesterov,nesterov2013gradient}. 
The second can be improved to
$O(N^{-2}\log(N))$ through a regularization technique, if one knew in advance that the function is convex and had an
estimate on the distance of the initial point to an optimal solution \cite{nest_optima}. 

\paragraph{Towards an automatically adaptive algorithm.} 
So far, our analysis has not taken into account the cost of obtaining the iterates $\bar{x}_j$ and $\tilde{x}_j$ by the algorithm $\mathcal{M}$.
We emphasize again that
the two results above do not require any assumption on $\smthpara$, which leaves us a degree of 
freedom. In order to develop the global complexity, we need to evaluate the total number of 
iterations performed by~$\mtd$ throughout the process. Clearly, this complexity heavily depends on the choice of  $\smthpara$, since it controls the magnitude of regularization we add to improve the 
convexity of the subproblem. This is the point where a careful analysis is needed, because our algorithm must 
adapt to $\weakcnx$ without knowing it in advance. The next section is
 entirely dedicated to this issue. In particular, we will explain how to automatically adapt the parameter $\kappa$ (Algorithm~\ref{alg:adap-at-catalyst}).
\section{The \autonewalgo~algorithm}\label{sec:autoalgo}
In this section, we work towards understanding the global efficiency
of Algorithm~\ref{alg: uniform_cat_nols}, which automatically adapts
to the weak convexity parameter. For this, we must take into account
the cost of approximately solving the proximal subproblems to the
desired stopping criteria. We expect that once the subproblem becomes 
strongly convex, the given optimization method $\mathcal{M}$ can solve
it efficiently. For this reason, we first focus on the computational cost
for solving the sub-problems, before introducing a new algorithm with known worst-case complexity.

\subsection{Solving the sub-problems efficiently}
When
$\smthpara$ is large enough, the subproblems become strongly convex;
thus globally solvable.
Henceforth, we will assume that $\mathcal{M}$ satisfies the following natural linear convergence assumption.

\paragraph{Linear convergence of $\mtd$ for strongly-convex problems.} \label{para: assumptions_M}
We assume that for any $\smthpara > \weakcnx$, there exist $A_{\smthpara} \ge 0$ and
$\tau_{\smthpara} \in (0,1)$ so that the following hold:
\begin{enumerate}
\item For any prox-center $y \in \R^p$ and initial $z_0 \in \R^p$ the
  iterates $\{z_t\}_{t \ge 1}$ generated by $\mathcal{M}$ on the
  problem $\min_z \env(z; y)$ satisfy 
\begin{equation} 
\text{dist}^2(0, \partial \env(z_t; y) ) \le A_{\smthpara}
  (1-\tau_{\smthpara})^t (\env(z_0; y) - \env^*(y) ), \label{eq:criteria} \end{equation}
where $\env(y)^* := \inf_z \env(z; y)$. If the method $\mtd$ is randomized, 
we require the same inequality to hold in expectation.
\item The rates $\tau_{\smthpara}$ and the constants $A_{\smthpara}$
  are increasing in $\smthpara$. 
\end{enumerate}

\begin{remark}
{\rm
Our assumption on the linear rate of convergence of $\mtd$ differs from the one considered by~\cite{catalyst,catalyst_new},
which was given in terms of function values. However, if the problem is a composite one, both points of
   view are near-equivalent, as discussed in
   Section~\ref{subsec:conv} and the precise relationship is given in Appendix~\ref{sec:appendix_global_comp}.
   We choose the norm of the subgradient as our measurement because the
   complexity analysis is easier. 
   }
\end{remark} 
Then, 
a straightforward analysis bounds 
the computational complexity to achieve an $\epsilon$-stationary point.
\begin{lemma}
Let us consider a strongly convex problem $\env(\cdot; y)$ and a linearly convergent method $\mtd$ generating a sequence of iterates $\{z_t\}_{t\geq 0}$. Define
$T(\varepsilon) = \inf \{ t\geq 1, \text{\rm dist} \big (0, \partial
\env(z_t; y) \big ) \leq \varepsilon \}, $  where~$\varepsilon$ is the target accuracy; then,
 \begin{enumerate}
  \item If $\mtd$ is deterministic, 
  $$  T(\varepsilon) \leq  \frac{1}{\tau_{\smthpara}} \log \left ( \frac{A_{\smthpara} \left (
        \env(z_0; y)- \env^*(y) \right )}{\varepsilon^2} \right ).$$
  \item  If $\mtd$ is randomized, then
  $$  \mathbb{E}\left [ T(\varepsilon) \right ]\leq  
  \frac{1}{\tau_{\smthpara}} \log \left ( \frac{A_{\smthpara} \left ( \env(z_0; y)- \env^*(y) \right )}{\tau_{\smthpara} \varepsilon^2} \right ).$$ see Lemma C.1 of \cite{catalyst}.
 \end{enumerate}
\end{lemma}
As we can see, we only lose a factor in the log term by switching from deterministic to randomized algorithms. For the sake of 
simplicity, we perform our analysis only for deterministic algorithms and the analysis for randomized algorithms holds in
the same way in expectation.

We can now prove a global complexity bound for \newalgosp if the weak
convexity constant $\weakcnx$ is known. For this, we introduce $\smthparacvx$, a $\mathcal{M}$ dependent smoothing parameter
and set it in the same way as the smoothing parameter in
\cite{catalyst}. 
\begin{theorem}[Global convergence bounds for \newalgosp with $\rho$
  known]\label{thm:global_conv_known_rho} Suppose the weak convexity constant $\weakcnx$ is
  known and the function $f$ is lower bounded. We let $\tilde{O}$ hide universal constants and  logarithmic
dependencies in $A_L$, $A_{\smthparacvx}$, $L$, $\varepsilon$, $\smthparainit$,
$\smthparacvx$, and $\norm{x^*-x_0}^2$. Then, the following statements hold.
\begin{enumerate}
\item Algorithm \ref{alg: uniform_cat_nols} generates a point
  satisfying $\text{dist}(0, \partial f(x)) \le \varepsilon$ after at
  most
	$$\tilde{O}\left(\left(\tau_L^{-1}+\tau_{\smthparacvx}^{-1}\right)\cdot\frac{\weakcnx
            (f(x_0)-f^*)}{\varepsilon^2}\right)$$
	iterations of the method $\mathcal{M}$.
	\item If $f$ is convex, then Algorithm~\ref{alg: uniform_cat_nols} generates a point $x$
	satisfying  $\text{dist} \big (0, \partial f(x) \big ) \leq \varepsilon$ after at
	most $$\tilde{O}\left(\left(\tau_L^{-1}+\tau_{\smthparacvx}^{-1}\right)\cdot\frac{L^{1/3}\left(\smthparacvx\|x^*-x_0\|^2\right)^{1/3}}{\varepsilon^{2/3}}\right)
	$$
	iterations of the method $\mathcal{M}$. 
	\item If $f$ is convex, then Algorithm~\ref{alg: uniform_cat_nols}
	 generates a point $x$ satisfying  $f(x)-f^*\leq \varepsilon$
	after at most
	$$\tilde{O}\left(\left(\tau_L^{-1}+\tau_{\smthparacvx}^{-1}\right)\cdot \frac{\sqrt{\smthparacvx\|x^*-x_0\|^2}}{\sqrt\varepsilon}\right)
	$$
	iterations of the method $\mathcal{M}$.
\end{enumerate}
\end{theorem}

\paragraph{Bounding the required iterations when $\smthpara >  \weakcnx$ and restart strategy.} 
Recall that we add a quadratic to $f$ with the hope to make each subproblem convex. Thus, if $\weakcnx$ is known, then we 
should set $\smthpara > \weakcnx$. In this first stage, we show that whenever $\smthpara > \weakcnx$, then 
the number of inner calls to $\mathcal{M}$ can be bounded with a proper initialization.
Consider the subproblem 
\begin{equation} \label{eq:arbitarysubproblem}
   \min_{x \in \R^p} \left\{ \env (x;y)  = f(x) + \frac{\kappa}{2} \norm{x-y}^2 \right\},
\end{equation}
and define the initialization point $z_0$ by
\begin{enumerate}
 \item if $f$ is smooth, then set $z_0 = y$;
 \item if $f = f_0+\psi$ is composite, with $f_0$ $L$-smooth, then
   set $z_0 = \proxi_{\eta \psi} (y-\eta \nabla f_0(y))$ with $\eta \leq \frac{1}{L+\smthpara}$.
\end{enumerate}

\begin{theorem}\label{thm:inner complex, kappa>rho}
Consider the subproblem (\ref{eq:arbitarysubproblem}) and suppose $\smthpara > \weakcnx$. 
Then initializing $\mathcal{M}$ at the previous $z_0$ generates a sequence of iterates $(z_t)_{t \geq 0}$ such that 
\begin{enumerate}
\item  in at most $T_\smthpara$ iterations where
\[ T_\smthpara =  \frac{1}{\tau_\kappa} \log \left ( \frac{8A_\smthpara
      (L+\smthpara)}{(\smthpara-\weakcnx)^2} \right ),\]
 the output $z_T$ satisfies $\env(z_T; y) \le \env(z_0; y)$ (descent condition) and
  $\text{dist}(0, \partial \env(z_T;y) ) \le \smthpara \norm{z_T-y}$
  (adaptive stationary condition);
\item in at most $S_{\smthpara}\log(\cnt +1)$ iterations where
\[S_{\smthpara} \log(\cnt +1) = \frac{1}{\tau_\smthpara} \log \left (
  \frac{8A_\smthpara (L+\smthpara) (\cnt + 1)^2}{(\smthpara-\weakcnx)^2} \right ),\]
the output $z_S$ satisfies $\text{dist}(0, \partial \env(z_S;y)) \le \frac{\smthpara}{\cnt
    +1} \norm{z_S-y}$ (modified adaptive stationary condition).
\end{enumerate}
\end{theorem}

The proof is technical and is presented in Appendix~\ref{sec: complexity}. 
The lesson we learn here is that as soon as the subproblem becomes
strongly convex, it can be solved in almost a constant number of iterations. 
Herein arises a problem--the choice of the smoothing parameter $\smthpara$.
On one hand, when $f$ is already convex, we may want to choose $\smthpara$ 
small in order to obtain the desired optimal complexity. On the other hand, 
when the problem is non convex, a small $\smthpara$ may not ensure the 
strong convexity of the subproblems. Because of such different behavior 
according to the convexity of the function, we introduce an additional parameter
$\smthparacvx$ to handle the regularization of the extrapolation step.
Moreover, in order to choose a $\smthpara >  \weakcnx$ in the nonconvex case, 
we need to know in advance an estimate of $\weakcnx$. This is not an easy task for
large scale machine learning problems such as neural networks. Thus we 
propose an adaptive step to handle it automatically.

\begin{algorithm}[!t]
   \caption{\autonewalgo} \label{alg:adap-at-catalyst}
 \begin{algorithmic}[1]
    \INPUT{Fix a point $x_0\in \text{dom} \,
	    f$, real numbers $\smthparainit, \smthpara_{\text{cvx}}
            > 0$ and $T,S>0$, and an
	  opt. method $\mathcal{M}$.}\\
    \textbf{initialization:} $\alpha_1 = 1$, $v_0 = x_0$.\\

    \textbf{repeat} for $\cnt =1, 2, \dots$
\begin{enumerate}[wide,labelindent = 10pt]
     \item Compute $$(\proxx_\cnt, \smthpara_\cnt)
      =\text{\linesearch}~ (x_{\cnt-1},\smthpara_{\cnt-1}, T).$$ 
\item Compute $y_\cnt = \alpha_\cnt v_{\cnt-1} + (1-\alpha_\cnt) x_{\cnt-1}$ 
and apply $S\log(k+1)$ iterations of $\mathcal{M}$ to find
    \begin{equation}
       \accx_{\cnt} \approx \argmin_{x \in \R^p} f_{\smthpara_{\text{cvx}}}(x,y_k),
 \label{eq:accer_adapt_2}
\end{equation}    
      by using the initialization strategy described below~(\ref{eq:arbitarysubproblem}).
\item Update  $v_\cnt $ and $\alpha_{\cnt +1}$ by
$$ v_\cnt=x_{\cnt-1} + \tfrac{1}{\alpha_\cnt} (\accx_\cnt-x_{\cnt-1}) \quad \text{ and } 
\quad \alpha_{\cnt +1} = \frac{\sqrt{\alpha_\cnt^4+4\alpha_\cnt^2} - \alpha_\cnt^2}{2}.$$
\item Choose $x_\cnt$ to be any point satisfying $f(x_\cnt) = \min \{ f(\proxx_\cnt),
    f(\accx_\cnt) \}.$
\end{enumerate}
\textbf{until} the stopping criterion $\text{\rm dist} \big
(0, \partial f(\proxx_\cnt) \big ) < \varepsilon$
 \end{algorithmic}
\end{algorithm}

\subsection{\autonewalgo: adaptation to weak convexity} 
We now introduce~\autonewalgo, presented in Algorithm~\ref{alg:adap-at-catalyst}, which can automatically adapt to 
the \emph{unknown weak convexity} constant of the objective.
The algorithm relies on a procedure to automatically adapt to~$\weakcnx$, described in Algorithm~\ref{alg:grad}.  
 
The idea is to fix in advance a number of iterations $T$, let
$\mathcal{M}$ run on the subproblem for $T$ iterations, output the
point $z_T$, and check if a sufficient decrease occurs. 
We show that if we set $T =
\tilde{O}(\tau_L^{-1})$, where the notation $\tilde{O}$
hides logarithmic dependencies in $L$ and $A_L$, where $L$ is the Lipschitz constant of the smooth part of~$f$; 
then, if the
subproblem were convex, the following conditions would be guaranteed:
 \begin{enumerate}[itemsep=0pt,topsep=0pt,parsep=0pt]
\item Descent condition: $f_\smthpara(z_T;x) \leq f_\smthpara(x;x)$;
\item Adaptive stationary condition:  $\text{\rm dist} \big (0, \partial
                        \env(z_T; x) \big )\leq \smthpara \norm{z_T-x}. $
\end{enumerate}
 Thus, if either condition is not satisfied, then the subproblem is deemed not convex and 
 we double $\smthpara$ and repeat. The procedure yields an estimate of
 $\weakcnx$ in a logarithmic number of increases; see Lemma~\ref{lem: doubling_kappa}.

\paragraph{Relative stationarity and predefining $S$.} 
One of the main differences of our approach with the Catalyst algorithm of~\cite{catalyst,catalyst_new} is
to use a \emph{pre-defined} number of iterations, $T$ and $S$, for
solving the subproblems. 
We introduce $\smthparacvx$, a $\mathcal{M}$ dependent smoothing parameter
and set it in the same way as the smoothing parameter in
\cite{catalyst,catalyst_new}. 
The automatic acceleration of our algorithm when the problem is convex
is due to extrapolation steps in Step 2-3 of~\newalgo. 
We show that if we set $S = \tilde{O}\left (
  \tau_{\smthpara_{\text{cvx}}}^{-1} \right )$, where $\tilde{O}$
hides logarithmic dependencies in $L$, $\smthpara_{\smthpara}$, and $A_{\smthparacvx}$, then we
can be sure that, for convex objectives, 
\begin{equation}
\text{\rm dist}\big (0, \partial
                        f_{\smthpara_{\text{cvx}}}(\accx_\cnt; y_\cnt) \big ) <
                       \frac{\smthpara_{\text{cvx}}}{\cnt+1}
                       \norm{\accx_{\cnt}-y_{\cnt}}.\label{eq:accx_stop_criteria_new}
\end{equation}
This relative stationarity of $\accx_\cnt$, including the choice of $\smthparacvx$, shall be crucial to guarantee that the scheme accelerates in the convex setting. 
An additional $k+1$ factor appears compared to the previous adaptive stationary condition because we need higher accuracy for solving the subproblem 
to achieve the accelerated rate in $1/\sqrt{\epsilon}$. 

We shall see in the experiments that our strategy 
of predefining $T$ and $S$ works quite well. The
theoretical bounds we derive are, in general, too conservative; we
observe in our experiments
that one may choose
$T$ and $S$ significantly smaller than the theory suggests and still retain the
stopping criteria.

\begin{algorithm}[H]
   \caption{\linesearch~$(y,\smthpara, T)$}\label{alg:grad}
   \begin{algorithmic}[1]
      \INPUT $y\in \R^p$,  method $\mtd$,  $\smthpara > 0$, number of iterations $T$.\\
  	 \textbf{Repeat} Compute
  	 \begin{equation*}
  	z_T \approx \argmin_{z \in \R^p} \env (z ; y) .
  	\end{equation*}
  	by running $T$ iterations of $\mathcal{M}$
      by using the initialization strategy described below~(\ref{eq:arbitarysubproblem}). \\
	\textbf{If} $\env (z_{T};y) > \env(y;y)$ or $\text{dist}(\partial
	\env(z_{T};y),0) > \smthpara \norm{z_{T}-y}$, \\
	\textbf{then} go to repeat with $\smthpara \rightarrow 2 \smthpara$.\\ 
	\textbf{else} go to output.
	\OUTPUT  $(z_{T}, \smthpara)$.
   \end{algorithmic}
\end{algorithm}

To derive the global complexity results for \autonewalgosp that
match optimal convergence guarantees, we make a distinction between
the regularization parameter $\smthpara$ in the proximal point step and 
in the extrapolation step. 
For the proximal point step, we apply Algorithm~\ref{alg:grad} to adaptively produce 
a sequence of $\smthpara_\cnt$ initializing at $\smthparainit > 0$, an initial guess of $\weakcnx$. 
The resulting $\proxx_\cnt$ and $\smthpara_\cnt$ satisfy both the following inequalities:
\begin{equation}
\text{\rm dist} \big (0, \partial
                        f_{\smthpara_{\cnt}}(\proxx_\cnt; x_{\cnt-1}) \big )<
			\smthpara_\cnt \norm{\proxx_\cnt-x_{\cnt}} \text{
                          and }
                        f_{\smthpara_{\cnt}}(\proxx_\cnt; x_{\cnt-1}) \le
                        f_{\smthpara_{\cnt}}(x_{\cnt-1}; x_{\cnt-1}). \label{eq:prox_stop_criteria_new}
\end{equation}
For the extrapolation step,
we introduce the parameter $\smthparacvx$ which essentially depends on the Lipschitz 
constant $L$. The choice 
is the same as the smoothing parameter in \cite{catalyst,catalyst_new} and depends on
the method $\mathcal{M}$. With a similar predefined iteration strategy,
the resulting $\accx_\cnt$ satisfies the following inequality if the original 
objective is convex,
\begin{equation}
\text{\rm dist}\big (0, \partial
                        f_{\smthparacvx}(\accx_\cnt; y_\cnt) \big ) <
                       \frac{\smthparacvx}{\cnt+1}
                       \norm{\accx_{\cnt}-y_{\cnt}}.\label{eq:accx_stop_criteria_new2}
\end{equation}

\subsection{Convergence analysis} 
Let us next postulate that $T$ and $S$ are chosen large enough to guarantee that $\proxx_{\cnt}$ and $\accx_{\cnt}$ satisfy
conditions \eqref{eq:prox_stop_criteria_new} and \eqref{eq:accx_stop_criteria_new2} for the
corresponding subproblems, and see how the outer algorithm complexity
resembles the guarantees of Theorem~\ref{theo:outerloop-ncvx-basic}
and Theorem~\ref{theo:outerloop-cvx-basic}. The main technical
difference is that~$\smthpara$ changes at each iteration $\cnt$,
which requires keeping track of the effects of
$\smthpara_\cnt$ and $\smthparacvx$ on the proof. 

\begin{theorem}[Outer-loop complexity,
  \autonewalgo] \label{thm: computational_complex_result} Fix real constants $\smthparainit, \smthparacvx >
  0$, the function $f$ is lower bounded, and $x_0 \in \text{dom}~f$. Set $\smthpara_{\max}:=\max_{\cnt\geq 1} \smthpara_\cnt$. 
Suppose that the number of iterations $T$ is such that
$\proxx_{\cnt}$ satisfies \eqref{eq:prox_stop_criteria_new}. Define $f^* := \lim_{\cnt \to \infty}
f(x_{\cnt})$. Then for any $N \ge 1$, the iterates generated by Algorithm~\ref{alg:adap-at-catalyst} satisfy,
\begin{align*}
\min_{j = 1, \hdots, N}  ~\text{\rm dist}^2 \big (0, \partial
  f(\proxx_j) \big ) \le \frac{8 \smthpara_{\max}}{N} (f(x_0)-f^*).
\end{align*}
If in addition the function $f$ is convex and $S_\cnt$ is chosen so that
$\accx_{\cnt}$ satisfies \eqref{eq:accx_stop_criteria_new2}, then
\begin{align*}
\min_{j =1, \hdots, 2N}~ \text{\rm dist}^2 \big (0, \partial f(\proxx_j)
  \big ) \le
  \frac{32\smthpara_{\max} \smthparacvx}{N(N+1)^2} \norm{x^*-x_0}^2,
\end{align*}
and 
\begin{equation} f(x_N)-f(x^*) \le \frac{4 \smthparacvx}{(N+1)^2}
  \norm{x^*-x_0}^2 ,\end{equation}
where $x^*$ is any minimizer of the function $f$. 
\end{theorem}

\paragraph{Inner-loop Complexity}
In light of Theorem~\ref{thm: computational_complex_result}, we must now understand how to choose $T$ and $S$ as small as possible, while guaranteeing that
$\proxx_{\cnt}$ and $\accx_{\cnt}$ satisfy
\eqref{eq:prox_stop_criteria_new} and
\eqref{eq:accx_stop_criteria_new2} hold for each $\cnt$. 
The quantities~$T$ and
$S$ depend on the method $\mtd$'s convergence rate parameter $\tau_\smthpara$ which only depends on $L$ and $\smthpara$. 
For example, the convergence rate parameter $\tau_\smthpara^{-1} = (L+\smthpara) / \smthpara$ for gradient descent and 
 $\tau_\smthpara^{-1} = n+(L+\smthpara) / \smthpara$ for SVRG.
The values of $T$ and $S$ must be set
beforehand without knowing the true value of the weak convexity
constant $\weakcnx$. Using Theorem~\ref{thm:inner complex, kappa>rho},
we assert the following choices for $T$ and $S$. 

\begin{theorem}[Inner complexity for \autonewalgo~:
  determining the values $T$ and
  $S$] \label{thm: inner_complexity}
Suppose the stopping criteria are \eqref{eq:prox_stop_criteria_new}
and \eqref{eq:accx_stop_criteria_new2} as in
in Theorem~\ref{thm: computational_complex_result}, and choose $T$
and $S$ in Algorithm~\ref{alg:adap-at-catalyst} to be the smallest numbers satisfying
\[T \ge \frac{1}{\tau_L} \log \left (\frac{40 A_{4L}}{L} \right ),\]
and 
\begin{align*}
S\log(\cnt + 1) \geq \frac{1}{\tau_{\smthparacvx}} \log \left ( \frac{8
      A_{\smthparacvx}(\smthparacvx +
      L) (\cnt+1)^2 }{\smthparacvx^2} \right ),
\end{align*}
for all $k$. In particular, 
\begin{align*}
T&=O\left(\frac{1}{\tau_L} \log\left(A_{4L}, L \right)\right),\\
S&=O \left ( \frac{1}{\tau_{\smthparacvx}} \log(A_{\smthparacvx}, L, \smthparacvx)
   \right ).
\end{align*}
Then $\smthpara_{\max}\leq 4L$ and the following hold for any index $\cnt \ge 1$:
\begin{enumerate}
\item Generating  $\proxx_{\cnt}$ in Algorithm~\ref{alg:adap-at-catalyst} requires at most $\tilde{O}
  \left ( \tau_L^{-1} \right )$ iterations
  of $\mathcal{M}$;
\item Generating $\accx_{\cnt}$ in Algorithm~\ref{alg:adap-at-catalyst} requires at most $\tilde{O}
  \left ( \tau_{\smthparacvx}^{-1}  \right )$ iterations
  of $\mathcal{M}$.
\end{enumerate}
where $\tilde{O}$ hides universal constants and logarithmic
dependencies on $\cnt$, $L$, $\smthparacvx$, $A_L$, and $A_{\smthparacvx}$.
\end{theorem}

Appendix~\ref{sec: complexity} is devoted to proving
Theorem~\ref{thm: inner_complexity}, but we outline below the general
procedure and state the two main propositions (see
Proposition~\ref{prop: terminates} and Proposition~\ref{prop: inner_comp_cnx}).

We summarize the proof of Theorem~\ref{thm: inner_complexity} as followed:
\begin{enumerate}
\item \label{item:1} When $\smthpara > \weakcnx + L$, we compute the number of iterations of
$\mathcal{M}$ to produce a point satisfying \eqref{eq:prox_stop_criteria_new}. Such a
  point will become $\proxx_\cnt$.
\item \label{item:2} When the function $f$ is convex, we compute the number of
   iterations of $\mathcal{M}$ to produce a point which satisfies the \eqref{eq:accx_stop_criteria_new2}
   condition. Such a point will become the point $\accx_\cnt$. 
\item We compute the smallest number of times we must double $\smthparainit$
  until it becomes larger than $\weakcnx + L$. Thus eventually the
  condition $4L \ge \smthpara > \weakcnx + L$ will occur. 
\item We always set the number of iterations of $\mathcal{M}$ to produce
  $\proxx_{\cnt}$ and $\accx_{\cnt}$ as in Step~\ref{item:1} and Step~\ref{item:2}, respectively, regardless of whether 
  $\env(\cdot; x_{\cnt})$ is convex or $f$ is convex. 
\end{enumerate}

The next proposition shows that \linesearch~terminates with a suitable
choice for $\proxx_{\cnt}$ after $T$ number of iterations. 

\begin{proposition}[Inner complexity for $\proxx_\cnt$] \label{prop:
    terminates}
   Suppose $\weakcnx +L < \smthpara \le 4L$. By initializing the method~$\mtd$ using the strategy suggested in Algorithm~\ref{alg:adap-at-catalyst} for solving
   \[\min_z \left\{ \env(z; x) := f(z) + \frac{\smthpara}{2} \norm{z-x}^2 \right\}\]
we may run the method~$\mtd$ for at least~$T$ iterations, where
\[T \ge \frac{1}{\tau_L} \log \left (\frac{40 A_{4L}}{L} \right );\]
then, the output $z_{T}$ satisfies $\env(z_{T}; x) \le \env(x;x)$ and
$\text{dist}\big (0,\partial \env(z_{{T}}; x) \big )\le
\smthpara \norm{z_T-x}$. 

\end{proposition}

Under the additional assumption that the function $f$ is convex, we
produce a point with \eqref{eq:accx_stop_criteria_new2} when the number
of iterations $S$ is chosen sufficiently large. 

\begin{proposition}[Inner-loop complexity for
  $\accx_\cnt$] \label{prop: inner_comp_cnx} 
   Consider the method~$\mtd$ with the initialization strategy suggested in Algorithm~\ref{alg:adap-at-catalyst}
   for minimizing
  $f_{\smthparacvx}(\cdot; y_{\cnt})$ with linear convergence rates of the form
  \eqref{eq:criteria}. 
  Suppose the function $f$ is convex. 
  If the number of iterations of $\mathcal{M}$ is greater than
\[S = O \left ( \frac{1}{\tau_{\smthparacvx}}
    \log(A_{\smthparacvx}, L, \smthparacvx) \right )\]
such that 
\begin{align}
S\log(k+1) \ge \frac{1}{\tau_{\smthparacvx}} \log \left (\frac{8
      A_{\smthparacvx}(\smthparacvx +
      L)(\cnt+1)^2}{\smthparacvx^2} \right )\label{eq: wtf}, 
\end{align}
   then, the output $\tilde{z}_{S} = \accx_{\cnt}$ satisfies
   $\norm{\partial f_{\smthparacvx} (\tilde{z}_{S})}<
   \frac{\smthparacvx}{\cnt +1}
   \norm{\tilde{z}_{S_{\cnt}}-y_\cnt}$ for all $\cnt \ge 1$. 
\end{proposition}

We can now derive global complexity bounds by combining Theorem~\ref{thm: computational_complex_result} and
Theorem~\ref{thm: inner_complexity}, and a good choice for the constant $\smthparacvx$.

\begin{theorem}[Global complexity bounds for \autonewalgo]
   Choose $T$ and $S$ as in
  Theorem~\ref{thm: inner_complexity}. We let $\tilde{O}$ hide universal constants and  logarithmic
dependencies in $A_L$, $A_{\smthparacvx}$, $L$, $\varepsilon$, $\smthparainit$,
$\smthparacvx$, and $\norm{x^*-x_0}^2$. Then, the following statements hold.
\begin{enumerate}
	\item 
	 Algorithm~\ref{alg:adap-at-catalyst} generates a point $x$ satisfying
$\text{dist} \big (0, \partial f(x) \big ) \leq \varepsilon$ after at most
	$$\tilde{O}\left(\left(\tau_L^{-1}+\tau_{\smthparacvx}^{-1}\right)\cdot\frac{L(f(x_0)-f^*)}{\varepsilon^2}\right)$$
	iterations of the method $\mathcal{M}$.
	
	\item If $f$ is convex, then Algorithm~\ref{alg:adap-at-catalyst} generates a point $x$
	satisfying  $\text{dist} \big (0, \partial f(x) \big ) \leq \varepsilon$ after at
	most $$\tilde{O}\left(\left(\tau_L^{-1}+\tau_{\smthparacvx}^{-1}\right)\cdot\frac{L^{1/3}\left(\smthparacvx\|x^*-x_0\|^2\right)^{1/3}}{\varepsilon^{2/3}}\right)
	$$
	iterations of the method $\mathcal{M}$. 
	\item If $f$ is convex, then Algorithm~\ref{alg:adap-at-catalyst}
	 generates a point $x$ satisfying  $f(x)-f^*\leq \varepsilon$
	after at most
	$$\tilde{O}\left(\left(\tau_L^{-1}+\tau_{\smthparacvx}^{-1}\right)\cdot \frac{\sqrt{\smthparacvx\|x^*-x_0\|^2}}{\sqrt\varepsilon}\right)
	$$
	iterations of the method $\mathcal{M}$.
\end{enumerate}
\end{theorem}

\begin{remark}
\rm{In general, the linear convergence parameter of $\mtd$, $\tau_\smthpara$, depends on the condition number of the problem $f_\smthpara$. Here, $\tau_L$ and $\tau_{\smthparacvx}$ are precisely given by plugging in $\smthpara=L$ and $\smthparacvx$ respectively into $\tau_{\smthpara}$. To clarify, let $\mtd$ be SVRG, $\tau_\smthpara$ is given by $\frac{1}{n+\frac{\smthpara+L}{\smthpara}}$ which yields $\tau_L = 1/(n+2)$. A more detailed computation is given in Table~\ref{table:parameters}. For all the incremental methods we considered, these parameters $\tau_L$ and $\tau_{\smthpara}$ are on the order of $1/n$.   }
\end{remark}

\begin{remark}
\rm{If $\mathcal{M}$ is a first order method, the convergence guarantee in
the convex setting is \emph{near-optimal}, up to logarithmic factors,
when compared to $O(1/\sqrt{\varepsilon})$~\cite{catalyst,woodworth:srebro:2016}. 
In the non-convex setting, our approach matches, up
to logarithmic factors, the best known rate for this class of
functions, namely $O(1/\varepsilon^2)$~\cite{Cartis2010,Cartis2014}. 
Moreover, our rates dependence on
the dimension and Lipschitz constant equals, up to log factors, the best known
dependencies in both the convex and nonconvex setting. These logarithmic factors may be
the price we pay for having a generic algorithm. }
\end{remark}

\section{Applications to Existing Algorithms} \label{sec:existing_alg}
We now show how to accelerate existing algorithms $\mathcal{M}$ and
compare the convergence guaranties before and after~\autonewalgo. In
particular, we focus on the gradient descent algorithm, randomized
coordinate descent, and on the 
incremental methods SAGA and SVRG. For all the algorithms considered,
we state the convergence guaranties in terms of the \emph{total number of
  iterations} (in expectation, if appropriate) to reach an accuracy of
$\varepsilon$; in the convex setting, the accuracy is
stated in terms of functional error, $f(x) - \inf f < \varepsilon$ and in the
nonconvex setting, the appropriate measure is
stationarity, namely $\text{dist}(0,\partial f(x))
< \varepsilon$. All the algorithms considered have formulations for
the composite setting with analogous convergence rates. 

Table~\ref{ncvx_catalyst:tbl: exist} presents convergence rates for
SAGA \cite{saga}, (prox) SVRG \cite{proxsvrg}, randomized coordinate
descent (Rand. CD) \cite{Wright_CD}, and gradient descent (FG). 
       
\begin{table}[hbtp!]
   \renewcommand{\arraystretch}{2.8}
   \vspace*{-0.1cm}
\centering
\small
\setlength\tabcolsep{2pt}
   \begin{tabular}{|c | c | c  | c |}
	\hline
& \begin{minipage}{0.2\textwidth} \begin{center} Theoretical \\
    stepsize \end{center} \end{minipage} & {Nonconvex}  & {Convex}\\
\hline
SVRG \cite{proxsvrg} & $\displaystyle O\left (\frac{1}{L} \right )$ & not avail.  & $\displaystyle O \left (n \frac{L}{\varepsilon} \right )$ \\
\hline

ncvx-SVRG \cite{allen2016variance,reddi2016stochastic,reddi2016proximal} & {$\displaystyle O\left (\frac{1}{n^{2/3} L} \right )$}  & {$\displaystyle O\left (  \frac{n^{2/3} L}{\varepsilon^2} \right )$} & {$\displaystyle O\left ( \sqrt{n} \frac{ L}{\varepsilon} \right )$}  \\
\hline

\autonewalgo-SVRG & {$ \displaystyle O\left (\frac{1}{L} \right )$} & {$ \displaystyle \tilde{O} \left ( \frac{n L}{\varepsilon^2} \right )$} & {$\displaystyle \tilde{O} \left ( \sqrt{n} \sqrt{ \frac{L}{\varepsilon} }\right )$} \\
\hline
SAGA \cite{saga} & $\displaystyle O \left (\frac{1}{L} \right )$ & not avail. & $\displaystyle O \left (n \frac{L}{\varepsilon} \right )$\\
\hline
ncvx-SAGA
     \cite{reddi2016stochastic,reddi2016proximal} & {$\displaystyle O\left (\frac{1}{n^{2/3} L} \right )$}  & {$\displaystyle O\left (  \frac{n^{2/3} L}{\varepsilon^2} \right )$} & {$\displaystyle O\left ( \sqrt{n} \frac{ L}{\varepsilon} \right )$}\\
\hline
\autonewalgo-SAGA & $\displaystyle O \left ( \frac{1}{L}
                              \right ) $& {$ \displaystyle \tilde{O} \left ( \frac{n L}{\varepsilon^2} \right )$} & {$\displaystyle \tilde{O} \left ( \sqrt{n} \sqrt{ \frac{L}{\varepsilon} }\right )$}  \\
\hline
FG & $\displaystyle O \left ( \frac{1}{L} \right )$
           & $\displaystyle O \left (n \frac{L}{\varepsilon^2} \right
             )$ & $\displaystyle O \left ( n \frac{L}{\varepsilon} \right )$ \\
\hline
\autonewalgo-FG &  $\displaystyle O \left ( \frac{1}{L}
                              \right ) $ & $\displaystyle \tilde{O} \left ( n\frac{L}{\varepsilon^2}
                              \right ) $  & $\displaystyle O \left (n \sqrt{\frac{L}{\varepsilon}}
                              \right ) $ \\
\hline
\begin{minipage}{0.3\textwidth} \begin{center} Rand. CD \\ \cite{Wright_CD,Schmidt_CD,Nesterov_CD,Richtarik_CD} \end{center} \end{minipage} & $\displaystyle O \left ( \frac{1}{L_{\max}} \right )$
           & not avail. & $\displaystyle O \left (p \frac{L_{\max}}{\varepsilon} \right )$ \\
\hline
\autonewalgo-Rand. CD &  $\displaystyle O \left ( \frac{1}{L_{\max}}
                              \right ) $ & $\displaystyle \tilde{O} \left ( p^2\frac{L_{\max}}{\varepsilon^2}
                              \right ) $  &$\displaystyle O \left (p \sqrt{\frac{L_{\max}}{\varepsilon}}
                              \right ) $ \\
\hline
\end{tabular}
\caption{Comparison of rates of convergence, before and after the
  \autonewalgo~, resp. in the non-convex and convex cases. For the
  comparison, in the convex case, we only present the number of
 iterations to obtain a point $x$ satisfying $f(x) - f^* <
 \varepsilon$. In the non-convex case, we show the number of
 iterations to obtain a point $x$ satisfying $ \text{dist}(0, \partial
 f(x)) <\varepsilon$. }\label{ncvx_catalyst:tbl: exist}
   \vspace*{-0.3cm}
\end{table}

The original SVRG \cite{proxsvrg} has no guarantee for nonconvex functions. However a nonconvex extension of SVRG was proposed in~\cite{reddi2016proximal}. Their convergence rate gives a better dependence on $n$ compared to ours, namely $O(\frac{n^{2/3} L }{\epsilon^2})$. This is achieved thanks to a mini-batching strategy. In order to obtain a similar dependency on $n$, we need a tighter bound for SVRG with mini-batching applied to $\mu$-strongly convex problems, namely $O\left (\left ( n^{2/3}+\frac{L}{\mu} \right ) \log\left (\frac{1}{\varepsilon} \right )  \right )$. To the best of our knowledge, such a rate is currently unknown. Therefore, for ncvx-SVRG, we present the results without mini-batching. With mini-batching, the same convergence rate can be obtained by using a batch size $b =n^{2/3}$ and a stepsize $O(1/L)$. Similarly for ncvx-SAGA. For Rand. CD, we present the results for a smooth function $f$, with $L_{\max}$ the max. of the coordinate-wise Lipschitz constants for $\nabla f$ and $p$ is the dimension of the domain of $f$.

\subsection{Practical parameter choices and convergence rates}
The smoothing parameter
$\smthparacvx$ drives the convergence rate of \autonewalgosp
in the convex setting. To determine
$\smthparacvx$, we pretend $\weakcnx = 0$ and compute the
global complexity of our scheme. As such, we end up with the same
complexity result as Catalyst \cite{catalyst}. Following their work, the rule
of thumb is to maximize the ratio $\tau_\smthpara/\sqrt{L + \smthpara}$ for convex problems.
On the other hand, the choice of $\smthparainit$ is independent of
$\mathcal{M}$; it is an initial lower estimate for the weak convexity constant
$\weakcnx$. In practice, we typically choose $\kappa_0 = \smthparacvx$; 
For incremental approaches a natural heuristic is also to choose $S=T=n$,
meaning that $S$ iterations of~$\mtd$ performs one pass over the data.
In Table~\ref{table:parameters}, we present the values of $\smthparacvx$ used for various algorithms,
as well as other quantities that are useful to derive the convergence rates.

\paragraph{Full gradient method.} A first illustration is the
algorithm obtained when accelerating the regular ``full'' gradient
(FG). Here, the optimal choice for $\smthparacvx$ is
$L$. In the convex setting, we get an accelerated rate of $O(n
\sqrt{L/\varepsilon} \log(1/\varepsilon))$ which agrees with Nesterov's accelerated
variant (AFG) up to logarithmic factors. On the other hand, in the
nonconvex setting, our approach achieves no worse rate than $O(n
L/\varepsilon^2 \log(1/\varepsilon) )$, which agrees with the standard gradient
descent up to logarithmic factors. We note that under stronger assumptions,
namely $C^2$-smoothness of the objective, the accelerated algorithm in
\cite{carmon2017convex} achieves the same rate as (AFG) for the convex setting
and $O(\varepsilon^{-7/4} \log(1/\varepsilon))$ for the 
nonconvex setting. Their
approach, however, does not extend to composite setting nor to stochastic methods. Our marginal loss is
the price we pay for considering a much larger class of functions. 

\paragraph{Randomized Coordinate Descent (Rand. CD).} Next, we consider
\autonewalgosp applied to randomized coordinate descent (Rand. CD)
\cite{Wright_CD,Schmidt_CD,Nesterov_CD,Richtarik_CD}, see
\cite{Wright_CD} for more references. We examine the
Rand. CD method in a simplified setting, namely, $\displaystyle \min_{x \in \R^p}~f(x)$
where $f$ is smooth and $|\nabla f(x+te_i)_i - (\nabla f(x))_i| <
L_i|t|$. We use the Rand. CD algorithm described in
\cite[Algorithm 3, Theorem 1]{Wright_CD}.  The Lipschitz constant
                is $L = pL_{\max}$. Following \cite{catalyst}, the optimal choice for $\smthparacvx =
\max_{i=1, \hdots, p} |L_i| := L_{\max}$. The relationship between the Lipschitz constants are $L_{\max} \le L \le
p L_{\max}$; see \cite{Wright_CD}. Under our procedure,
\autonewalgosp attains an accelerated rate of $\tilde{O}(p
\sqrt{L_{\max}/\varepsilon})$, matching (up to log factors) the guarantees of the
accelerated randomized coordinate descent in
\cite[Algorithm 4]{Wright_CD} for the convex setting. A direct
implementation of Rand. CD has no convergence
guarantees in the non-convex setting. 

\paragraph{Randomized incremental gradient.} We now consider
randomized incremental gradient methods such as SAGA \cite{saga} and
(prox) SVRG \cite{proxsvrg}. Here, the optimal choice for
$\smthparacvx$ is $O(L/n)$. Under the convex setting, we
achieve an accelerated rate of $O(\sqrt{n} \sqrt{L/\varepsilon}
\log(1/\varepsilon))$. A direct application of SVRG and SAGA have no
convergence guarantees in the non-convex setting. With our approach,
the resulting algorithm matches the guarantees for FG up to log
factors.

\begin{table}[hbtp]\label{table:parameters}
   \centering
\renewcommand{\arraystretch}{1.2}
\begin{tabular}{|c | m{6.25cm} | c | c | c| c|}
	\hline
   Variable & Description & GD & Rand. CD & SVRG & SAGA \\
\hline
   $1/\tau_L$ & linear conv. param. with $\smthpara =
                L$ &
                                                                    $2$
                               & $p+1$ & $n+2$ & $4n$ \\ 
\hline
   $\smthparacvx$ & smoothing param. for convex setting & $L$ & $L_{\max}$ &
                                                                   $L/(n-1)$ & $3L/(4n-3)$ \\
\hline
$1/\tau_{\smthpara_\text{cvx}}$ & linear conv. param. with
                                $\smthparacvx$ & $2$  & $2p$ & $2n$ & $4n$ \\
\hline
   $A_{4L}$ & constant from conv. rate of $\mathcal{M}$ &
                                                                    $8L$
                               & $8pL_{\max}$  & $8L$ & $8Ln$ \\
\hline
\end{tabular}
   \caption{Values of various quantities that are useful to derive the
     convergence rate of the different optimization methods.
       For Rand. CD, we only consider the smooth setting. In
       particular, $L_{\max}$ is the max. of the coordinate Lipschitz
       constants for $\nabla f(x)$ and $p$ is the dimension of
       the domain of $f$.}
\end{table}

\subsection{Detailed derivation of convergence rates}
Using the values of Table~\ref{table:parameters}, we may now specialize 
our convergence results to different methods. Many of the linearly
convergent methods (e.g. Rand. CD and incremental methods) state
convergence results in terms of function values instead of subgradients
as in Equation \eqref{eq:criteria}. We relate function values to subgradients by using the Lipschitz constant $L$:
\[\text{dist}^2(0, \partial
f(x) ) \le 2L (f(x) - f(x^*)).\]

\paragraph{Gradient descent.} 
To compute the parameters $\tau_L$, $\smthparacvx$, etc, we use the
convergence analysis from \citep[][]{nesterov} for full gradient:

\begin{theorem}[Convergence guarantee for FG: Theorem 2.1.15 in \cite{nesterov}] Suppose
  the function $f : \R^p \to \R$ is $\mu$-strongly convex and has
  $L$-Lipschitz continuous gradient. Then the gradient descent method
  (FG) with stepsize $\delta = \frac{2}{\mu + L}$ generates a sequence
  $\{x_k\}$ such that 
\[ f(x_k) -f(x^*) \le \frac{L}{2} \left (\frac{L/\mu-1}{L/\mu + 1}
  \right )^{2k} \|x_0-x^*\|^2\],
where $x^*$ is the optimal solution of $f$. 
\end{theorem}

With this result, the number of iterations in
the inner loop are
\begin{align*}
T &\geq 2 \log (320) \\
S\log(\cnt +1 )& \geq 2 \log \left ( 64 (\cnt + 1)^2 \right ).
\end{align*}
The global complexity for gradient descent is 
{\small \begin{enumerate}
\item Algorithm~\ref{alg:adap-at-catalyst} will generate a point $x$ satisfying
$\text{dist} \big (0, \partial f(x) \big ) \leq \varepsilon$ after at most
\begin{align*}
O \left [  \frac{nL(f(x_0)-f^*)}{\varepsilon^2}\cdot \log \left ( \frac{L^2(f(x_0)-f^*)^2}{\varepsilon^4}
  \right ) +  n \log \left ( \frac{L}{\smthparainit} \right ) \right ]
\end{align*}
gradient computations. 
\item If $f$ is convex, then Algorithm~\ref{alg:adap-at-catalyst} will generate a point $x$
	satisfying  $\text{dist} \big (0, \partial f(x) \big ) \leq \varepsilon$ after at
	most 
\begin{align*}
O \Bigg [  \frac{n L^{2/3} \norm{x_0-x^*}^{2/3}}{\varepsilon^{2/3}}  \cdot \log \left  (\frac{L^{4/3}
                                   \norm{x_0-x^*}^{4/3}}{\varepsilon^{4/3}} \right ) +  n \log \left ( \frac{L}{\smthparainit} \right )\Bigg ]
\end{align*}
	gradient computations. 
\item If $f$ is convex, then Algorithm~\ref{alg:adap-at-catalyst} will
	 generate a point $x$ satisfying  $f(x)-f^*\leq \varepsilon$
	after at most
	$$O\left [  \frac{n\sqrt{L} \norm{x^*-x_0} }{\sqrt\varepsilon} 
          \cdot \log\left (\frac{L\norm{x_0-x^*}^2 }{\varepsilon}
  \right ) +n \log \left ( \frac{L}{\smthparainit}\right ) \right]
	$$
	gradient computations.
\end{enumerate} }

\paragraph{Rand. CD.} We use the convergence analysis from
\cite{Wright_CD}[Algorithm 3, Theorem 1]:

\begin{theorem}[Convergence guarantee for Rand. CD: Theorem 1 in \cite{Wright_CD} ]
Suppose the function $f: \R^p \to \R$ is $\mu$-strongly convex and each component
has an $L_i$-Lipschitz continuous gradient, namely for all $x \in \R^p$
and all $t \in \R$ we have
\[ \big | [\nabla f(x+te_i)]_i- [\nabla f(x)]_i \big | \le L_i |t|. \]
Set $L_{\max} = \max_{i =1, \hdots, p} L_i$. Then the iterates of
Algorithm 3 in \cite{Wright_CD} satisfy
\[ \mathbb{E} [f(x_k)]-f^* \le \left (1- \frac{\sigma}{p L_{\max}} \right )^k (f(x_0)-f^*). \]
\end{theorem}

For the Rand. CD, the number of iterations in the inner loop are
{\small \begin{align*}
T &\geq (p+1)\log (320)  \\
S\log(\cnt +1) &\geq 
		2p \log \left ( 64 \cdot p(1+p)
                   \cdot (k+1)^2 \right ) .
\end{align*}}

The global complexity for Rand. CD is
{\small \begin{enumerate}
\item Algorithm~\ref{alg:adap-at-catalyst} will generate a point $x$ satisfying
$\text{dist} \big (0, \partial f(x) \big ) \leq \varepsilon$ after at most
\begin{align*}
O \left [  \frac{p^2L_{\max}(f(x_0)-f^*)}{\varepsilon^2}\cdot \log \left ( \frac{(1+p)p^5L^2_{\max}(f(x_0)-f^*)^2}{\varepsilon^4}
  \right ) +  p \log \left ( \frac{pL_{\max}}{\smthparainit} \right ) \right ]
\end{align*}
gradient computations. 
\item If $f$ is convex, then Algorithm~\ref{alg:adap-at-catalyst} will generate a point $x$
	satisfying  $\text{dist} \big (0, \partial f(x) \big ) \leq \varepsilon$ after at
	most 
\begin{align*}
O \Bigg [  \frac{p^{4/3} L_{\max}^{2/3}
  \norm{x_0-x^*}^{2/3}}{\varepsilon^{2/3}}  \cdot \log \left
  (\frac{(1+p)p^{11/3} L_{\max}^{4/3}
                                   \norm{x_0-x^*}^{4/3}}{\varepsilon^{4/3}}
  \right ) +  p \log \left ( \frac{pL_{\max}}{\smthparainit} \right )\Bigg ]
\end{align*}
	gradient computations. 
\item If $f$ is convex, then Algorithm~\ref{alg:adap-at-catalyst} will
	 generate a point $x$ satisfying  $f(x)-f^*\leq \varepsilon$
	after at most
	$$O\left [  \frac{p\sqrt{L_{\max}} \norm{x^*-x_0} }{\sqrt\varepsilon} 
          \cdot \log\left (\frac{p(1+p) L_{\max} \norm{x_0-x^*}^2 }{\varepsilon}
  \right ) +p \log \left ( \frac{pL_{\max}}{\smthparainit}\right ) \right]
	$$
	gradient computations.
\end{enumerate} }

\paragraph{SVRG.} 
We use the convergence analysis established in \cite{proxsvrg}:
\begin{theorem}[Convergence guarantee for SVRG: Theorem 3.1 in
  \cite{proxsvrg} ] Suppose the function $1/n \sum_{i=1}^n f_i$ is
  $L$-Lipschitz and the function $f$ is $\mu$-strongly convex. Choose
  the real constant $0 < \theta < 1/4$ sufficiently small so that
\[\rho = \frac{1}{100 \theta (1-4 \theta)} + \frac{4 \theta \left ( \frac{L}{\mu}+1
    \right )}{100 \frac{L}{\mu} (1-4\theta)} < 1.\] 
Then the Prox-SVRG method in \cite{proxsvrg} has geometric convergence
in expectation:
\[ \mathbb{E} [f(x_k)] - f(x^*) \le \rho^k (f(x_0) - f(x^*) ). \]
In particular, each stage requires $n + 100L/\mu$ component gradient
evaluations so the overall complexity is
\[O \left ( (n + L/\mu) \log(1/\varepsilon) \right ).\]
\end{theorem}

For SVRG, the number of iterations in
the inner loop are
{\small \begin{align*}
T &\geq (n+2)\log (320)  \\
S\log(\cnt +1) &\geq 
		2n \log \left ( 64 \cdot n^2
                   \cdot (k+1)^2 \right ) .
\end{align*}}

The global complexity for SVRG when $n$ is sufficiently large is
{\small \begin{enumerate}
\item Algorithm~\ref{alg:adap-at-catalyst} will generate a point $x$ satisfying
$\text{dist} \big (0, \partial f(x) \big ) \leq \varepsilon$ after at
most
\begin{align*}
O \left [ \frac{nL(f(x_0)-f^*)}{\varepsilon^2} \cdot \log \left ( \frac{n^2L^2(f(x_0)-f^*)^2}{\varepsilon^4}
  \right ) + n \log \left ( \frac{L}{\smthparainit} \right ) \right ]
\end{align*}
gradient computations. 
\item If $f$ is convex, then Algorithm~\ref{alg:adap-at-catalyst} will generate a point $x$
	satisfying  $\text{dist} \big (0, \partial f(x) \big ) \leq \varepsilon$ after at
	most 
\begin{align*}
O \Bigg [  \frac{n^{2/3}L^{2/3} \norm{x^*-x_0}^{2/3}}{\varepsilon^{2/3}} \log \left (
                                                               \frac{n^{4/3} L^{4/3} 
                                                              \norm{x^*-x_0}^{4/3}    } {\varepsilon^{4/3}}
  \right )+  n^{2/3} \log \left ( \frac{L}{\smthparainit}\right ) \Bigg ]
\end{align*}
	gradient computations. 
\item If $f$ is convex, then Algorithm~\ref{alg:adap-at-catalyst} will
	 generate a point $x$ satisfying  $f(x)-f^*\leq \varepsilon$
	after at most
	$$O\left [ \frac{\sqrt{nL} \norm{x^*-x_0} }{\sqrt\varepsilon} 
          \cdot \log\left (\frac{nL\norm{x_0-x^*}^2}{\varepsilon}
  \right ) + \sqrt{n} \log \left ( \frac{L}{\smthparainit}\right )\right]
	$$
	gradient computations.
\end{enumerate} }

\paragraph{SAGA.} We observe that the variables for SAGA are the same as for SVRG up to a
multiplicative factors. Therefore, the global complexities results for SAGA
are, up to constant factors, the same as SVRG. 

\begin{theorem}[Convergence guarantee of SAGA \cite{saga} in
  Corollary 1]
Suppose the function $f: \R^p \to \R$ is $\mu$-strongly convex and
each $f_i$ has Lipschitz continuous derivatives with constant
$L$. Then the iterates $\{x_k\}$ generated by SAGA in \cite{saga} satisfy
\[\mathbb{E} \| x_k- x^* \|^2 \le \left (1 + \frac{2n}{3} \right ) \left ( 1 - \min \left \{
        \frac{1}{4n}, \frac{\mu}{3L} \right \} \right )^k
    \|x_0-x^*\|^2.\]
\end{theorem}
\section{Experiments}\label{sec:exp}
We investigate the performance of~\autonewalgosp on two standard
non-convex problems in machine learning, namely on sparse matrix
factorization and on training a simple two-layer neural network. 

\paragraph{Comparison with linearly convergent methods.} We report experimental results of~\autonewalgosp when applied to the incremental algorithms SVRG~\cite{proxsvrg} and SAGA~\cite{saga},
and consider the following variants:

\begin{itemize}
 \item ncvx SVRG/SAGA~\cite{reddi2016proximal,allen2016variance} with its theoretical stepsize $\eta =
   1/Ln^{2/3}$.
\item a minibatch variant of ncvx SVRG/SAGA~\cite{reddi2016proximal,allen2016variance} with batch size $b= n^{2/3}$ and stepsize $\eta = 1/L$.
\item SVRG/SAGA with large stepsize $\eta = 1/L$. This is variant of SVRG/SAGA, whose stepsize is not justified by theory for nonconvex problems, but which performs well in practice.
 \item 4WD-Catalyst SVRG/SAGA with its theoretical stepsize $\eta = 1/2L$.
\end{itemize}

The algorithm SVRG (resp. SAGA) was originally designed for minimizing
convex objectives. The nonconvex version was developed in \cite{reddi2016proximal,allen2016variance}, using a significantly smaller stepsize $\eta = 1/Ln^{2/3}$. Following~\cite{reddi2016proximal}, we also include in the comparison a heuristic variant that uses a large stepsize $\eta = 1/L$, where no theoretical guarantee is available for nonconvex objectives.
4WD-Catalyst SVRG and 4WD-Catalyst SAGA use a similar stepsize, but the Catalyst mechanism makes this choice theoretically grounded.

\paragraph{Comparison with popular stochastic algorithms.} We also include as baselines three
popular stochastic algorithms: stochastic gradient descent (SGD),
AdaGrad~\cite{adaGrad}, and Adam \cite{adam}.
\begin{itemize}[leftmargin=0.5cm,itemsep=0pt,topsep=0pt,parsep=0pt]
 \item SGD with constant stepsize.
 \item AdaGrad~\cite{adaGrad} with stepsize $\eta = 0.1$ or $0.01$. 
\item Adam \cite{adam} with stepsize $\alpha = 0.01$ or $0.001$,
  $\beta_1 = 0.9$, and $\beta_2 = 0.999$.
\end{itemize}
The stepsize (learning rate) of these algorithms are manually tuned to
output the best performance. Note that none of them, SGD,
AdaGrad~\cite{adaGrad}, or Adam \cite{adam} enjoys
linear convergence when the problem is strongly convex. Therefore, we do
not apply 4WD-Catalyst to these algorithms. 
SGD is used in both experiments, whereas AdaGrad and Adam are
used only on the neural network experiments and not on sparse matrix factorization since it is unclear how to apply it to a nonsmooth objective. 

\paragraph{Parameter settings.}
We start from an initial estimate of the Lipschitz constant $L$ and
use the theoretically recommended $\smthpara_0= \smthpara_{\text{cvx}}
= 2L/n$ in \autonewalgo.
We set the number of inner iterations $T=S=n$ in all experiments which
means making at most one pass over the data to solve each
sub-problem. Moreover, the $\log(k)$ dependency dictated by the theory
is dropped while solving the subproblem in~(\ref{eq:accer_adapt_2}). 
These choices turn out to be justified \textit{a posteriori}, as both SVRG and SAGA have a much better convergence rate in practice 
than the theoretical rate derived from a worst-case analysis. Indeed, in all experiments, one pass over the data to solve each sub-problem 
was found to be enough to guarantee sufficient descent. 

\paragraph{Sparse matrix factorization a.k.a. dictionary learning.}
Dictionary learning 
consists of representing a dataset $X=[x_1, \cdots, x_n] \in \R^{m \times n}$ 
as a product $X \approx D A$, where $D$ in~$\R^{m \times p}$ is called a dictionary, and
$A$ in $\R^{p \times n}$ is a sparse matrix. The classical non-convex formulation~\citep[see][]{mairal2014sparse} 
is 
\begin{equation*}
\min_{D \in \mathcal{C}, A \in \R^{p\times n}} \sum_{i=1}^n \frac{1}{2} \Vert x_i - D \alpha_i \Vert_2^2 +  \psi(\alpha_i),
\end{equation*}
where $A = [\alpha_1 \cdots \alpha_n]$ carries the decomposition coefficients
of signals $x_1 \cdots x_n$, $\psi$ is a sparsity-inducing regularization
and $\mathcal{C}$ is chosen as the set of matrices whose
columns are in the $\ell_2$-ball.
An equivalent point of view is the finite-sum problem
$\min_{D \in \mathcal{C}}  \frac{1}{n} \sum_{i=1}^n f_i(D)$ with 
\begin{equation}
   f_i(D) :=  \min_{\alpha \in \R^p} \frac{1}{2} \Vert x_i -D \alpha \Vert_2^2 + \psi(\alpha). \label{dl1}
\end{equation}
We consider the elastic-net regularization $\psi(\alpha) = \frac{\mu}{2} \Vert \alpha \Vert^2 + \lambda \Vert \alpha \Vert_1 $
 of~\cite{zou2005regularization}, which has a sparsity-inducing effect, and report the corresponding results in Figures~\ref{fig:patches_append_svrg} 
 and \ref{fig:patches_append_saga}. We learn
a dictionary in $\R^{m \times p}$ with $p=256$ elements on a set of whitened normalized
image patches of size $m=8 \times 8$.  Parameters are set to be as in~\citep{mairal2014sparse}---that is, a small value $\mu\!=\!1e-5$,
and~$\lambda\!=\!0.25$, leading to sparse matrices~$A$ (on average $\approx
4$ non-zero coefficients per column of~$A$).
Note that our implementations are based on the open-source SPAMS toolbox~\cite{mairal2010}.\footnote{available here \url{http://spams-devel.gforge.inria.fr.}}

\begin{figure*}[t!]
   \centering
   \includegraphics[width=.31\textwidth]{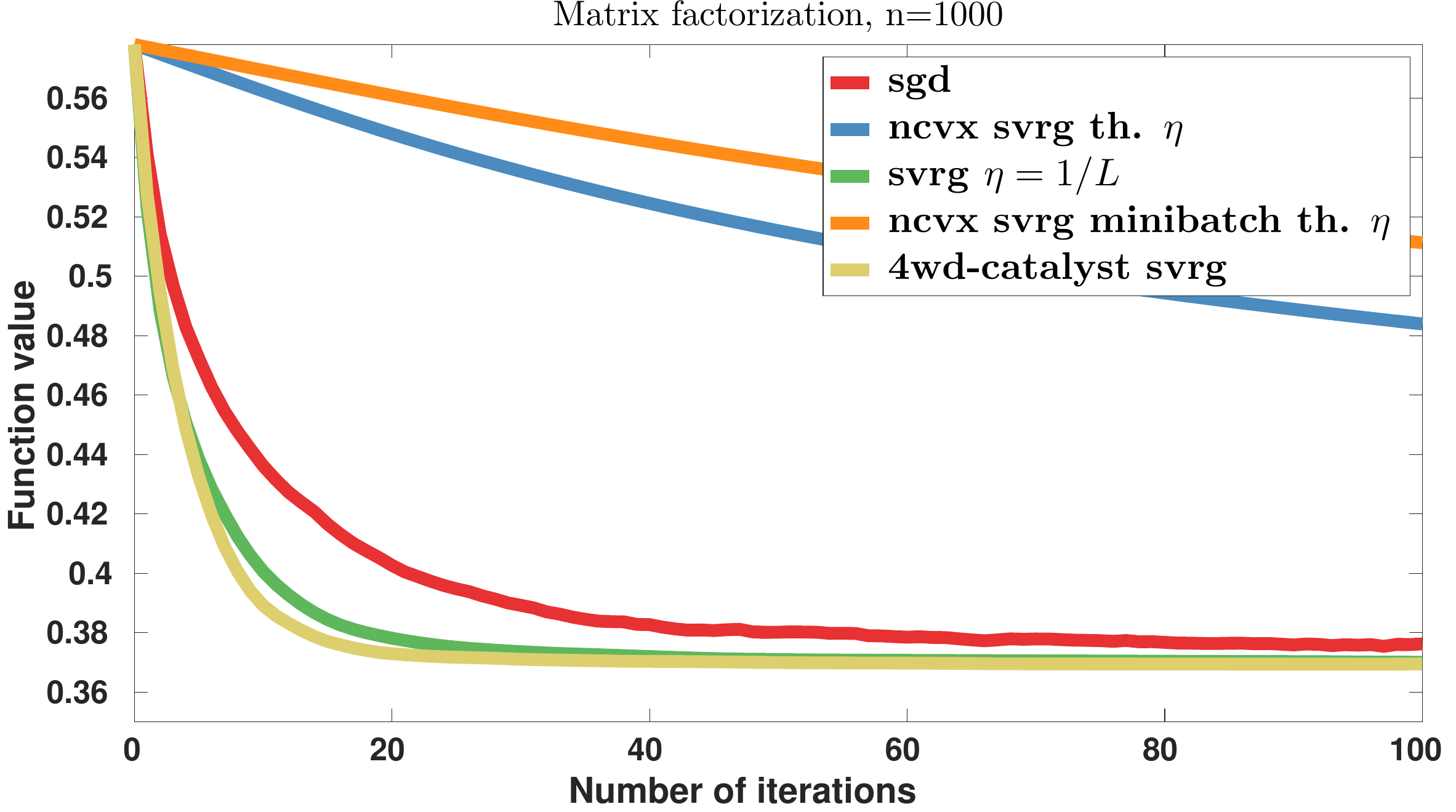}
   \includegraphics[width=.31\textwidth]{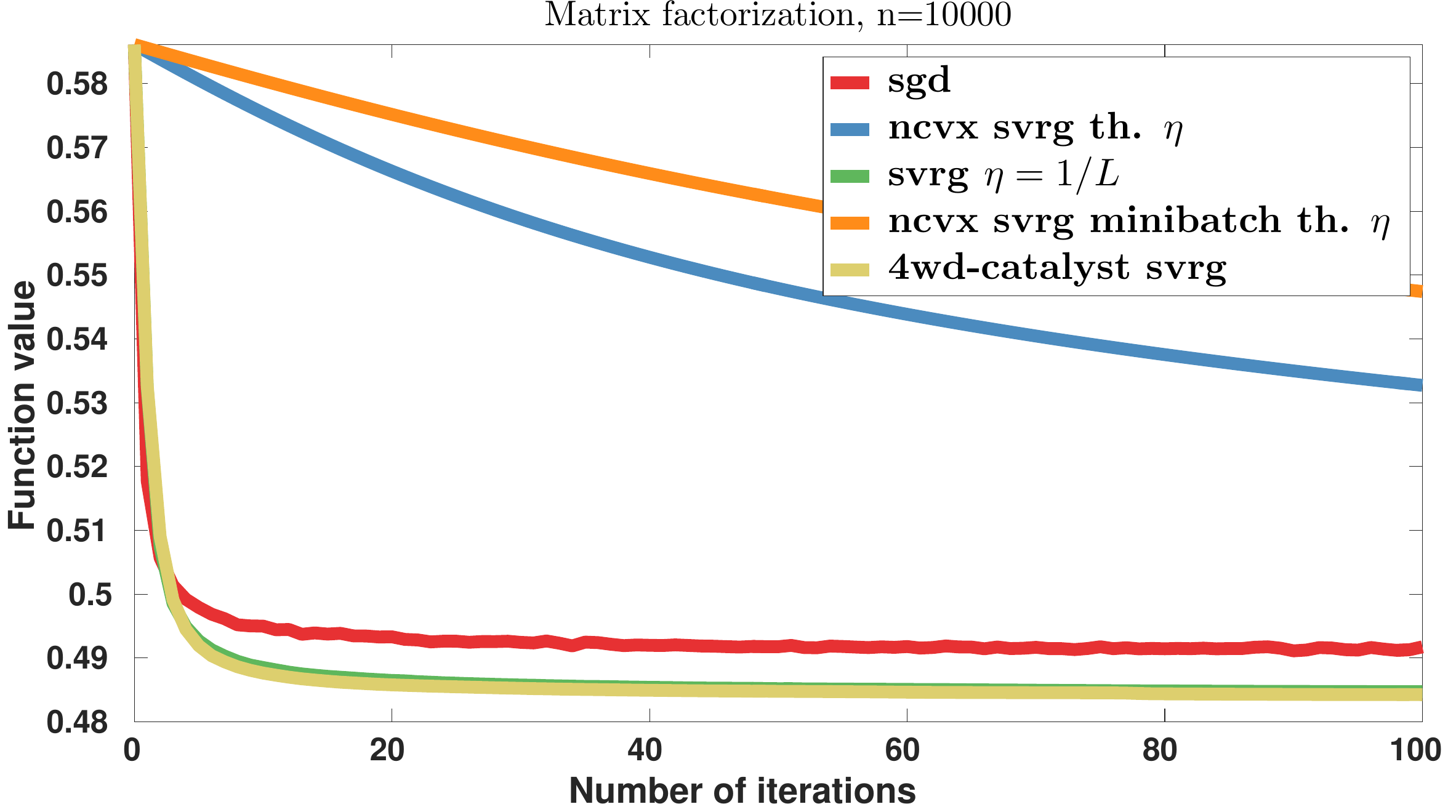}
   \includegraphics[width=.31\textwidth]{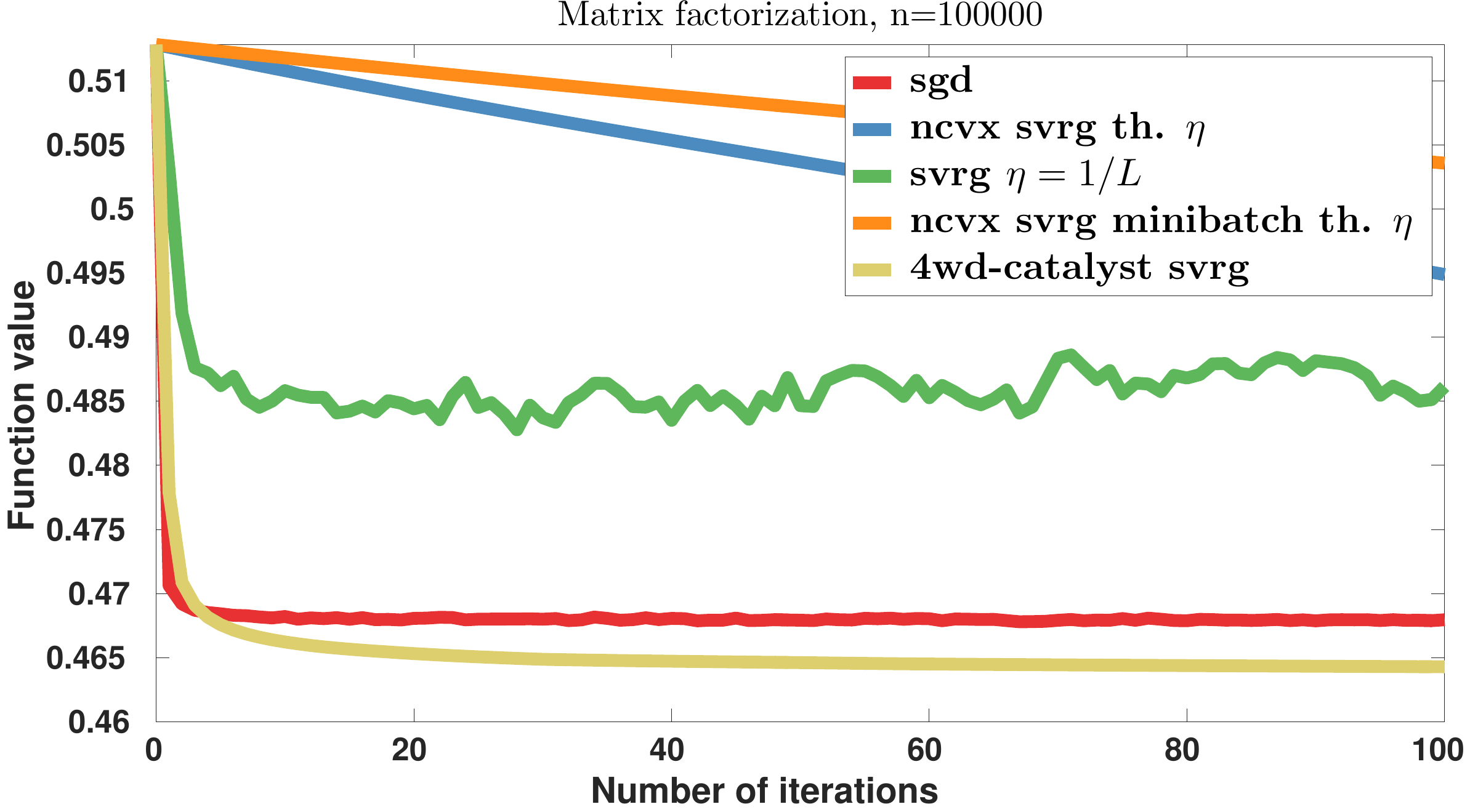}\\
   \includegraphics[width=.31\textwidth]{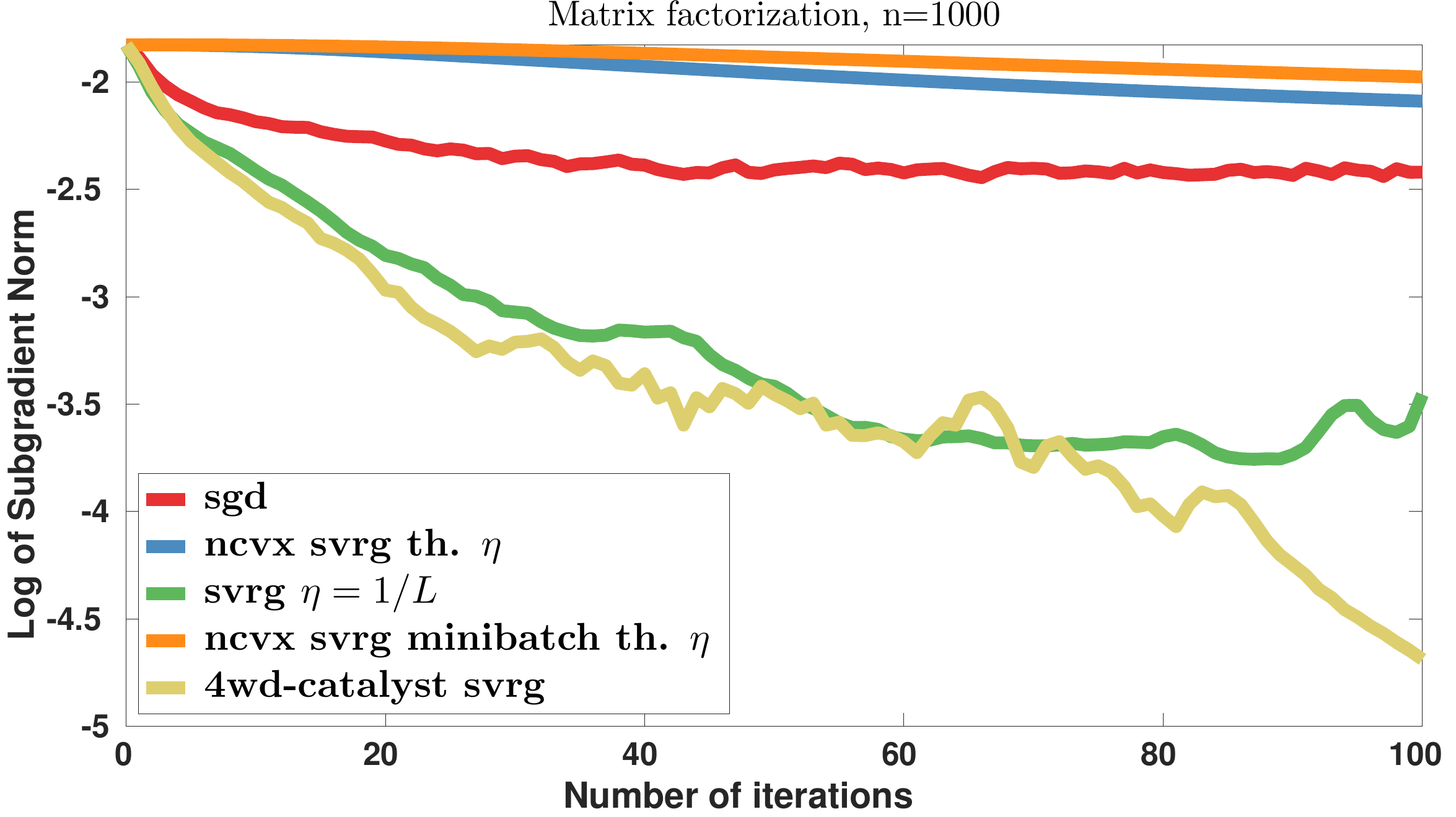}
   \includegraphics[width=.31\textwidth]{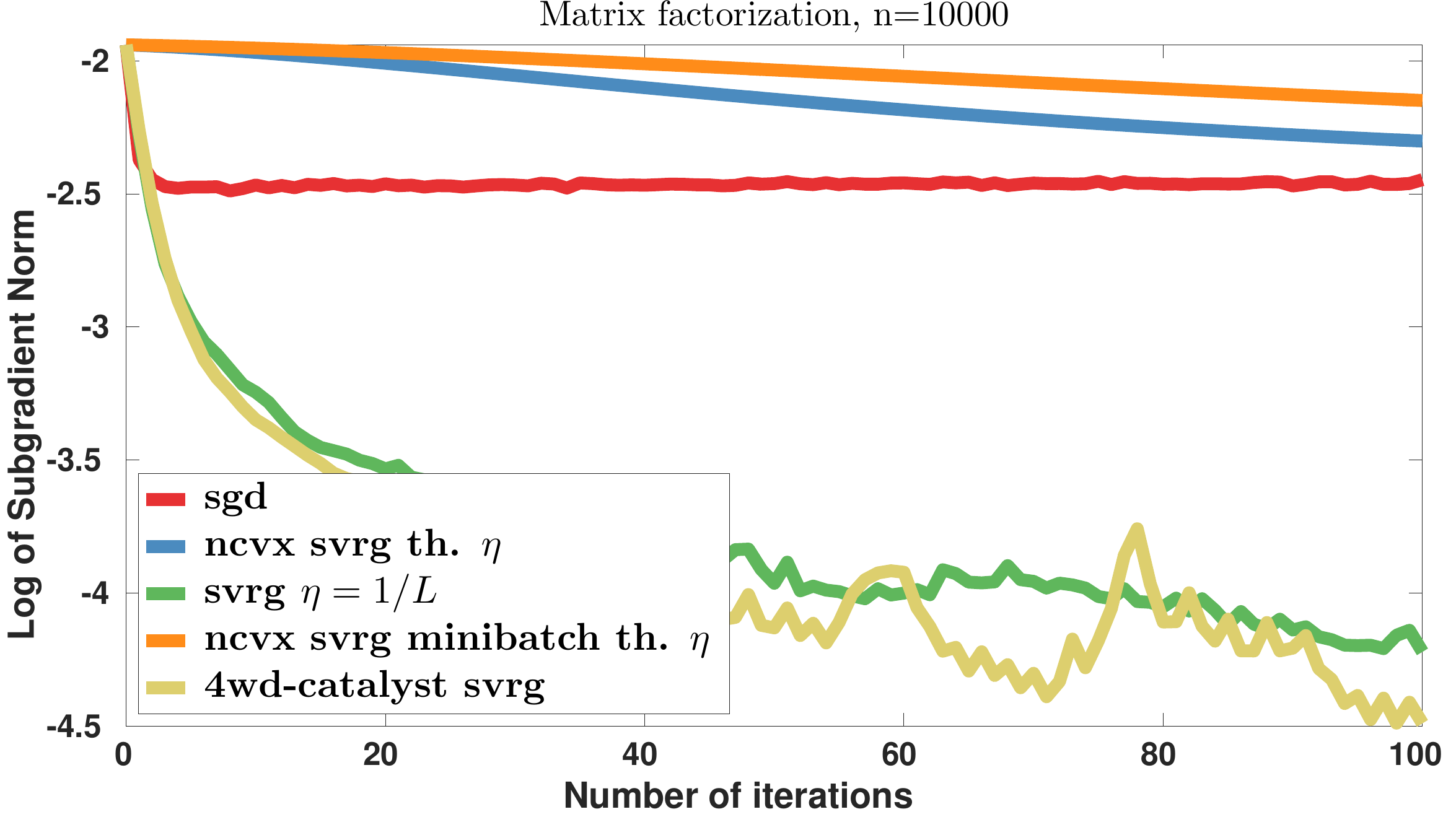}
   \includegraphics[width=.31\textwidth]{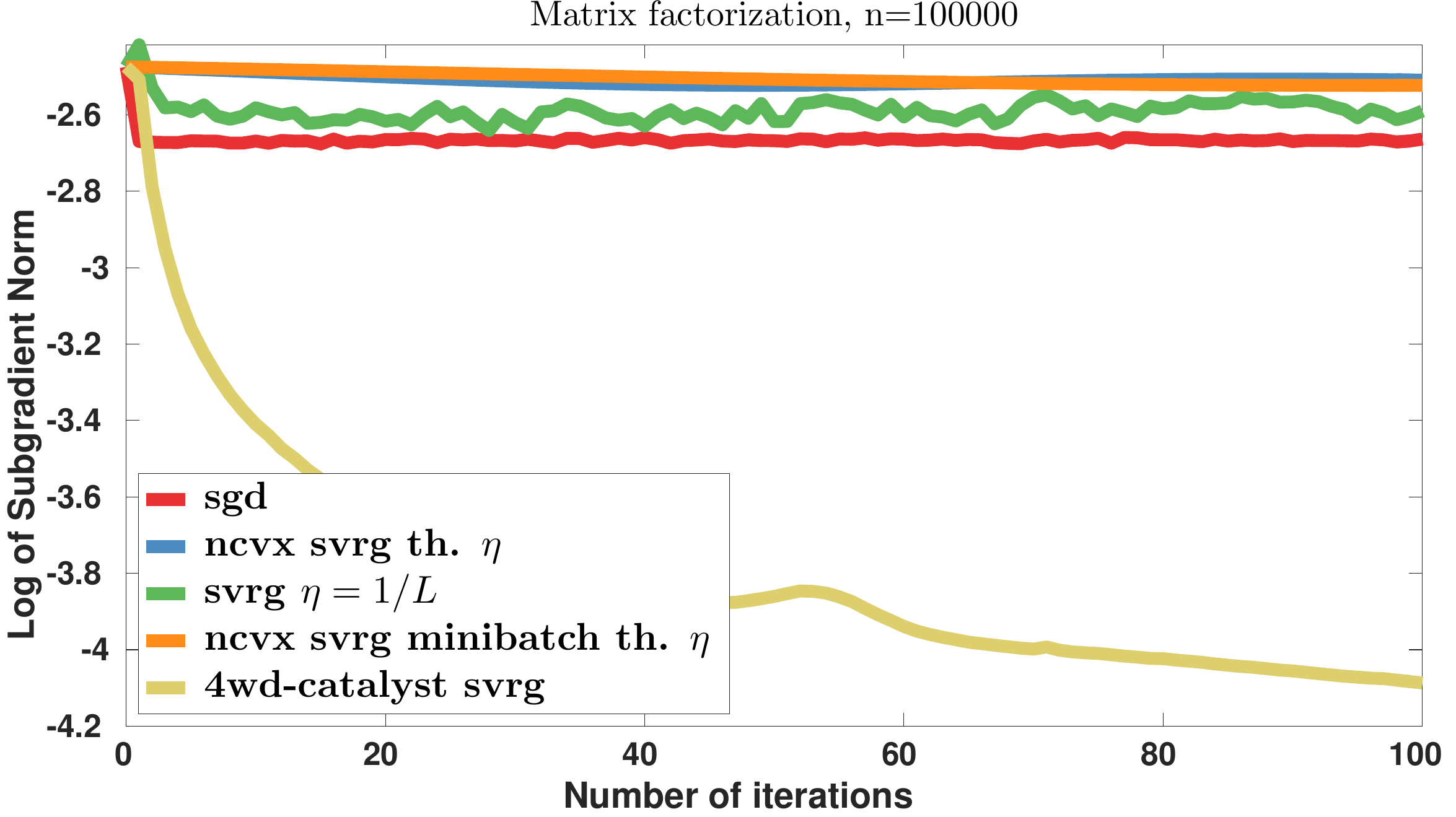}
   \vs
   \caption{Dictionary learning experiments using SVRG. We plot the function value (top) and the subgradient norm (bottom). From left to right, we vary the size of dataset from $n=1\,000$~~to~$n=100\,000$.}\label{fig:patches_append_svrg}
\end{figure*}

\begin{figure}[t!]
   \centering
   \includegraphics[width=.31\textwidth]{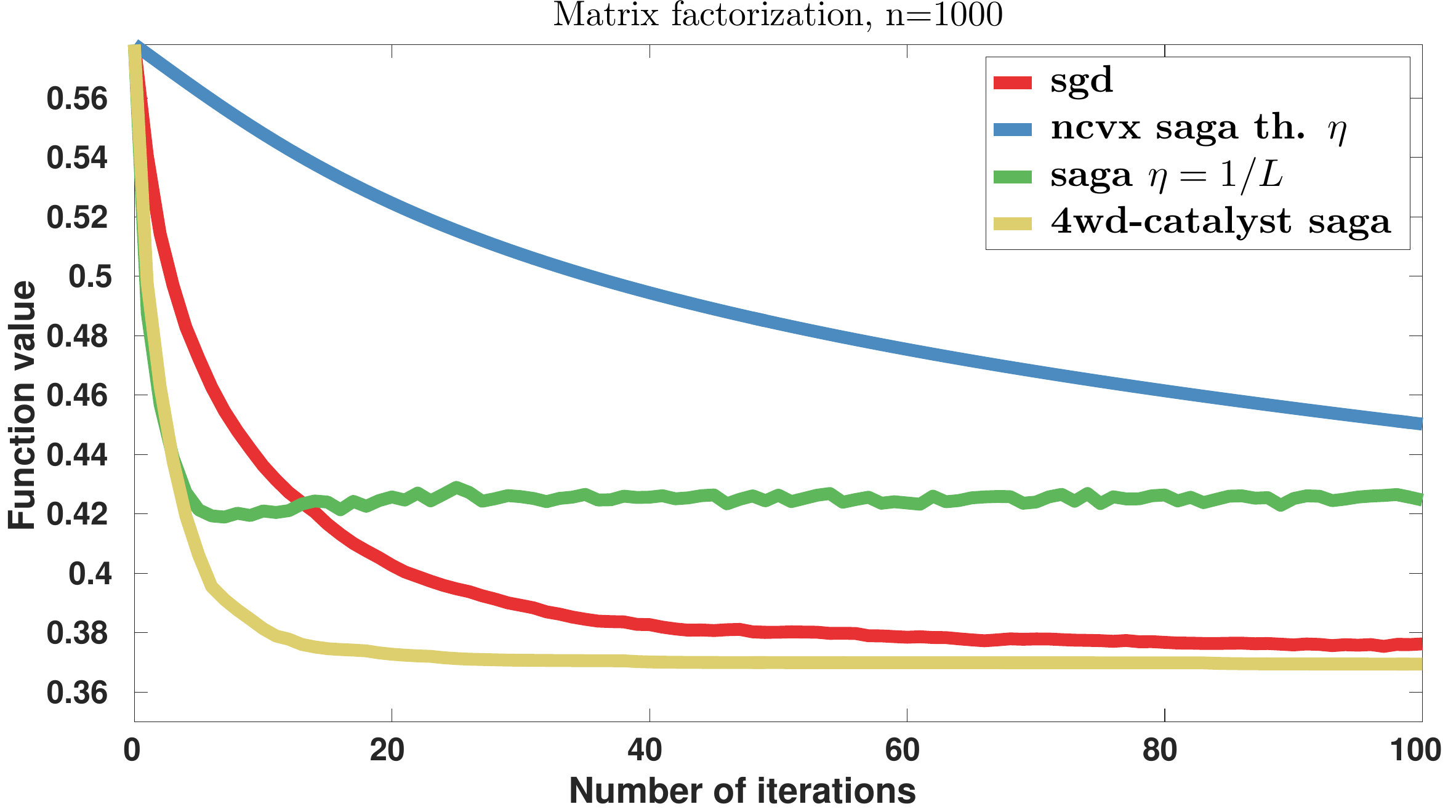}
   \includegraphics[width=.31\textwidth]{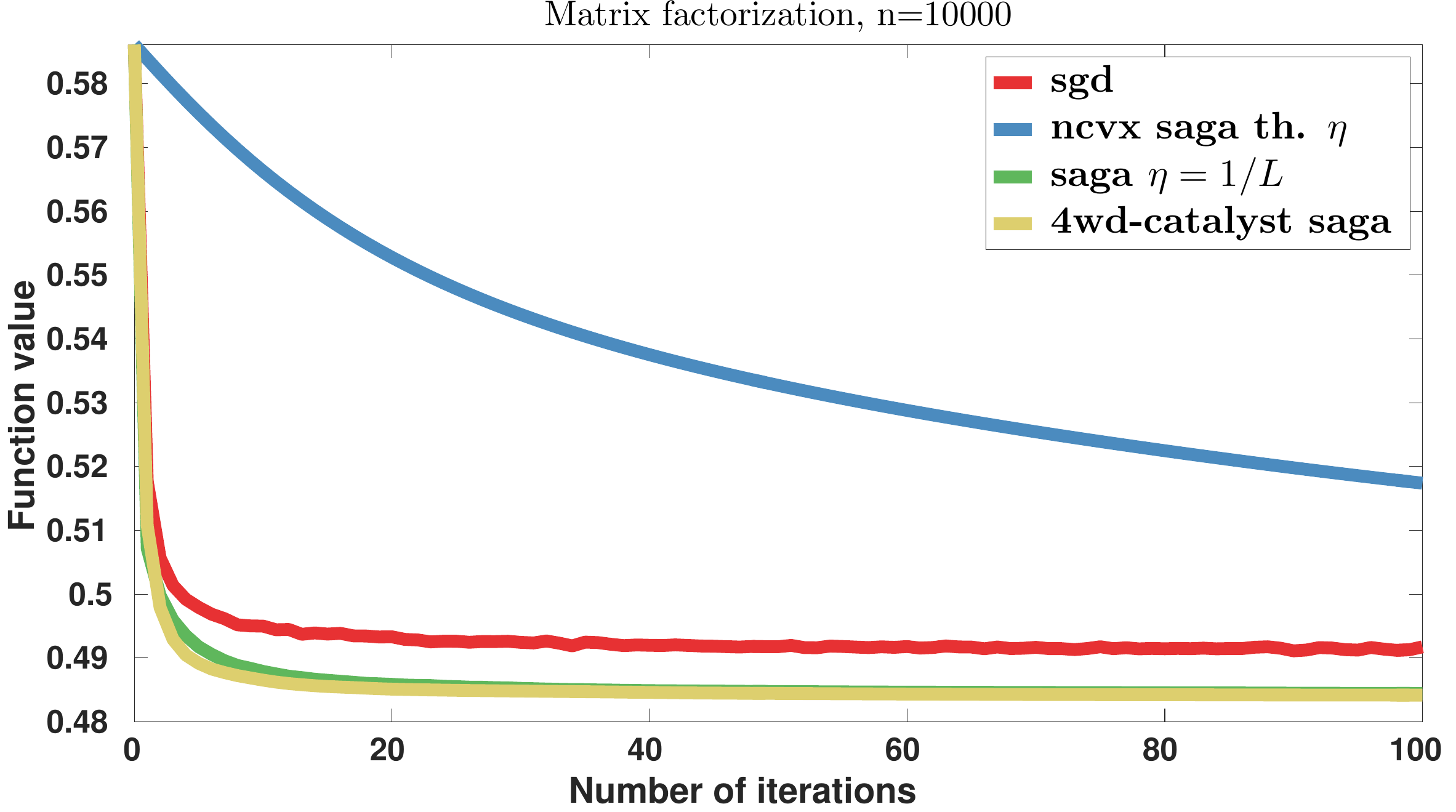}
   \includegraphics[width=.31\textwidth]{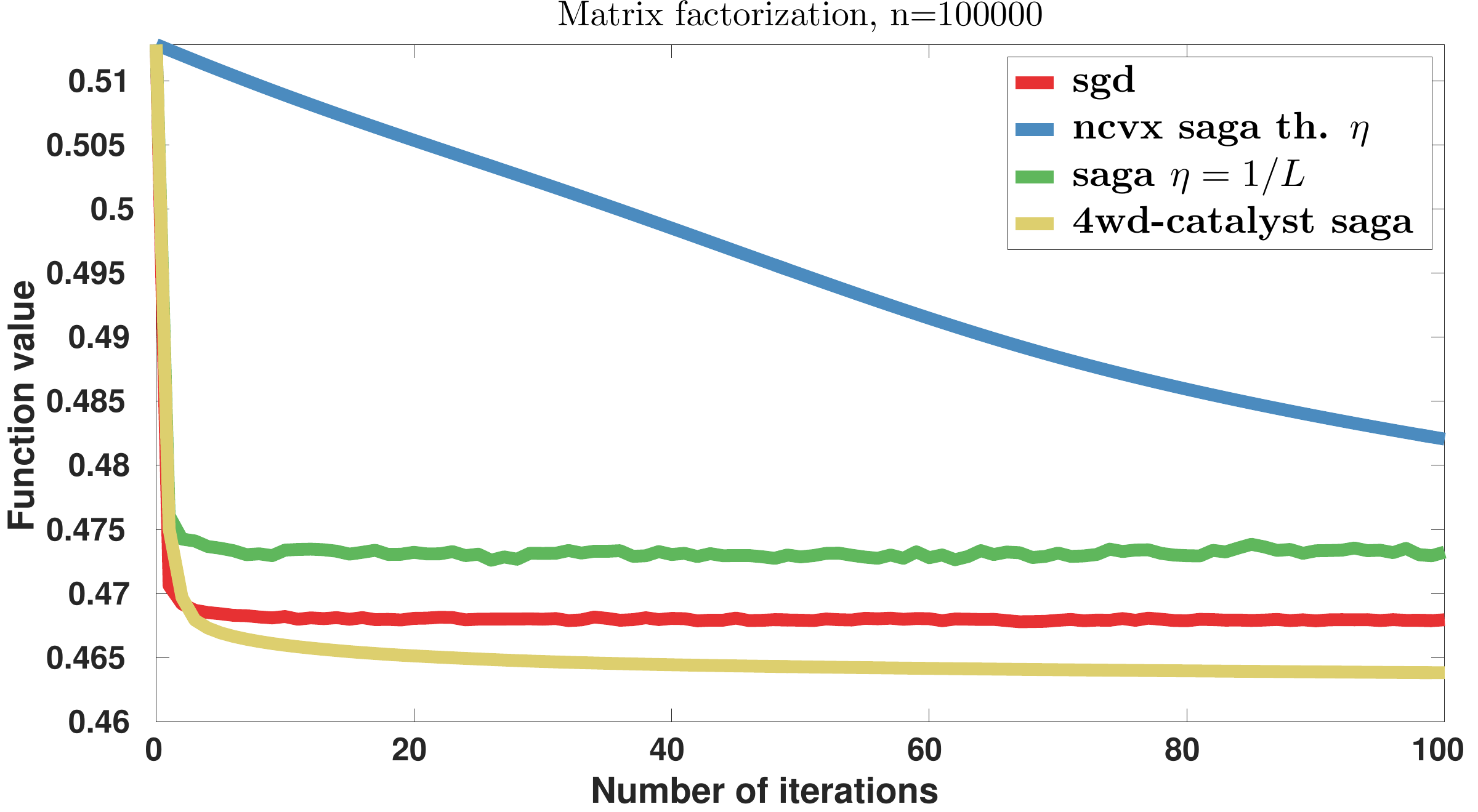}\\
   \includegraphics[width=.31\textwidth]{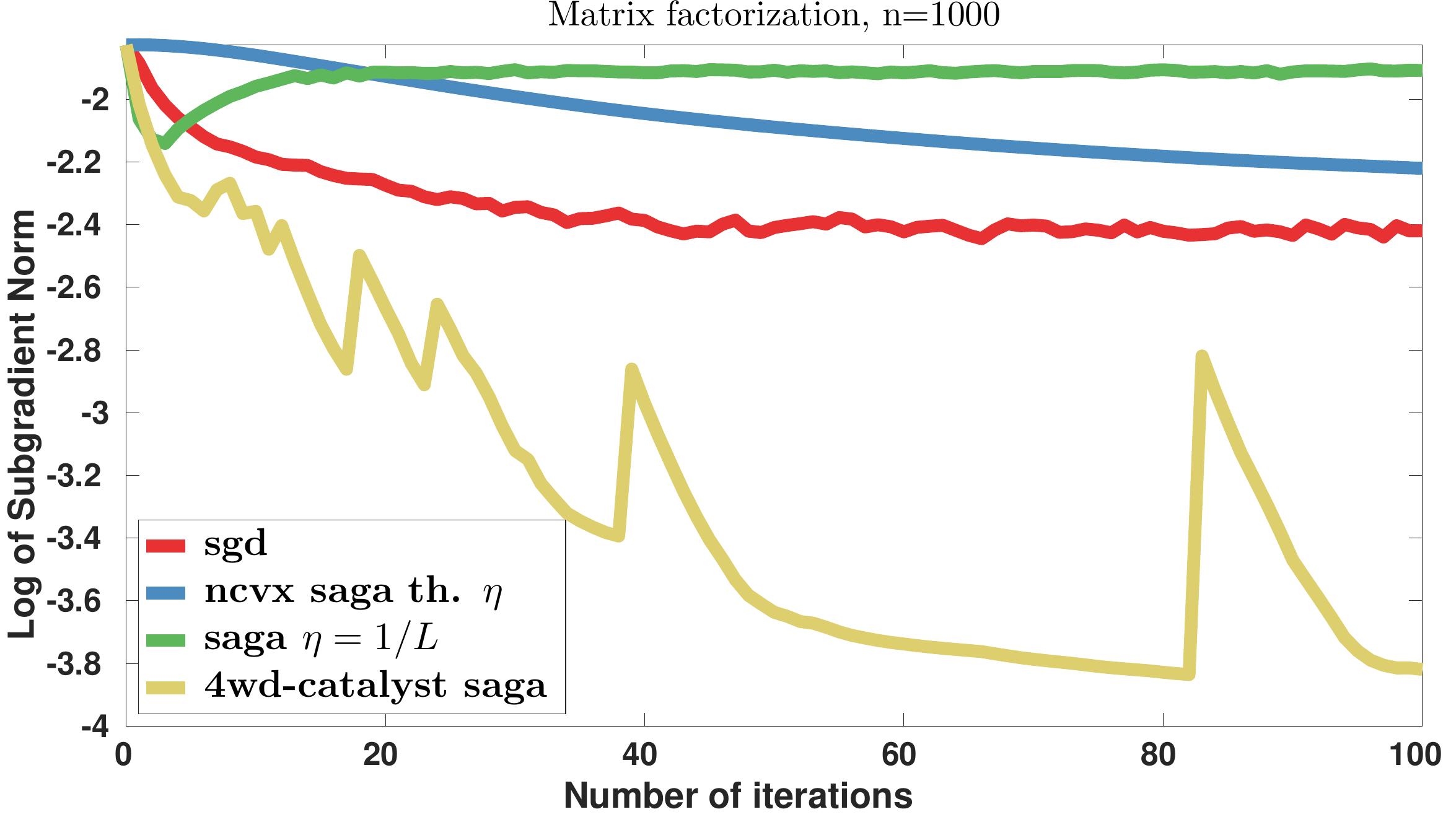}
   \includegraphics[width=.31\textwidth]{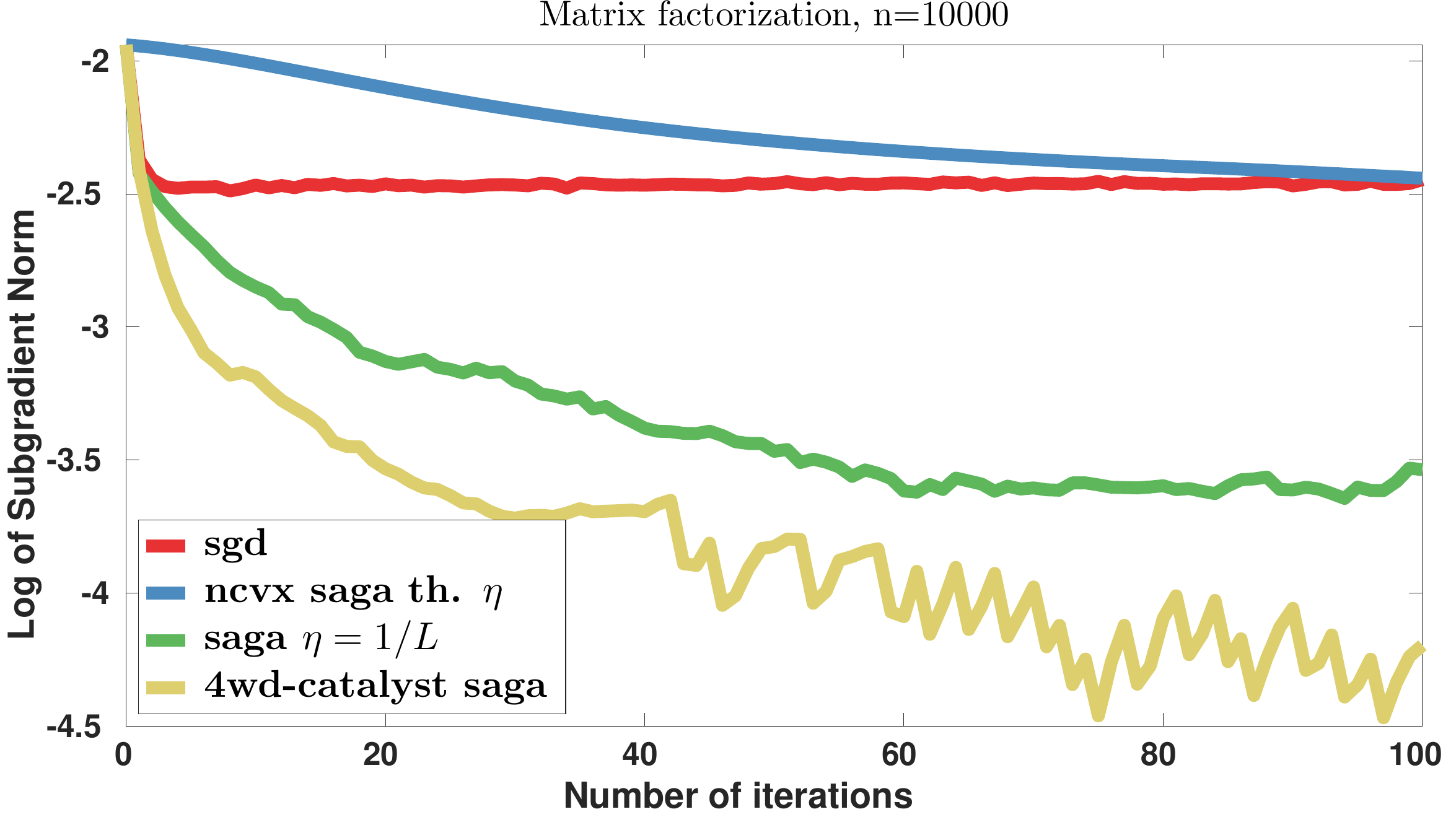}
   \includegraphics[width=.31\textwidth]{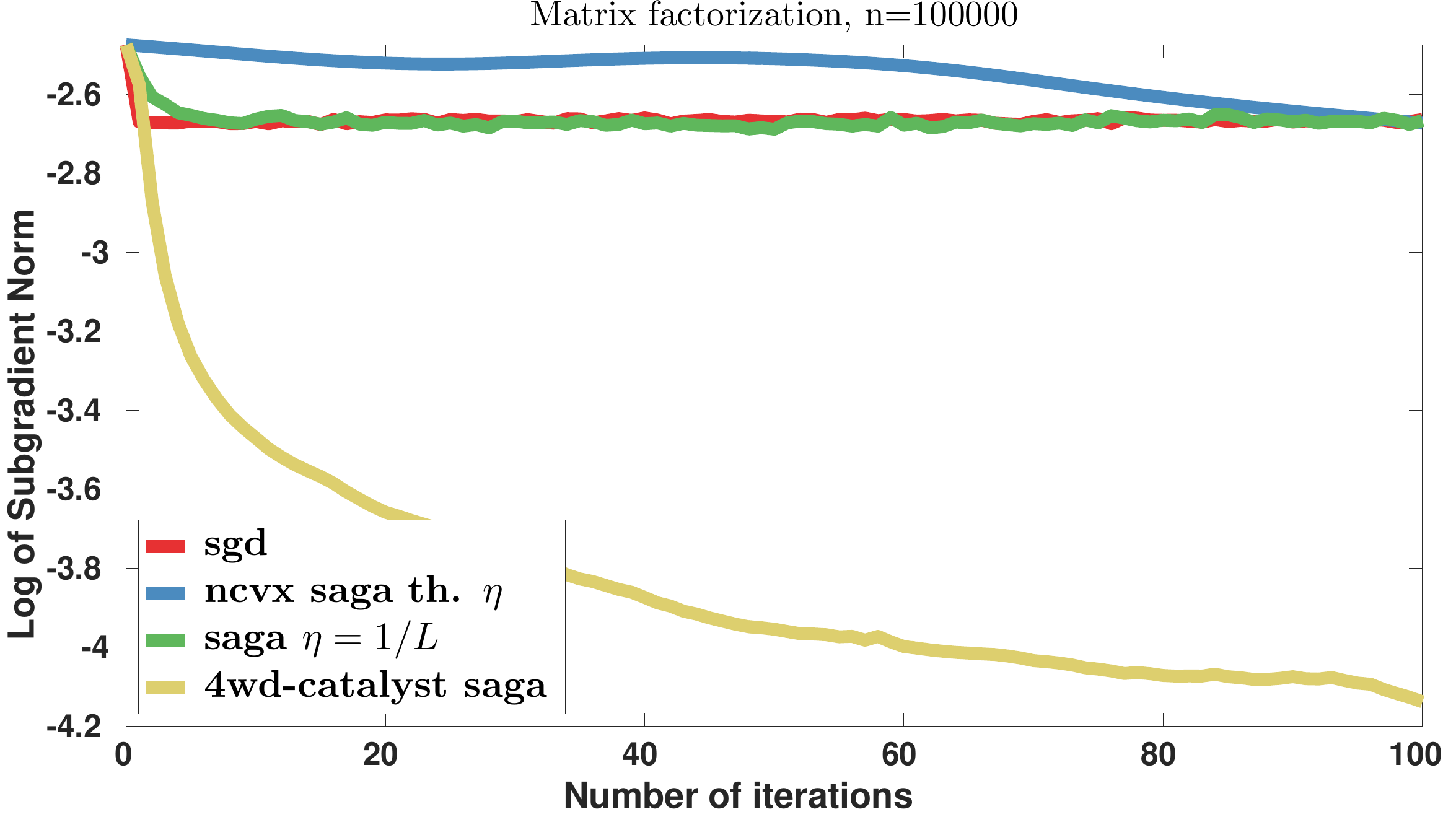}
   \caption{Dictionary learning experiments using SAGA. We plot the function value (top) and the subgradient norm (bottom). From left to right, we vary the size of dataset from $n=1\,000$~~to~$n=100\,000$.}\label{fig:patches_append_saga}
\end{figure}

\vs
\paragraph{Neural networks.}
We consider now simple binary classification problems for learning neural networks. 
Assume that we are given a training set $\{a_i,b_i\}_{i=1}^n$, where the
variables $b_i$ in $\{-1,+1\}$ represent class labels, and $a_i$ in $\R^p$ are
feature vectors.
The estimator of a label class is now given by a two-layer neural network
$\hat{b} = \text{sign}(w_2^\top \sigma ( W_1^\top a))$, where $W_1$
in~$\R^{p \times d}$ represents the weights of a hidden layer with~$d$
neurons, $w_2$ in~$\R^d$ carries the weight of the network's second layer,
and $\sigma(u) = \log(1+e^u)$ is a non-linear function, applied pointwise to its arguments.
We fix the number of hidden neurons to $d=100$ and use the logistic loss to fit the estimators to the true labels.
Since the memory required by SAGA becomes $n$ times larger than SVRG for nonlinear models, which is 
problematic for large $n$, we can only perform experiments with SVRG. The experimental results are reported
on two datasets~\textsf{alpha} and~\textsf{covtype} in Figures~\ref{fig:nn_append_alpha} 
and \ref{fig:nn_append_covtype}.

\begin{figure*}[t!]
   \includegraphics[width=.32\textwidth]{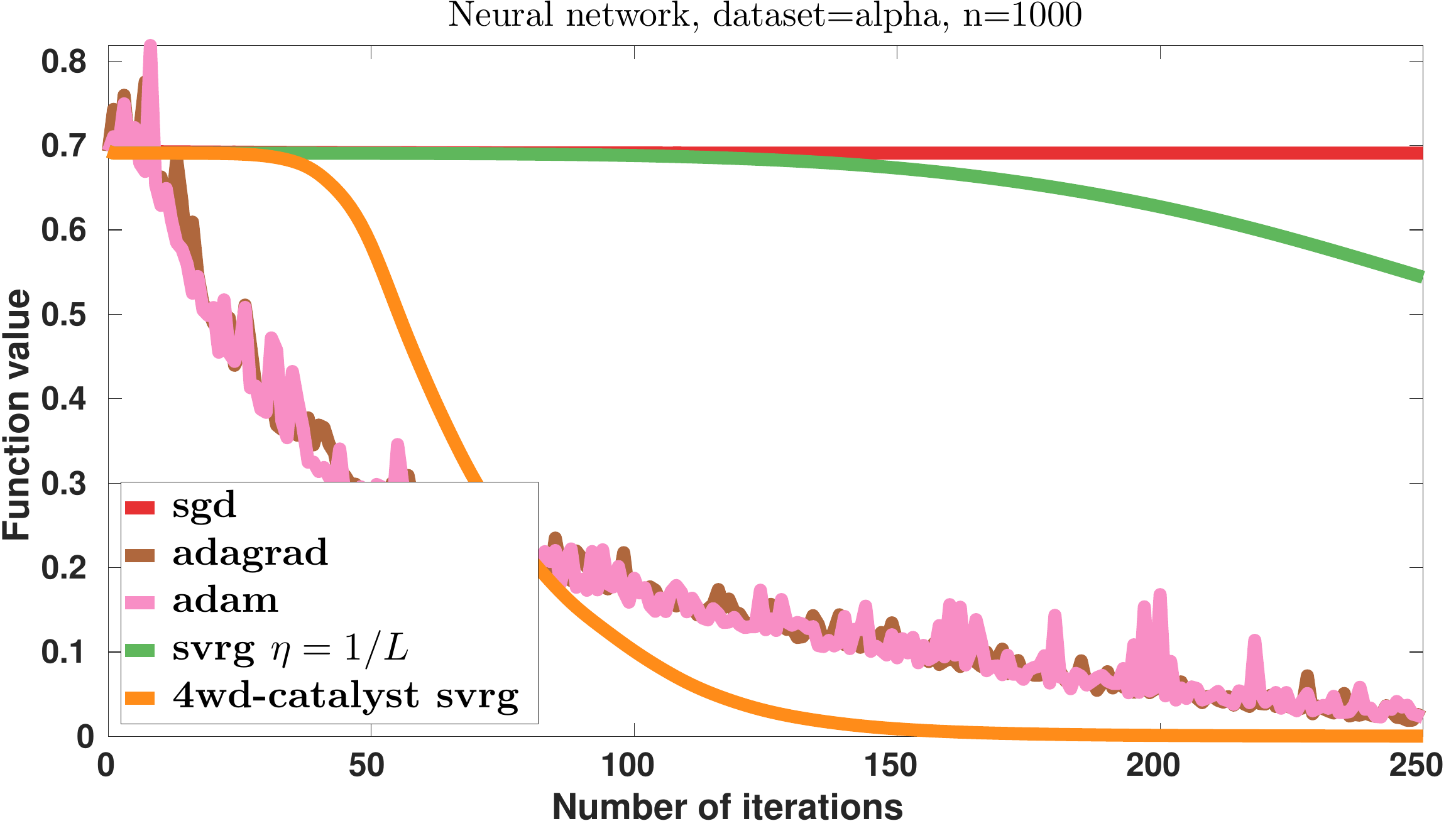}
   \includegraphics[width=.32\textwidth]{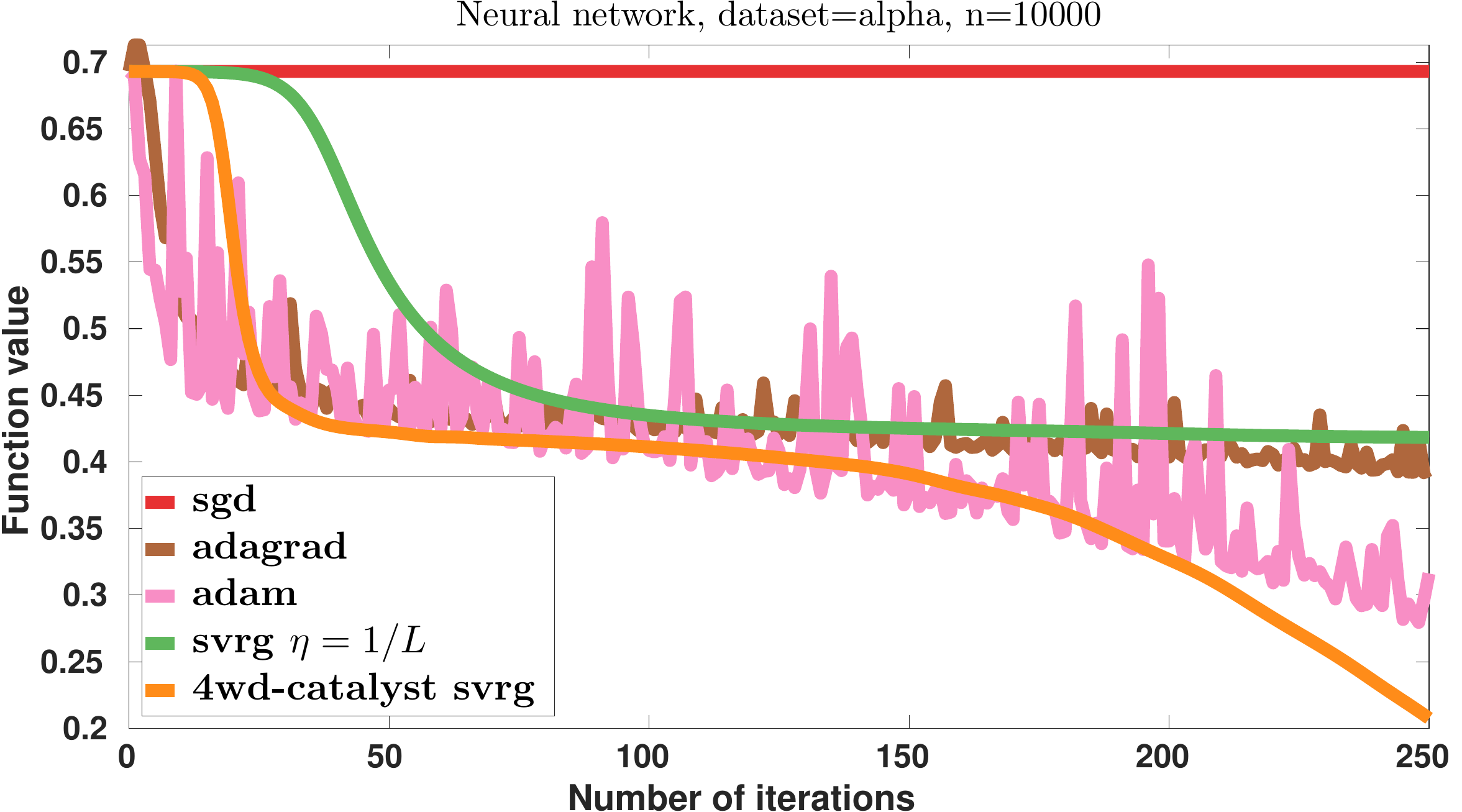}
   \includegraphics[width=.32\textwidth]{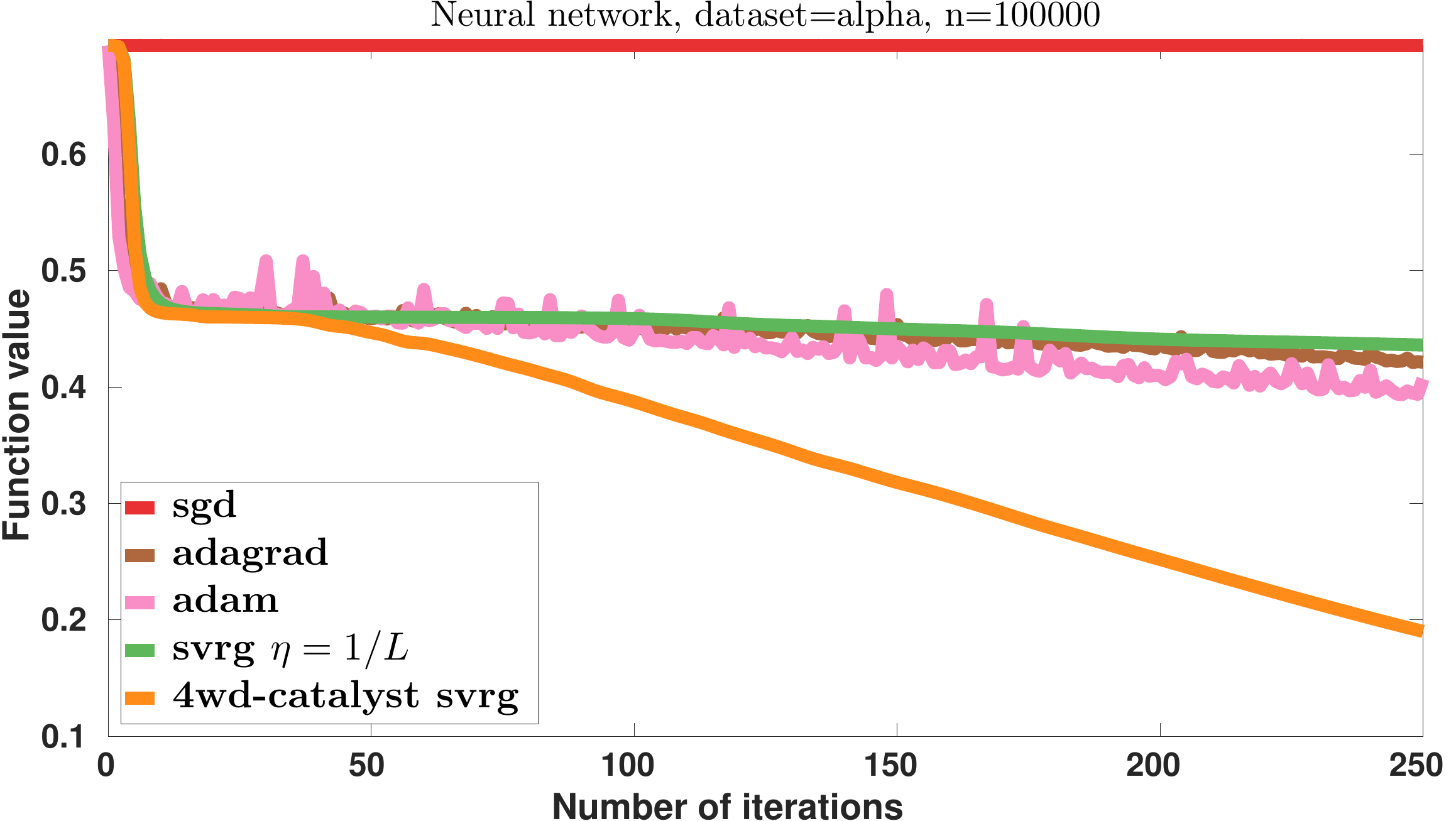} \\
   \includegraphics[width=.32\textwidth]{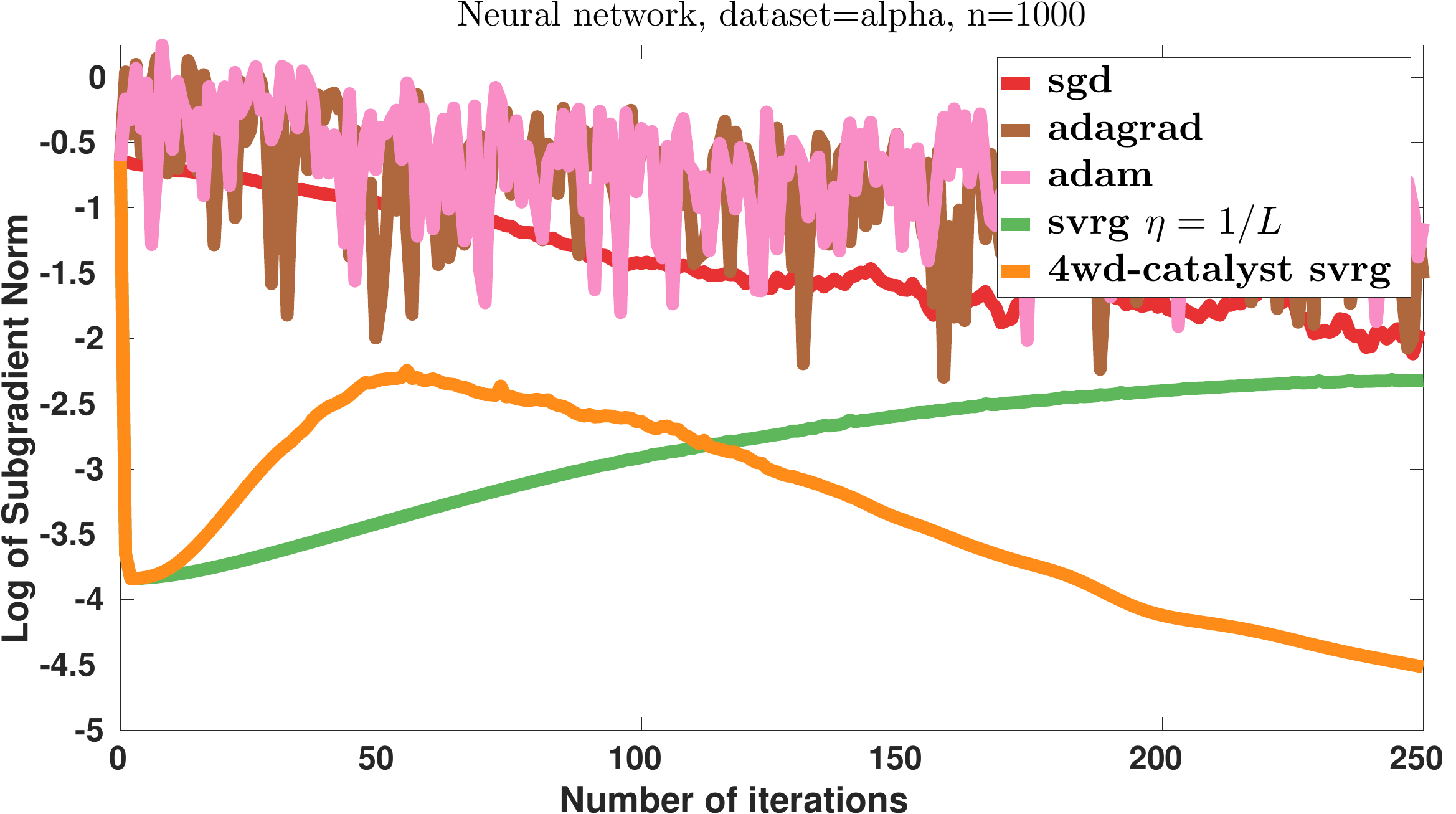}
   \includegraphics[width=.32\textwidth]{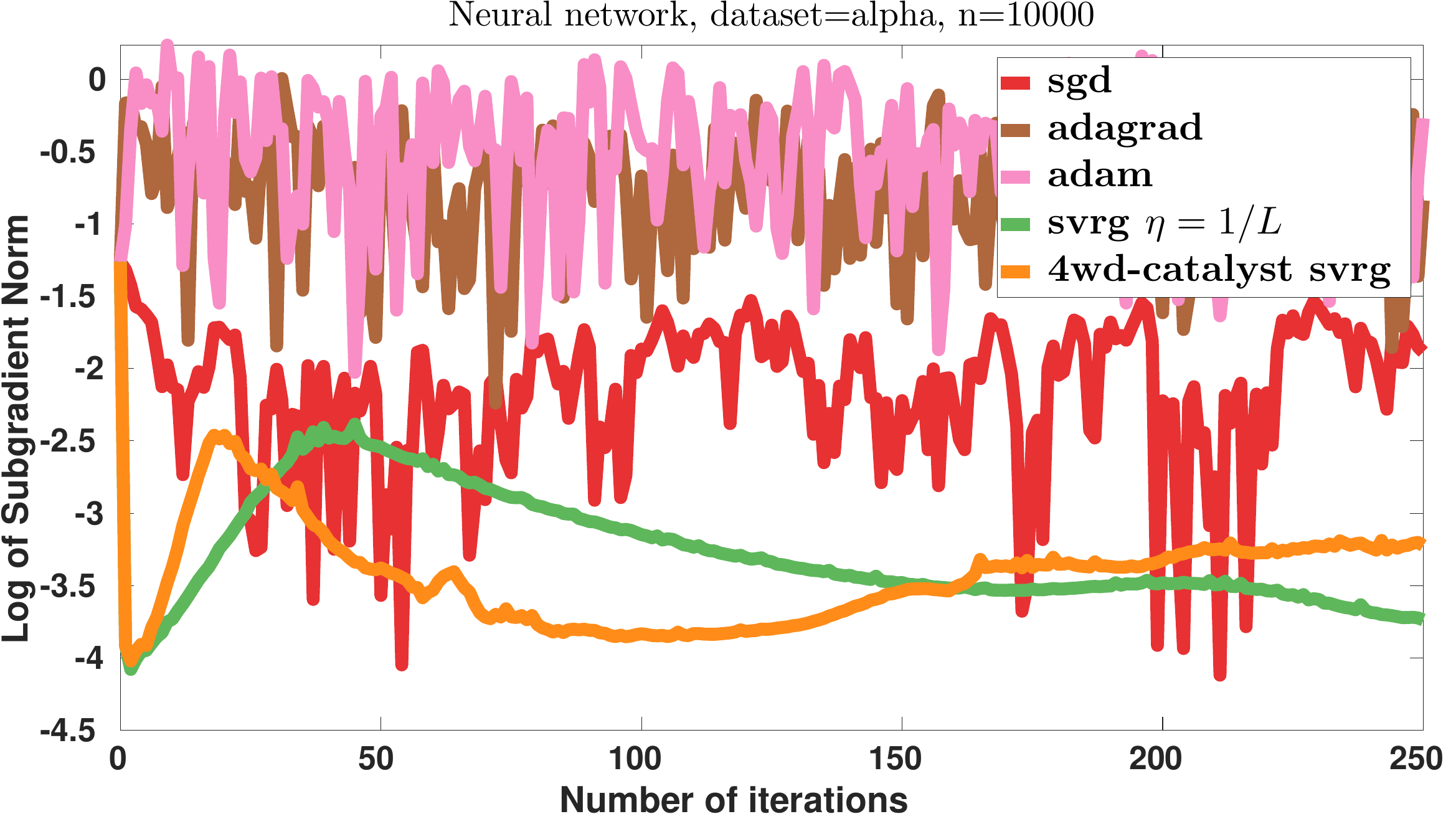}
   \includegraphics[width=.32\textwidth]{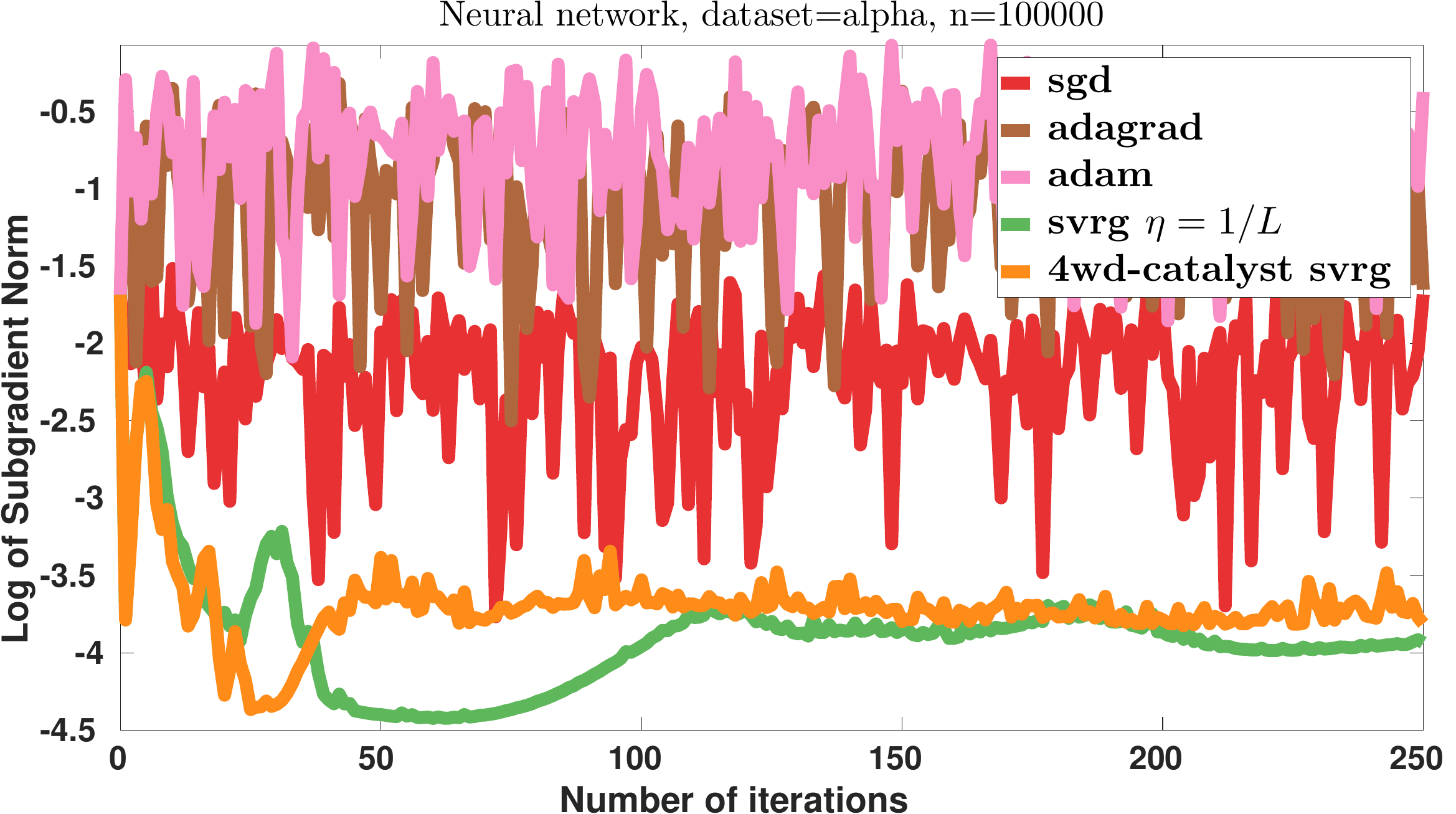} 
\vs
   \caption{Neural network experiments on subsets of dataset \textsf{alpha}. From left to right, we vary the size of the dataset's subset from $n=1\,000$ to~$n=100\,000$.}\label{fig:nn_append_alpha}
\vs
\end{figure*}

\begin{figure*}[t!]
   \includegraphics[width=.31\textwidth]{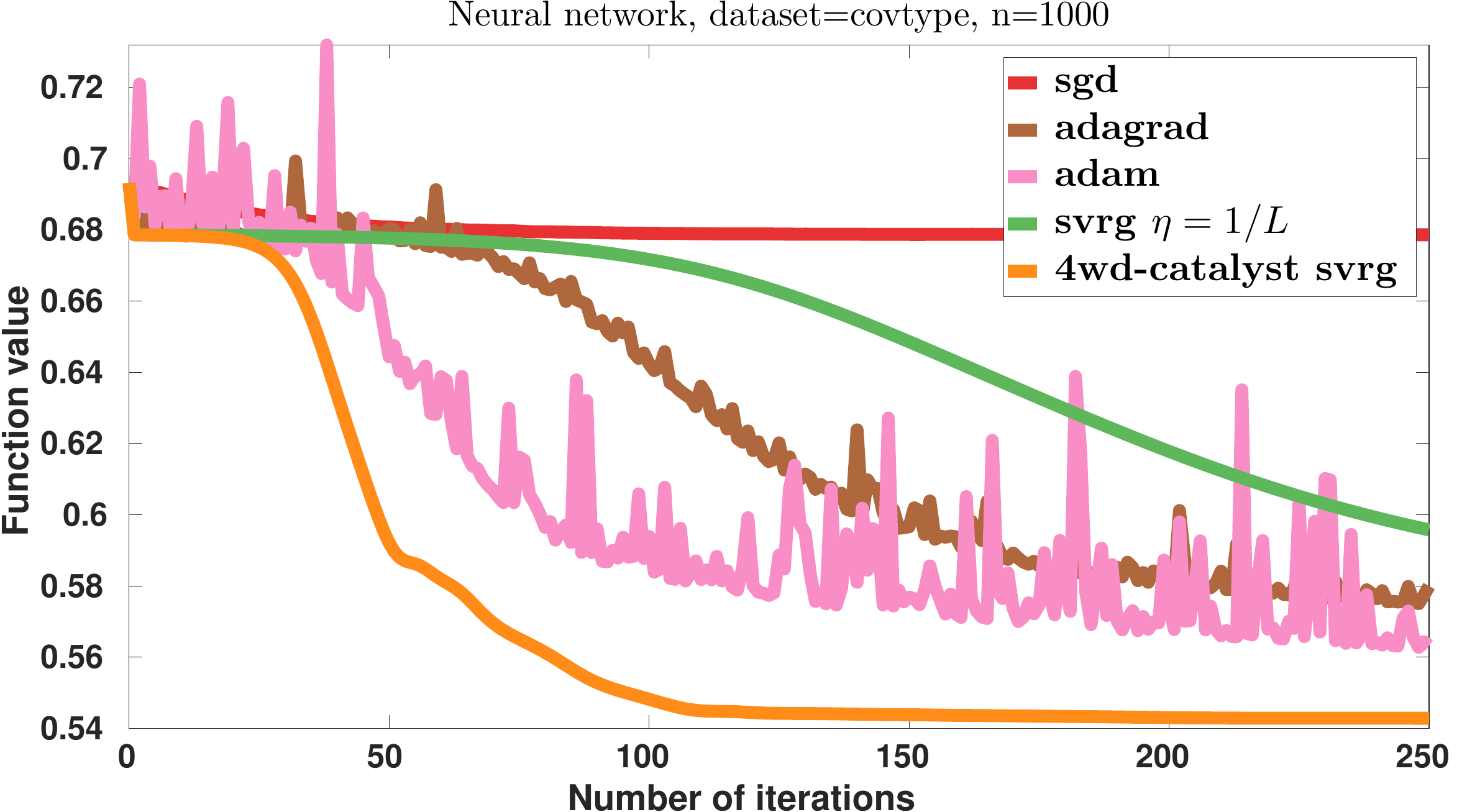}
   \includegraphics[width=.31\textwidth]{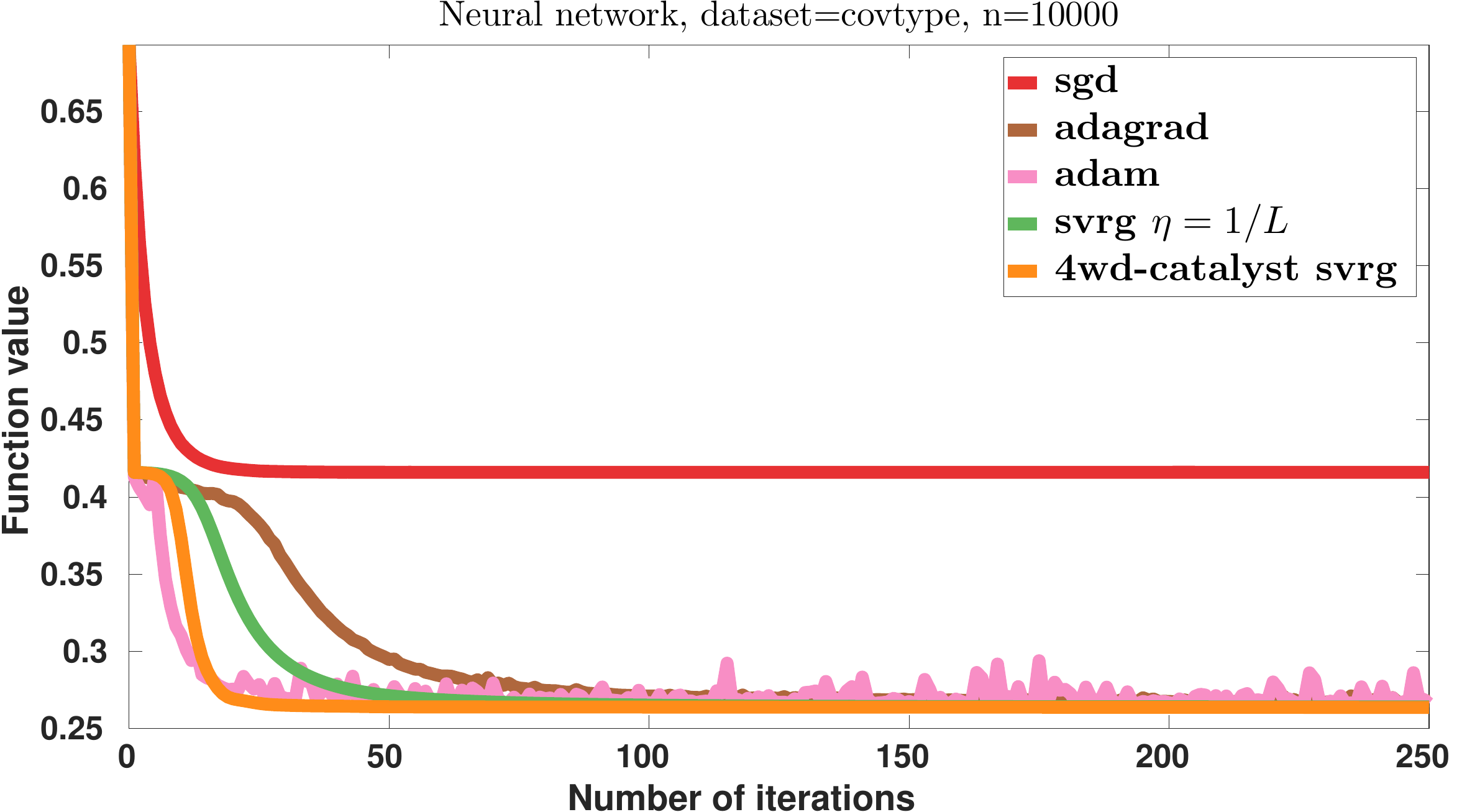}
   \includegraphics[width=.31\textwidth]{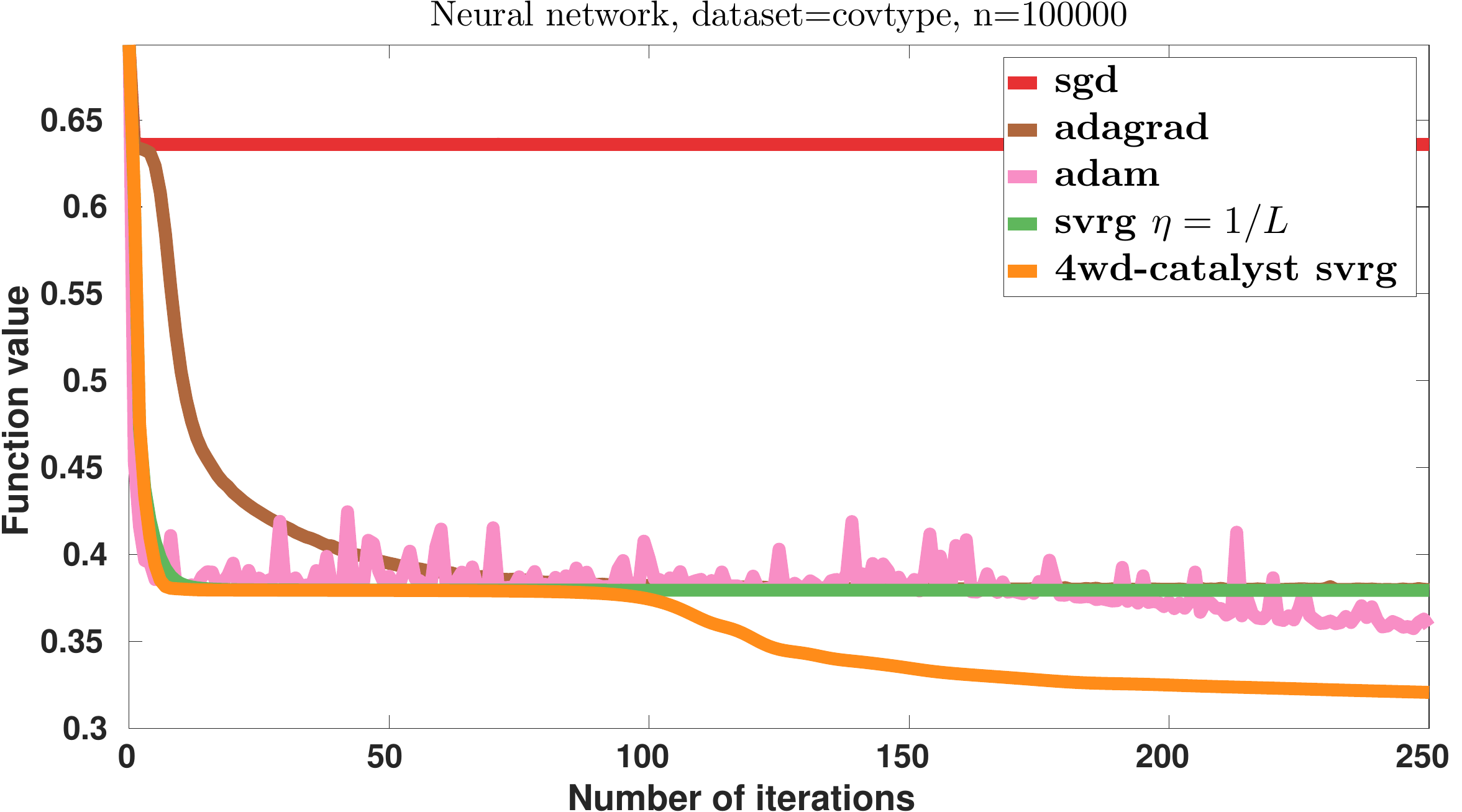} \\
   \includegraphics[width=.31\textwidth]{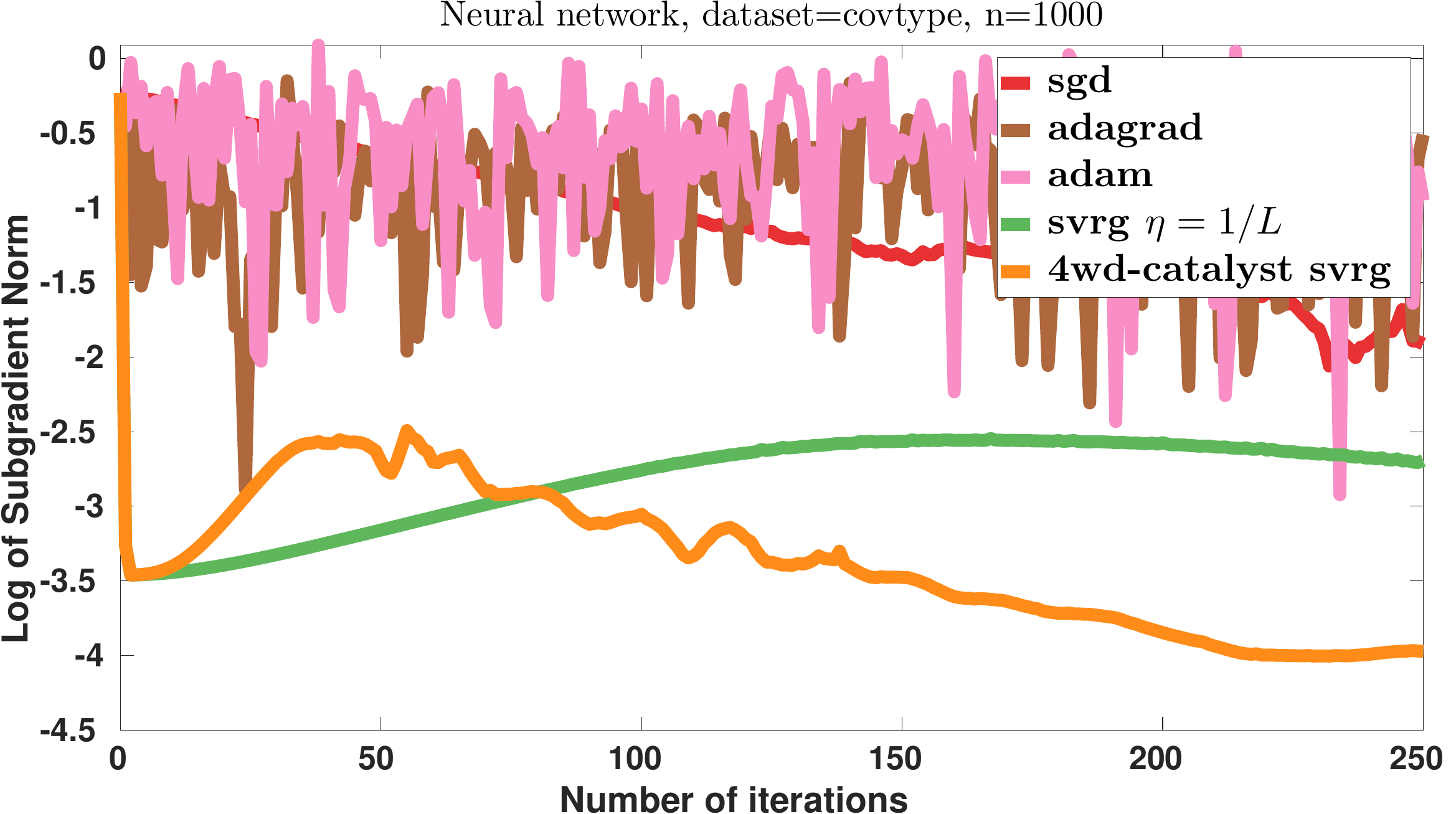}
   \includegraphics[width=.31\textwidth]{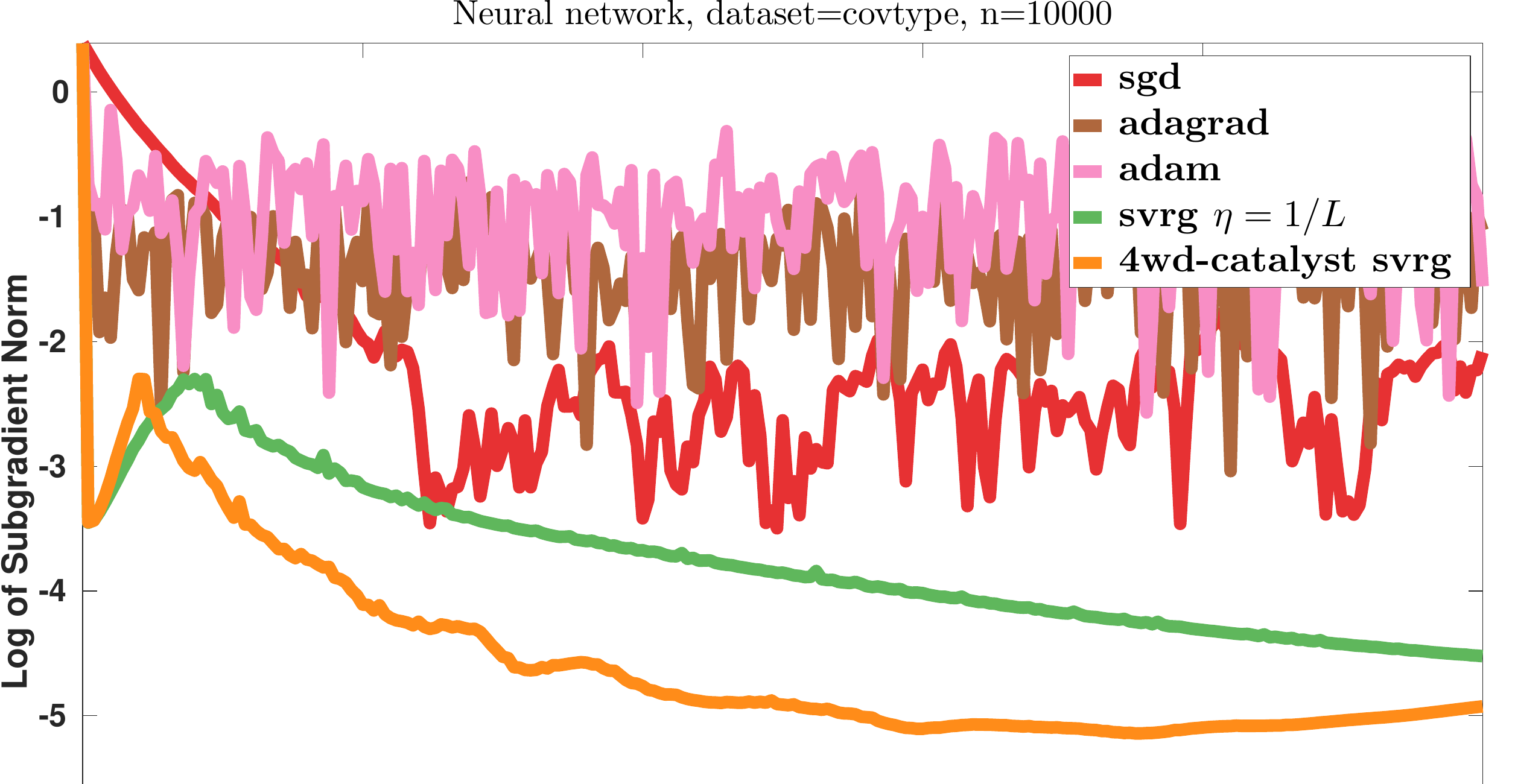}
   \includegraphics[width=.31\textwidth]{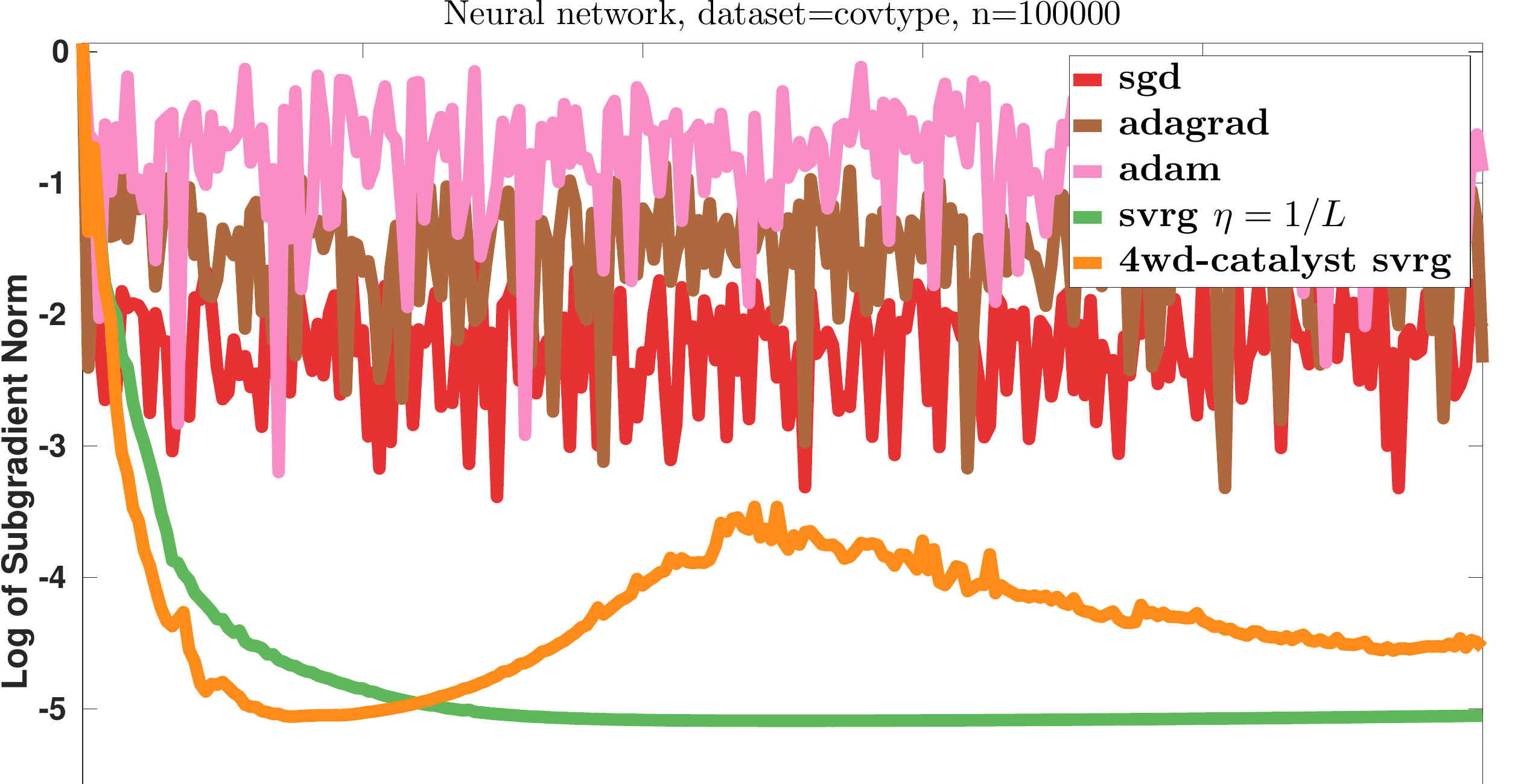}
   \caption{Neural network experiments on subsets of datasets \textsf{alpha}
   (top) and \textsf{covtype} (bottom).}
   \label{fig:nn_append_covtype}
\end{figure*}

\paragraph{Initial estimates of $L$.}
The proposed algorithm~\autonewalgosp requires an initial estimate of the Lipschitz constant $L$. 
In the problems we are considering, there is no simple closed form formula available to compute an estimate of $L$. 
We use the following heuristics to estimate $L$:
\begin{enumerate}
 \item For matrix factorization, it can be shown that the function $f_i$ defined in (\ref{dl1}) is differentiable 
 according to Danskin's theorem [see Bertsekas \cite{bertsekas1999nonlinear}, Proposition B.25] and its gradient is given by 
 $$\nabla_D f_i(D) = -(x_i-D\alpha_i(D))\alpha_i(D)^T \quad \text{ where } \quad \alpha_i(D)\in \argmin_{\alpha \in \R^p} \frac{1}{2} \norm{x_i-D\alpha}^2+\psi(\alpha). $$
 If the coefficients~$\alpha_i$ were fixed, the gradient would be linear in $D$ and thus admit $\norm{\alpha_i}^2$ as Lipschitz constant. Therefore, when initializing our algorithm 
 at $D_0$, we find $\alpha_i(D_0)$ for any $i \in [1,n]$ and use $\max_{i \in [1,n]}{\norm{\alpha_i(D_0)}^2}$ as an estimate of $L$.
 \item For neural networks, the formulation we are considering is differentiable. 
 We randomly generate two pairs of weight vectors $(W_1,W_2)$ and $(W_1',W_2')$ and use the quantity
 $$ \max_{i \in [1,n]} \left \{ \frac{\norm{\nabla f_i(W_1,W_2) - \nabla f_i(W_1',W_2)}}{\norm{W_1-W_1'}},\frac{\norm{\nabla f_i(W_1,W_2) - \nabla f_i(W_1,W_2')}}{\norm{W_2-W_2'}} \right \} $$
as an estimate of the Lipschitz constant, where $f_i$ denotes the loss function respect to $i$-th training sample $(a_i,b_i)$. 
We separate weights in each layer to estimate the Lipschitz constant \emph{per layer}. Indeed the scales of the weights can be quite different across layers.  
 
\end{enumerate}

\vs
\paragraph{Computational cost.}
For SGD, AdaGrad, Adam, and all the ncvx-SVRG/SAGA variants, one iteration
corresponds to one pass over the data in the plots. On the one hand,
since~\autonewalgo-SVRG/SAGA solves two sub-problems per iteration, the
cost per iteration is twice as large as the other algorithms. 
In our experiments, we observe that every time acceleration occurs then $\tilde{x}_k$ is almost always preferred to $\bar{x}_k$ in step 4 of \autonewalgo, half of the computations are in fact not performed when running~\autonewalgo-SVRG/SAGA.  

We report in Figure~\ref{fig:nn_append_covtype_logk} an experimental study where we vary $S$ on the neural network example. 
In terms of number of iterations, of course, the larger $S_k$ the better the performance. This is
not surprising as we solve each subproblem more accurately. Nevertheless, in terms of number of 
gradient evaluations, the relative performance is reversed. There is clearly no benefit to take larger $S_k$. 
This justifies in hindsight our choice of setting $S=n$.   

 \begin{figure*}[t!]
 \begin{center}
   \includegraphics[width=.41\textwidth]{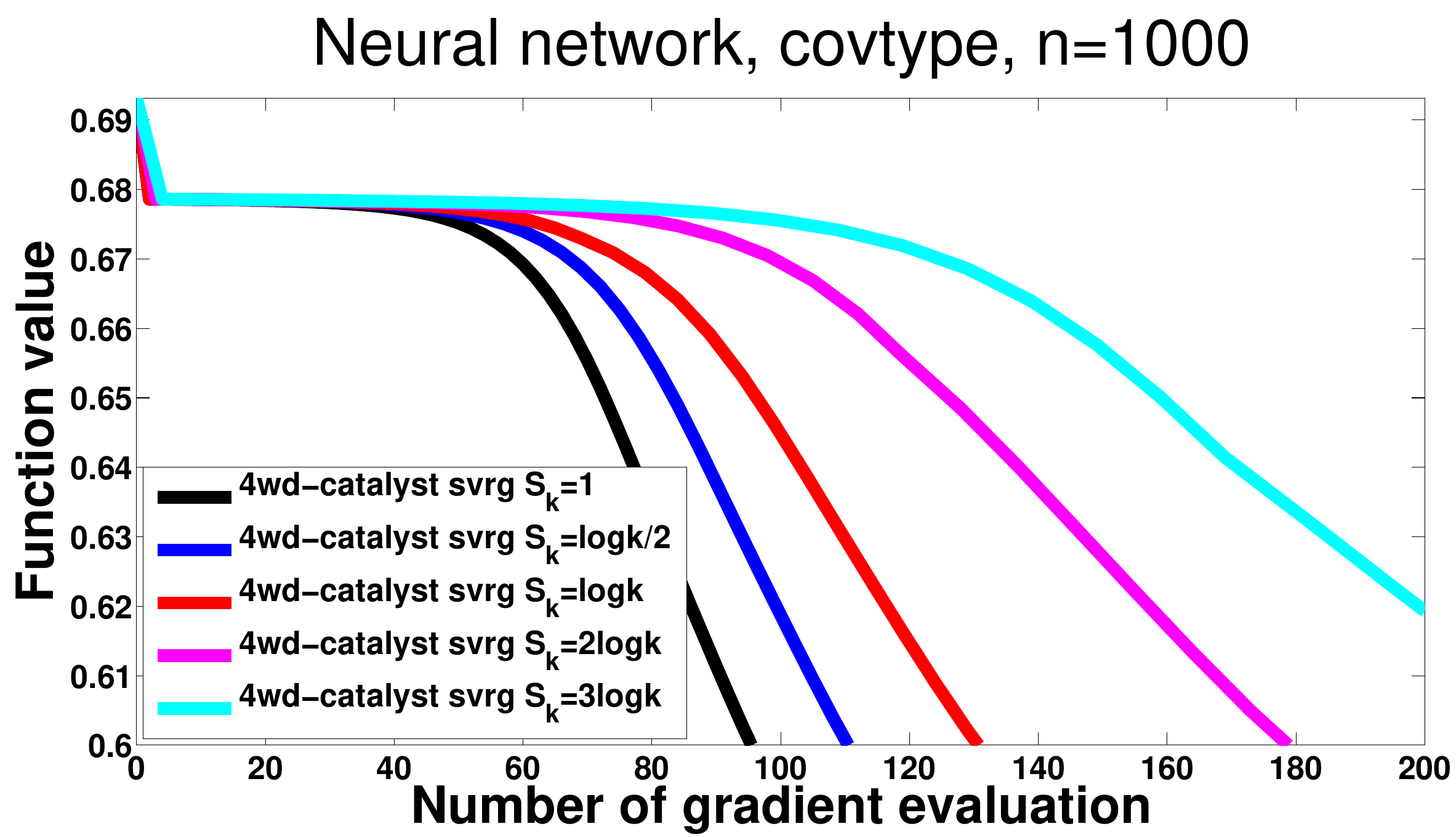}
   \includegraphics[width=.41\textwidth]{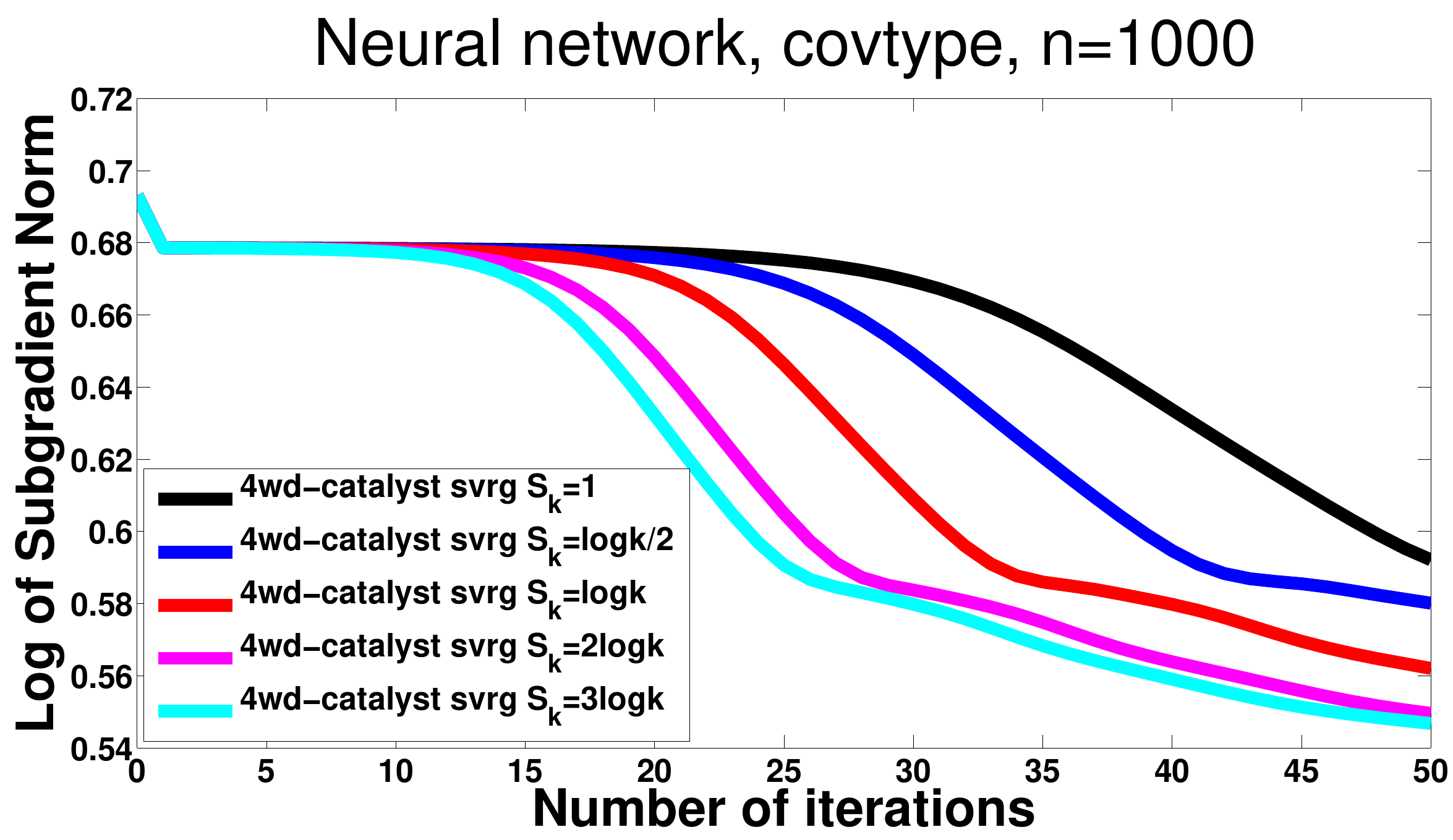}
      \caption{We run 50 iterations of \autonewalgosp SVRG with different choice of S on two-layer neural network. 
      The data is a subset of dataset \textsf{covtype}. 
      The x-axis is the number of gradient evaluations on the left, which is $T+S_k$ per iteration with $T=1$;  
      and the number of iterations on the right.}
   \label{fig:nn_append_covtype_logk}
   \end{center}
\end{figure*}

\vs
\paragraph{Experimental conclusions.}
In the matrix factorization experiments in Fig.~\ref{fig:patches_append_svrg} and
Fig.~\ref{fig:patches_append_saga},
\autonewalgo-SVRG/SAGA were always competitive, with a similar performance to the heuristic SVRG/SAGA-$\eta=1/L$ in two cases out of three, while being significantly better as soon as the amount of data~$n$ was large enough. As expected, the variants of SVRG with theoretical stepsizes have slow convergence, but exhibit a stable behavior compared to SVRG-$\eta=1/L$. This confirms the remarkable ability of~\autonewalgo-SVRG/SAGA to adapt to nonconvex terrains. 

In the neural network experiments, we observe that~\autonewalgo-SVRG
converges much faster overall in terms of objective values than other
algorithms. Yet Adam and AdaGrad often perform well-during the first
iterations, they oscillate a lot, which is a behavior commonly
observed. In constrast, \autonewalgo-SVRG always decreases and keeps
decreasing while other algorithms tend to stabilize, hence achieving
significantly lower objective values. 

More interestingly, as the algorithm proceeds, the subgradient norm
may increase at some point and then decrease, while the function value
keeps decreasing. 
This suggests that the extrapolation step, or the \linesearch~procedure, is helpful to escape bad stationary points, \eg, saddle-points. 
We leave the study of this particular phenomenon as a potential direction for
future work. 

\paragraph{Acknowledgments.}
The authors would like to thank J. Duchi for fruitful discussions related to this work. C. Paquette was partially supported by the ``Learning in Machines and Brains'' program of CIFAR. H. Lin and J. Mairal were supported by ERC grant SOLARIS (\# 714381) and ANR grant MACARON (ANR-14-CE23-0003-01). D. Drusvyatskiy was supported by AFOSR YIP FA9550-15-1-0237, NSF DMS 1651851, and NSF CCF 1740551 awards. Z. Harchaoui was supported by NSF CCF 1740551 award, the ``Learning in Machines and Brains'' program of CIFAR, and faculty research awards. This work was performed while C. Paquette was at University of Washington and H. Lin was at Inria.

\appendix
\section{Convergence rates in strongly-convex composite minimization}\label{subsec:conv}
We now briefly discuss convergence rates,
which are typically given in different forms in the convex and
non-convex cases. If the weak-convex constant is known, we can form a
strongly convex approximation similar to \cite{catalyst}. 
For that purpose, we consider a strongly-convex composite minimization problem
$$\min_{x\in\R^p}~ h(x):=f_0(x)+\psi(x),$$
where $f_0\colon\R^p\to\R$ is $\mu$-strongly convex and smooth with
$L$-Lipschitz continuous gradient $\nabla f_0$, and $\psi\colon\R^p\to \overline\R$
is a closed convex function with a computable proximal map
        $$\pr_{\beta \psi}(y):=\argmin_{z \in \R^p} \left\{\psi(y)+\tfrac{1}{2\beta}\|z-y\|^2\right\}.$$
Let $x^*$ be the minimizer of $h$ and $h^*$ be the minimal value of
$h$. In general, there are three types of measures of optimality that
one can monitor: $\|x-x^*\|^2$, $h(x)-h^*$, and $\text{dist}(0, \partial h(x))$. 

Since $h$ is strongly convex, the three of them are equivalent in terms of convergence rates if one can take an extra {\em prox-gradient step}:
  $$[x]_L:=\pr_{\psi/L}(x-L^{-1}\nabla f_0(x)).$$
To see this, define the {\em displacement vector}, also known as the gradient mapping, 
$g_L(x):=L(x-[x]_L)$, 
and notice the inclusion $g_L(x)\in \partial h([x]_L)$. In particular
$g_L(x)=0$ if and only if $x$ is the minimizer of $h$.
These next inequalities follow directly from Theorem 2.2.7 in \cite{nesterov}:
\begin{align*}
\tfrac{1}{2L}\|g_L(x)\|\leq&\|x-x^*\|\leq \tfrac{2}{\mu}\|g_L(x)\|\\
\tfrac{\mu}{2}\|x-x^*\|^2\leq &h(x)-h^*\leq \tfrac{1}{2\mu}|\partial h(x)|^2\\
2\mu(h([x]_L)-h^*)\leq &\|g_L(x)\|^2\leq 2L(h(x)-h([x]_L))
\end{align*}
Thus, an estimate of any one of the four quantities $\|x-x^*\|$,
$h(x)-h^*$, $\|g_L(x)\|$, or $\text{dist}(0, \partial h(x))$ directly implies an
estimate of the other three evaluated either at $x$ or at
$[x]_L$.
\section{Theoretical analysis of the basic algorithm}\label{sec:first}

We present here proofs of the theoretical results of the paper. All throughout the proofs, we shall work 
under the Assumptions on $f$ stated in Section~\ref{sec:algo} and the Assumptions on $\mathcal{M}$
stated in Section~\ref{sec:autoalgo}.
\subsection{Convergence guarantee of~\newalgosp}
In Theorem~\ref{theo:outerloop-ncvx-basic} and
Theorem~\ref{theo:outerloop-cvx-basic} under  an
appropriate tolerance policy on the proximal
subproblems~\eqref{eqn:prox_1} and~\eqref{eqn:accel_2},~\newalgosp
performs no worse than an exact proximal point method in general, while
automatically accelerating when $f$ is convex. 
For this, we need the following observations.

\begin{lemma}[Growth of $(\alpha_\cnt)$]
	Suppose the sequence $\{\alpha_{\cnt}\}_{\cnt \ge 1}$ is produced by
	Algorithm~\ref{alg: uniform_cat_nols}. Then, the following bounds hold for all $\cnt \ge 1$:
	\[ \frac{\sqrt{2}}{\cnt + 2} \le \alpha_{\cnt} \le \frac{2}{\cnt
		+1}. \]
	\label{lem: growth_a}
\end{lemma}

\begin{proof} 
This result is noted without proof in a remark of \cite{p_tseng}. 
For completeness, we give below a simple proof using induction.
   Clearly, the statement holds for $\cnt = 1$.  Assume the inequality
   on the right-hand side holds for $\cnt$. By using the induction hypothesis, we get
	$$\alpha_{\cnt+1}=\frac{\sqrt{\alpha_\cnt^4+4\alpha_\cnt^2}-\alpha_\cnt^2}{2}=\frac{2}{\sqrt{1+4/\alpha_\cnt^2}+1}\leq \frac{2}{\sqrt{1+(\cnt+1)^2}+1}\leq\frac{2}{\cnt+2},$$
	as claimed and the expression for $\alpha_{\cnt+1}$ is given
        by explicitly solving \eqref{eqn:accel_end}.
	 To
	show the lower
	bound, we note that for all $\cnt \ge 1$, we have $$\alpha_{\cnt +1}^2 =(1-\alpha_{\cnt+1})\alpha_\cnt^2=
	\prod_{i = 2}^{\cnt +1} (1-\alpha_i)\alpha_1^2=\prod_{i = 2}^{\cnt +1} (1-\alpha_i).$$ Using the established upper bound $\alpha_\cnt\leq \frac{2}{\cnt+1}$ yields
	\begin{align*}
	\alpha_{\cnt +1}^2 \ge \prod_{i=2}^{\cnt +1} \left (1-\frac{2}{i+1}
	\right ) = \frac{2}{(\cnt+2)(\cnt +1)} \ge \frac{2}{(\cnt +2)^2}.
	\end{align*}
	The result follows.
\end{proof}

\begin{lemma}[Prox-gradient and near-stationarity] Suppose $y^+$
  satisfies $\text{dist}(0, \partial \env(y^+;y)) < \varepsilon$. Then,
  the inequality holds:
\[ \text{\rm dist}\big (0, \partial f(y^+) \big ) \le \varepsilon +
\norm{\smthpara(y^+-y)}.\]
\label{lem: prox_grad_near_station}
\end{lemma}

\begin{proof} We can find $\xi \in \partial \env(y^+; y)$ with $\norm{\xi}
  \le \varepsilon$. Taking into account $\partial \env(y^+; y)= \partial f(y^+)+\smthpara(y^+-y)$ the result follows.
\end{proof}

Next we establish convergence guarantees of Theorem~\ref{theo:outerloop-ncvx-basic} and
Theorem~\ref{theo:outerloop-cvx-basic} for~\newalgosp. 

\begin{proof}[Proof of Theorem~\ref{theo:outerloop-ncvx-basic} and
  Theorem~\ref{theo:outerloop-cvx-basic}] The proof of Theorem~\ref{theo:outerloop-ncvx-basic}  follows the
  analysis of inexact proximal point
  method~\citep{catalyst,guler:1991,bertsekas:2015}. The descent
  condition in
  \eqref{eq:prox_stop_criteria} implies $\{f(x_\cnt)\}_{\cnt \ge 0}$ are monotonically
  decreasing. From this, we deduce 
\begin{equation} f(x_{\cnt-1}) = f_{\smthpara}(x_{\cnt-1}; x_{\cnt-1}) \ge
  f_{\smthpara}(\proxx_\cnt; x_{\cnt-1}) \ge f(x_\cnt) +
  \frac{\smthpara}{2}
  \norm{\proxx_{\cnt}-x_{\cnt-1}}^2. \label{eq: new_idea_10}
\end{equation}
Using the adaptive stationarity condition \eqref{eq:prox_stop_criteria}, we apply Lemma~\ref{lem: prox_grad_near_station} with $y=x_{\cnt-1}$, $y^+=\proxx_\cnt$ and 
$\varepsilon =\smthpara \norm{\proxx_\cnt-x_{\cnt-1}}$; hence we obtain 
\[\text{dist}(0, \partial f(\proxx_\cnt) ) \le 2 
    \norm{\smthpara(\proxx_{\cnt}-x_{\cnt-1})}. \]
We combine the above inequality with \eqref{eq: new_idea_10} to deduce
\begin{align}
\text{dist}^2(0, \partial f(\proxx_{\cnt})) \le  4
  \norm{\smthpara(\proxx_{\cnt}-x_{\cnt-1})}^2 \le 8 \smthpara \left(
  f(x_{\cnt-1})-f(x_{\cnt}) \right ). \label{eq: new_idea_2_simple}
\end{align}
Summing $j =1$ to $N$, we conclude 
\begin{align*} \min_{j=1, \hdots, N}~\big \{ \text{dist}^2(0, \partial
  f(\proxx_j)) \big \}
  &\le \frac{4}{N} \sum_{j=1}^N
  \norm{\smthpara(\proxx_{\cnt}-x_{\cnt-1})}^2)\\
  &\le  \frac{8 \smthpara}{N} \left (\sum_{j=1}^N f(x_{j-1})-f(x_j) \right )\\
&\le \frac{8 \smthpara}{N} \left ( f(x_0)-f^* \right ).
\end{align*}
Next, suppose the function $f$ is convex. Our analysis is similar to
that of~\cite{p_tseng,fista}. Using the stopping criteria
\eqref{eq:accx_stop_criteria}, fix an $\xi_\cnt \in \partial
\env(\accx_\cnt; y_\cnt)$ with $\norm{\xi_\cnt} <
\frac{\smthpara}{k+1} \norm{\accx_{\cnt}-y_\cnt}$. For any $x \in
\R^n$, Equation~\eqref{eqn:min_better}, and the strong convexity of
the function $\env(\cdot; y_\cnt)$ yields
\[f(x_\cnt) \le f(\accx_\cnt) \le f(x) + \frac{\smthpara}{2} \left (
    \norm{x-y_\cnt}^2 - \norm{x-\accx_{\cnt}}^2 -
    \norm{\accx_\cnt-y_\cnt}^2
  \right ) + \xi^T_{\cnt} \left (\accx_{\cnt}-x \right ).\]
We substitute $x = \alpha_\cnt x^* + (1-\alpha_\cnt) x_{\cnt-1}$ where
$x^*$ is any minimizer of $f$. Using the convexity of $f$, the norm of
$\xi_\cnt$, and
Equations~\eqref{eqn:accel_1} and \eqref{eqn:v_try}, we deduce
\begin{align}
f(x_\cnt) \le \alpha_\cnt f(x^*) &+ (1-\alpha_{\cnt}) f(x_{\cnt-1}) +
  \frac{\alpha_{\cnt}^2 \smthpara}{2} \left ( \norm{x^*-v_{\cnt-1}}^2 - \norm{x^*-v_{\cnt}}^2
  \right ) \nonumber\\
& -\frac{\smthpara}{2} \norm{\accx_\cnt-y_{\cnt}}^2 +
  \frac{\alpha_{\cnt} \smthpara}{\cnt+1} \norm{\accx_{\cnt}-y_\cnt}
  \norm{x^*-v_{\cnt}}. \label{eq:new_stuff_complete_square}
\end{align}
Set $\theta_\cnt = \frac{1}{k+1}$. Completing the square on
Equation \eqref{eq:new_stuff_complete_square}, we obtain 
\[ \frac{-\smthpara}{2} \norm{\accx_{\cnt}-y_\cnt}^2 + \alpha_\cnt
  \theta_{\cnt} \smthpara \norm{\accx_{\cnt}-y_{\cnt}} \norm{x^*-v_{\cnt}} \le
  \frac{\smthpara}{2} \left (\alpha_{\cnt}\theta_{\cnt} \right )^2
  \norm{x^*-v_{\cnt}}^2. \]
Hence, we deduce
\begin{align*}
f(x_\cnt)-f^* &\le (1-\alpha_{\cnt}) (f(x_{\cnt-1})-f^*) +
  \frac{\alpha_{\cnt}^2 \smthpara}{2} \left (
  \norm{x^*-v_{\cnt-1}}^2-\norm{x^*-v_{\cnt}}^2 \right )\\
&\qquad \qquad  +
  \frac{\smthpara}{2} \left ( \alpha_{\cnt}
  \theta_{\cnt} \right )^2 \norm{x^*-v_{\cnt}}^2.\\
&= (1-\alpha_{\cnt}) (f(x_{\cnt-1})-f^*) +
  \frac{\alpha_{\cnt}^2 \smthpara}{2} \left (
  \norm{x^*-v_{\cnt-1}}^2-\left (1-\theta_{\cnt}^2\right ) \norm{x^*-v_{\cnt}}^2 \right )
\end{align*}
Denote $A_\cnt := 1- \theta_{\cnt}^2$. Subtracting $f^*$ from both
sides and using the inequality
$\frac{1-\alpha_\cnt}{\alpha_\cnt^2}=\frac{1}{\alpha_{\cnt-1}^2}$ and
$\alpha_1 \equiv 1$, we derive the following recursion argument:
\begin{align*}
\frac{f(x_{\cnt})-f^*}{\alpha_{\cnt}^2} + \frac{A_{\cnt}\smthpara}{2} \norm{x^*-v_{\cnt}}^2
  &\le \frac{1-\alpha_{\cnt}}{\alpha_{\cnt}^2} \big ( f(x_{\cnt-1})-f^*
  \big ) + \frac{\smthpara}{2} \norm{x^*-v_{\cnt-1}}^2\\
&\le \frac{1}{A_{\cnt-1}} \left ( \frac{f(x_{\cnt-1})-f^*}{\alpha_{\cnt-1}^2}+ \frac{A_{\cnt-1}\smthpara }{2}
  \norm{x^*-v_{\cnt-1}}^2 \right ).
\end{align*}
The last inequality follows because $0 < A_{\cnt-1} \le 1$. Iterating
$N$ times,we deduce
\begin{align}
\frac{f(x_N)-f^*}{\alpha_N^2} \le \prod_{j=2}^N \frac{1}{A_{j-1}} \left (
  \frac{\smthpara}{2} \norm{x^*-v_0}^2
  \right ). \label{eq: new_idea_3_simple}
\end{align}
We note
\begin{equation} \prod_{j=2}^N \frac{1}{A_{j-1}} = \frac{1}{\prod_{j=2}^N \left ( 1-
  \frac{1}{(j+1)^2} \right ) } \le 2; \label{eq:important_sum}
\end{equation}
thereby concluding the result. Summing up \eqref{eq: new_idea_2_simple} from $j =
N+1$ to $2N$, we obtain
\begin{align*} \min_{j=1, \hdots, 2N}~\big \{ \text{dist}^2(0, \partial
  f(\proxx_j)) \big \}
  &\le \frac{4}{N} \sum_{j=N+1}^{2N} 
  \norm{\smthpara(\proxx_{\cnt}-x_{\cnt-1})}^2)\\
  &\le  \frac{8 \smthpara}{N} \left (\sum_{j=N+1}^{2N} f(x_{j-1})-f(x_j) \right )\\
&\le \frac{8 \smthpara}{N} \left ( f(x_N)-f^* \right )
\end{align*}
Combining this inequality with \eqref{eq: new_idea_3_simple}, the result is
shown.  
\end{proof}
\section{Analysis of \autonewalgo~and \linesearch}\label{sec:appendix_global_comp}

\paragraph*{Linear convergence interlude.}
Our assumption on the linear rate of convergence of $\mathcal{M}$ (see \eqref{eq:criteria}) may look strange at first sight. Nevertheless, most linearly convergent first-order methods $\mathcal{M}$ for composite minimization either already satisfy this assumption or can be made to satisfy it by introducing an extra prox-gradient step. 
To see this, recall the convex composite minimization problem from Section~\ref{subsec:conv}
$$\min_{z\in\R^p}~ h(z):=f_0(z)+\psi(z),$$
where 
\begin{enumerate}
	\item 
$f_0\colon\R^p\to\R$ is convex and $C^1$-smooth with the gradient $\nabla f_0$ that is $L$-Lipschitz,
\item $\psi\colon\R^p\to \overline\R$ is a closed convex function with a computable proximal map
$$\pr_{\beta\psi}(y):=\argmin_z~ \{\psi(y)+\tfrac{1}{2\beta}\|z-y\|^2\}.$$
\end{enumerate}
See~\cite{Parikh13} for a survey of proximal maps. 
Typical linear convergence guarantees of an optimization algorithm assert existence of constants $A\in \R$ and $\tau\in (0,1)$ satisfying
\begin{equation}\label{eqn:lin_can_comp}
h(z_{t})-h^*\leq A(1-\tau)^t(h(z_0)-h^*)
\end{equation}
for each $t=0,1,2,\ldots,\infty$. To bring such convergence guarantees into the desired form \eqref{eq:criteria}, define the prox-gradient step $$[z]_L:=\pr_{\psi/L}(z-L^{-1}\nabla f_0(z)),$$
and the displacement vector
$$g_L(z)=L(z-[z]_L),$$ 
and notice the inclusion $g_L(z)\in \partial h([z]_L)$.
The following inequality follows from  \cite{nesterov2013gradient}:
\begin{align*}
\|g_L(z)\|^2\leq 2L(h(z)-h([z]_L))\leq 2L(h(z)-h^*).
\end{align*}
Thus, the linear rate of convergence \eqref{eqn:lin_can_comp} implies 
$$\|g_L(z_{t})\|^2\leq 2LA (1-\tau)^t(h(z_0)-h^*),$$
which is exactly in the desired form \eqref{eq:criteria}.
\subsection{Convergence analysis of the adaptive algorithm: \autonewalgo
  }\label{subsec:adaptive}

First, under some reasonable assumptions on the method $\mathcal{M}$ (see
Section~\ref{para: assumptions_M}), the sub-method
\linesearch~terminates. 

\begin{lemma}[{\linesearch} terminates] \label{lem:lscomplete}
Assume that $\tau_{\smthpara}\to 1$ when~$\smthpara \to +\infty$. The procedure \linesearch $(x,\smthpara, \varepsilon, T)$ terminates after finitely many iterations. 
\end{lemma}
\begin{proof}
Due to our assumptions on $\mathcal{M}$ and the expressions $\env(x ;x )=f(x)$ and $\env^*(x)\geq f^*$, we have
\begin{equation}\label{eqn:conv_fin_term}
\text{\rm dist}^2 \big (0,\partial \env(z_T;x) \big )\leq A(1-\tau_{\mathcal{\smthpara}})^{T}
\big ( f(x)-\env^*(
x) \big )\leq 
A(1-\tau_{\smthpara})^{T}
\big ( f(x)-f^*) \big ).
\end{equation}
Since $\tau_{\smthpara}$ tends to one, for all sufficiency large $\smthpara$, we can be sure that the right-hand-side is smaller than~$\varepsilon^2$. On the other hand, for $\smthpara>\weakcnx$, the function $\env(\cdot;x)$ is $(\smthpara-\weakcnx)$-strongly convex and therefore we have 
$\text{dist}^2(0,\partial \env(z_T;x))\geq 2(\smthpara-\weakcnx)(\env(z_T;x)-\env^*(x))$. Combining this with \eqref{eqn:conv_fin_term}, we deduce  
$$\env(z_T;x)-\env^*(x)\leq \frac{A(1-\tau_{\mathcal{\smthpara}})^{T}}{2(\smthpara-\weakcnx)}
\big ( f(x)-\env^*(x) \big ).$$
Letting $\smthpara \to \infty$, we deduce $\env(z_T;x)\leq f(x)$, as required. Thus the loop indeed terminates.
\end{proof}

We prove the main result, Theorem~\ref{thm: computational_complex_result}, for \autonewalgo. 

\begin{proof}[Proof of Theorem~\ref{thm:
    computational_complex_result}] The proof closely resembles the
  proofs of Theorem~\ref{theo:outerloop-cvx-basic} and
  Theorem~\ref{theo:outerloop-cvx-basic}, so we omit some of the
  details. The main difference in the proof is that we keep track of the
  effects the parameters $\smthparacvx$ and
  $\smthparainit$ have on the inequalities as well as the sequence of
  $\smthpara_{\cnt}$. Since $\{f(x_\cnt)\}_{\cnt \ge 0}$ are monotonically
  decreasing, we deduce 
\begin{equation} f(x_{\cnt-1}) = f_{\smthpara_{\cnt}}(x_{\cnt-1}; x_{\cnt-1}) \ge
  f_{\smthpara_{\cnt}}(\proxx_\cnt; x_{\cnt-1}) \ge f(x_\cnt) +
  \frac{\smthpara_{\cnt}}{2}
  \norm{\proxx_{\cnt}-x_{\cnt-1}}^2. \label{eq: new_idea_1}
\end{equation}
Using the adaptive stationary condition
\eqref{eq:prox_stop_criteria_new}, we apply Lemma~\ref{lem:
  prox_grad_near_station} with $\varepsilon = \smthpara_{\cnt}
\norm{\proxx_\cnt-x_{\cnt-1}}$; hence we obtain 
\[\text{dist}(0, \partial f(\proxx_\cnt) ) \le 2 
    \norm{\smthpara_\cnt(\proxx_{\cnt}-x_{\cnt-1})}. \]
We combine the above inequality with \eqref{eq: new_idea_1} to deduce
\begin{align}
\text{dist}^2(0, \partial f(\proxx_{\cnt})) \le  4
  \norm{\smthpara_{\cnt}(\proxx_{\cnt}-x_{\cnt-1})}^2 \le 8 \smthpara_{\max} \left(
  f(x_{\cnt-1})-f(x_{\cnt}) \right ). \label{eq: new_idea_2}
\end{align}
Summing $j =1$ to $N$, we conclude 
\begin{align*} \min_{j=1, \hdots, N}~\big \{ \text{dist}^2(0, \partial
  f(\proxx_j)) \big \}
  &\le \frac{4}{N} \sum_{j=1}^N 2
  \norm{\smthpara_{\cnt}(\proxx_{\cnt}-x_{\cnt-1})}^2)\\
  &\le  \frac{8 \smthpara_{\max}}{N} \left (\sum_{j=1}^N f(x_{j-1})-f(x_j) \right )\\
&\le \frac{8 \smthpara_{\max}}{N} \left ( f(x_0)-f^* \right ).
\end{align*}
Suppose the function $f$ is convex. Using in the stopping criteria
\eqref{eq:accx_stop_criteria_new} in replacement of \eqref{eq:prox_stop_criteria}, we deduce a
similar expression as \eqref{eq:new_stuff_complete_square}:
\begin{align*}
f(x_\cnt) \le \alpha_\cnt f(x^*) &+ (1-\alpha_{\cnt}) f(x_{\cnt-1}) +
  \frac{\alpha_{\cnt}^2 \smthparacvx}{2} \left ( \norm{x^*-v_{\cnt-1}}^2 - \norm{x^*-v_{\cnt}}^2
  \right )\\
& -\frac{\smthparacvx}{2} \norm{\accx_\cnt-y_{\cnt}}^2 +
  \frac{\alpha_{\cnt} \smthparacvx}{\cnt+1} \norm{\accx_\cnt-y_\cnt} \norm{x^*-v_{\cnt}}.
\end{align*}
Denote $\theta_\cnt
  = \frac{1}{k+1}$. Completing the square, we obtain 
\[ \frac{-\smthparacvx}{2} \norm{\accx_{\cnt}-y_\cnt}^2 + \alpha_\cnt
  \theta_{\cnt} \smthparacvx \norm{\accx_{\cnt}-y_{\cnt}} \norm{x^*-v_{\cnt}} \le
  \frac{\smthparacvx}{2} \left (\alpha_{\cnt}\theta_{\cnt} \right )^2
  \norm{x^*-v_{\cnt}}^2. \]
Hence, we deduce
\begin{align*}
f(x_\cnt)-f^* &\le (1-\alpha_{\cnt}) (f(x_{\cnt-1})-f^*) +
  \frac{\alpha_{\cnt}^2 \smthparacvx}{2} \left (
  \norm{x^*-v_{\cnt-1}}^2-\norm{x^*-v_{\cnt}}^2 \right )\\
&\qquad \qquad  +
  \frac{\smthparacvx}{2} \left ( \alpha_{\cnt}
  \theta_{\cnt} \right )^2 \norm{x^*-v_{\cnt}}^2.\\
&= (1-\alpha_{\cnt}) (f(x_{\cnt-1})-f^*) +
  \frac{\alpha_{\cnt}^2 \smthparacvx}{2} \left (
  \norm{x^*-v_{\cnt-1}}^2-\left (1-\theta_{\cnt}^2\right ) \norm{x^*-v_{\cnt}}^2 \right )
\end{align*}
Denote $A_\cnt := 1- \theta_{\cnt}^2$. Following
the standard recursion argument as in the proofs of Theorem~\ref{theo:outerloop-cvx-basic} and
  Theorem~\ref{theo:outerloop-cvx-basic}, we conclude
\begin{align*}
\frac{f(x_{\cnt})-f^*}{\alpha_{\cnt}^2} + \frac{A_{\cnt}\smthparacvx}{2} \norm{x^*-v_{\cnt}}^2
  &\le \frac{1-\alpha_{\cnt}}{\alpha_{\cnt}^2} \big ( f(x_{\cnt-1})-f^*
  \big ) + \frac{\smthparacvx}{2} \norm{x^*-v_{\cnt-1}}^2\\
&\le \frac{1}{A_{\cnt-1}} \left ( \frac{f(x_{\cnt-1})-f^*}{\alpha_{\cnt-1}^2}+ \frac{A_{\cnt-1}\smthparacvx }{2}
  \norm{x^*-v_{\cnt-1}}^2 \right ).
\end{align*}
The last inequality follows because $0 < A_{\cnt-1} \le 1$. Iterating
$N$ times, we deduce
\begin{align}
\frac{f(x_N)-f^*}{\alpha_N^2} \le \prod_{j=2}^N \frac{1}{A_{j-1}} \left (
  \frac{\smthparacvx}{2} \norm{x^*-v_0}^2
  \right ). \label{eq: new_idea_3}
\end{align}
We note 
\[\prod_{j=2}^N \frac{1}{A_{j-1}} = \frac{1}{\prod_{j=2}^N \left ( 1-
  \frac{1}{(j+1)^2} \right ) } \le 2;\]
thus the result is shown. Summing up \eqref{eq: new_idea_2} from $j =
N+1$ to $2N$, we obtain
\begin{align*} \min_{j=1, \hdots, 2N}~\big \{ \text{dist}^2(0, \partial
  f(\proxx_j)) \big \}
  &\le \frac{4}{N} \sum_{j=N+1}^{2N}
  \norm{\smthpara_{\cnt}(\proxx_{\cnt}-x_{\cnt-1})}^2)\\
  &\le  \frac{8 \smthpara_{\max} }{N} \left (\sum_{j=N+1}^{2N} f(x_{j-1})-f(x_j) \right )\\
&\le \frac{8 \smthpara_{\max}}{N} \left ( f(x_N)-f^* \right )
\end{align*}
Combining this inequality with \eqref{eq: new_idea_3}, the result is
shown. 
\end{proof}
\section{Inner-loop complexity: proof of Theorem~\ref{thm:
    inner_complexity}} \label{sec: complexity}

Recall, the following notation
\begin{align}
f_0(x;y) &= \frac{1}{n} \sum_{i=1}^n
  f_i(x) + \frac{\smthpara}{2} \norm{x-y}^2 \nonumber\\
y^0 &= \text{prox}_{1/(\smthpara+L) f_0}\left ( y-
\frac{1}{\smthpara + L} \nabla f_0(y;y) \right ). \label{eq:proximal_initial}
\end{align}

\begin{lemma}[Relationship between function values and iterates
  of the prox] \label{prop: proximal} Assuming $\psi(x)$ is convex and the parameter
  $\smthpara > \weakcnx$, then 
\begin{equation} f_\smthpara(y^0; y) - f_{\smthpara}^*(y) \le \frac{\smthpara +
    L}{2} \norm{y^*-y}^2 \label{eq:new_prox_iterate}
\end{equation}
where $y^*$ is a minima of $f_{\smthpara}(\cdot; y)$ and
$f_{\smthpara}^*(y)$ is the optimal value. 
\end{lemma}

\begin{proof} As the $\smthpara$ is chosen sufficiently large, we know $f_0(\cdot;
  y)$ is convex and differentiable with $(\smthpara +
  L)$-Lipschitz continuous gradient. Hence, we deduce for all $x$
\begin{equation}
f_0(y;y) + \nabla f_0(y;y)^T (x-y) \le f_0(x;y). \label{eq: new_cvx_1}
\end{equation}
Using the definition of $y^0$ and the
$(\smthpara+L)$-Lip. continuous gradient of $f_0(\cdot;y)$, we
  conclude for all $x$
\begin{equation}
\begin{aligned}
f_\smthpara(y^0; y) = f_0(y^0;y) + \psi(y^0) &\le f_0(y; y) + \nabla
                                               f_0(y;y)^T(y_0-y) +
                                               \frac{\smthpara+L}{2} \norm{y_0-y}^2+
                                               \psi(y_0) \\
&\le f_0(y; y) + \nabla f_0(y;y)^T(x-y) + \frac{\smthpara +
  L}{2} \norm{x-y}^2 + \psi(x). \label{eq: new_cvx_2}
\end{aligned}
\end{equation}
By setting $x = y^*$ in both \eqref{eq: new_cvx_1} and \eqref{eq:
  new_cvx_2} and combining these results, we conclude
\begin{align*}
f_\smthpara(y^0; y) \le f_\smthpara^*(y) + \frac{\smthpara+L}{2}
  \norm{y^*-y}^2. 
\end{align*}
\end{proof}

Note that if we are not in the composite setting and $\smthpara >
\weakcnx$, then $\env(\cdot, y)$ is $(\smthpara+L)$-strongly
convex. Using standard bounds for strongly convex functions, Equation
\eqref{eq:new_prox_iterate} follows (see \cite{nesterov}). We next show an important lemma for deducing the inner complexities. 

\begin{lemma}\label{lem:E2} 
Assume $\smthpara > \weakcnx$. Given any $\varepsilon \leq \frac{\smthpara-\weakcnx}{2}$, if an iterate $z$ satisfies $\text{dist}(0, \partial \env(z; y) ) \leq \varepsilon \norm{y^*-y}, $ then 
 \begin{equation}
   \text{dist}(0, \partial \env(z; y) ) \leq 2\varepsilon \norm{z-y}.
 \end{equation}
\end{lemma}
\begin{proof} Since $\smthpara > \weakcnx$, we know $\env(\cdot;y)$ is
  $(\smthpara-\weakcnx)$-strongly convex. Therefore, by
  \cite{nesterov}, we know 
\begin{equation}
\norm{z-y^*} \le \frac{1}{\smthpara-\weakcnx}
  \text{dist}(0, \partial \env(z;y) ). \label{eq:E_new_1}
\end{equation}
By the triangle inequality and Equation \eqref{eq:E_new_1}, we deduce
\begin{align*}
\text{dist}(0, \partial \env(z; y) ) \leq \varepsilon \norm{y^*-y}
  &\leq \varepsilon  \big ( \norm{y^*-z}+\norm{z-y} \big ) \\
&\le \frac{\varepsilon}{\smthpara-\weakcnx} \cdot \text{dist}(0, \partial
  \env(z;y) ) + \varepsilon \norm{z-y}\\
&\le \frac{1}{2} \cdot \text{dist}(0, \partial
  \env(z;y) ) + \varepsilon \norm{z-y}. 
\end{align*}
The last inequality follows because of the assumption $\varepsilon \le
\frac{\smthpara-\weakcnx}{2}$. Rearranging the terms above, we get the
desired result. 
\end{proof}

These two lemmas together give us Theorem~\ref{thm:inner complex, kappa>rho}.

\begin{proof}[Proof of Theorem~\ref{thm:inner complex, kappa>rho}]
First, we prove that $z_T$ satisfies both adaptive stationary condition and
the descent condition. Recall, the point $y^0$ is defined to be the
prox or $y$ depending on if $\env(\cdot; y)$ is a composite form or
smooth, respectively (see statement of Theorem~\ref{thm:inner complex, kappa>rho}).
By Lemma \ref{prop: proximal} (or the remark following it), the starting $y^0$ satisfies 
 \begin{equation*} f_\smthpara(y^0; y) - f_{\smthpara}^*(y) \le \frac{\smthpara +
    L}{2} \norm{y^*-y}^2.
\end{equation*}
By the linear convergence assumption of $\mathcal{M}$ (see
\eqref{eq:criteria}) and the above equation, after $T:= T_\smthpara$ iterations initializing
from $y^0$, we have 
\begin{equation}
\begin{aligned}
\text{dist}^2(0, \partial \env(z_T;y) ) &\le A_\smthpara
  (1-\tau_{\smthpara})^T \left ( f_\smthpara(y^0; y) -
  f_{\smthpara}^*(y) \right )\\
&\le A_{\smthpara} e^{-T \cdot \tau_{\smthpara}} \left ( \env(y^0; y) -
 \env^*(y) \right )\\
&\le \frac{(\smthpara-\weakcnx)^2}{8 (L + \smthpara)} \cdot \frac{L +
  \smthpara}{2} \norm{y^*-y}^2\\
&\le \frac{(\smthpara-\weakcnx)^2}{16} \norm{y^*-y}^2. 
\end{aligned} \label{eq:something_important_1}
\end{equation}
Take the square root and apply Lemma~\ref{lem:E2} yields
$$ \text{dist}(0, \partial \env(z_T; y) ) \leq \frac{\smthpara-\weakcnx}{2}\norm{z_T-y} \leq \smthpara \norm{z_T-y},$$
which gives the adaptive stationary condition. Next, we show the descent condition. Let $v \in \partial \env(z_T;y)$ 
such that $\norm{v} \leq (\smthpara-\weakcnx) \norm{z_T-y}/2$, by the $(\smthpara-\weakcnx)$-strong
convexity of $\env(\cdot; y)$, we deduce 
\begin{align*}
 f_\smthpara(y; y) & \geq f_\smthpara(z_T; y)+ \langle v, y -z_T  \rangle + \frac{\smthpara-\weakcnx}{2} \norm{z_T-y}^2 \\
		   & \geq f_\smthpara(z_T; y)- \norm{v} \norm{y -z_T}  + \frac{\smthpara-\weakcnx}{2} \norm{z_T-y}^2 \\
		   & \geq f_\smthpara(z_T; y).
\end{align*}
This yields the descent condition which completes the proof for
$T$. The proof for $S_\smthpara$ is similar to $T_\smthpara$, so we
omit many of the details. In this case, we only need to show the
adaptive stationary condition. For convenience, we denote $S = S_{\smthpara}$. Following the same argument as in
Equation \eqref{eq:something_important_1} but with $S\log(\cnt+1)$
number of iterations, we
deduce
\[ \text{dist}^2(0, \partial \env(z_S; y) ) \le
  \frac{(\smthpara-\weakcnx)^2}{16 (\cnt +1)^2} \norm{y^*-y}^2.\]
By applying Lemma~\ref{lem:E2}, we obtain
$$ \text{dist}(0, \partial \env(z_S; y) ) \leq
\frac{(\smthpara-\weakcnx)}{2(k+1)}\norm{z_T-y} \leq \frac{\smthpara}{\cnt
  +1} \norm{z_S-y},$$
which proves the desired result for $z_S$. 
\end{proof}

Assuming Proposition~\ref{prop: terminates} and
Proposition~\ref{prop: inner_comp_cnx} hold as well as Lemma~\ref{lem:
  doubling_kappa}, we begin by providing the proof of Theorem~\ref{thm: inner_complexity}. 

\begin{proof}[Proof of Theorem~\ref{thm: inner_complexity}] We
  consider two cases: (i) the function $f$ is non-convex and (ii) the
  function $f$ is convex. 
   First, we consider the non-convex
  setting. To produce $\proxx_{\cnt}$, the method $\mathcal{M}$ is called 
   \begin{equation}
      T \log \left ( \tfrac{4L}{\smthparainit}
      \right ) / \log(2) \label{eq:complexityM}
   \end{equation} 
   number of times. This follows from Proposition~\ref{prop:
terminates} and Lemma~\ref{lem: doubling_kappa}. The reasoning is that
once $\smthpara > \weakcnx + L$, which only takes at most
$\log(4L/\smthparainit)$ number of increases of $\smthpara$ to reach, then the 
iterate $\proxx_{\cnt}$ satisfies the stopping criteria
\eqref{eq:prox_stop_criteria_new}. Each time we increase $\smthpara$ we
run $\mathcal{M}$ for $T$ iterations. Therefore, the total number of
iterations of $\mathcal{M}$ is given by multiplying $T$ with
$\log(4L/\smthparainit)$. 
To produce $\accx_{\cnt}$, the method
$\mathcal{M}$ is called $S\log(\cnt +1)$ number of times. (Note: the
proof of
Theorem~\ref{thm: computational_complex_result} does not need
$\accx_{\cnt}$ to satisfy \eqref{eq:accx_stop_criteria_new} in the
non-convex case).

Next, suppose the function $f$ is convex. As before, to produce
   $\proxx_{\cnt}$ the method $\mathcal{M}$ is called~(\ref{eq:complexityM}) times.
To produce
$\accx_{\cnt}$, the method $\mathcal{M}$ is called $S\log(\cnt+1)$
number of times. By Proposition~\ref{prop: inner_comp_cnx}, the
iterate $\accx_{\cnt}$ satisfies \eqref{eq:accx_stop_criteria_new}; a key ingredient in the proof of Theorem~\ref{thm:
  computational_complex_result}.
\end{proof}

\subsection{Inner complexity for $\proxx_{\cnt}$: proof of Proposition~\ref{prop: terminates}} \label{subsec:
  inner_comp_proxx}
Next, we supply the proof of Proposition~\ref{prop: terminates} which
shows that by choosing $\smthpara$ large enough,
Algorithm~\ref{alg:grad} terminates. 

\begin{proof}[Proof of Proposition~\ref{prop: terminates}] 
The idea is to apply
Theorem~\ref{thm:inner complex, kappa>rho}. Since the parameter
$A_{\smthpara}$ increases with $\smthpara$, then we upper bound it by
$A_{\smthpara_\cnt} \le A_{4L}$. Moreover, we have
$\smthpara - \weakcnx \ge \weakcnx + L - \weakcnx = L$. Lastly, since
$\tau_{\smthpara}$ is increasing in $\smthpara$, we know
$\frac{1}{\tau_{\smthpara}} \le \frac{1}{\tau_{L}}$. Plugging these
bound into Theorem~\ref{thm:inner complex, kappa>rho}, we see that for
any smoothing parameter $\smthpara$ satisfying $\weakcnx + L <
\smthpara < 4L$, we get the desired result. 
\end{proof}

Next, we compute the maximum number of times we must double
$\smthpara$ until $\smthpara > \weakcnx + L$.

 \begin{lemma}[Doubling $\smthpara$] \label{lem: doubling_kappa}
 	If we set $T$ and $S$ according to Theorem~\ref{thm: inner_complexity}, then
 	the doubling of $\smthparainit$ will terminate as soon as $\smthpara > \weakcnx + L$. Thus the number of times
 	$\smthparainit$ must be doubled in Algorithm~\ref{alg:grad} is at most
 	\[ \frac{\log \left (\frac{2 (\weakcnx + L)}{\smthparainit} \right
 		)}{\log(2)} \le \left \lceil \frac{\log \left (\frac{4L}{\smthparainit} \right
 		)}{\log(2)} \right \rceil. \]
 \end{lemma}

Since $\smthpara$ is
doubled (Algorithm~\ref{alg:grad}) and $T$ is chosen as in
Proposition~\ref{prop: terminates} , the maximum the value
$\smthpara$, $\smthpara_{\max}$, takes is $2(\weakcnx + L) \le 4L$.

\subsection{Inner complexity for $\accx_{\cnt}$: proof of Proposition~\ref{prop: inner_comp_cnx}} \label{subsec: inner_comp_accx}
In this section, we prove Proposition~\ref{prop: inner_comp_cnx}, an
inner complexity result for the iterates $\accx_{\cnt}$. Recall that the inner-complexity analysis for $\accx_{\cnt}$ is
important only when~$f$ is convex (see Section~\ref{sec:autoalgo}).
Therefore, we assume throughout this section that the function $f$ is
convex. We are now ready to prove Proposition~\ref{prop: inner_comp_cnx}. 

\begin{proof}[Proof of Proposition~\ref{prop: inner_comp_cnx}] The
  proof immediately follows from Theorem~\ref{thm:inner complex,
    kappa>rho} by setting $\smthpara = \smthpara_{\text{cvx}}$ and
  $\weakcnx = 0$ as the function $f$ is convex. 
\end{proof}

\begin{proof}[Proof of Theorem~\ref{thm:global_conv_known_rho}] 
Since $\weakcnx$ is known, we let $\kappa = 2 \weakcnx$ for $S$ and $\kappa = \smthparacvx$ for $T$ in Theorem~\ref{thm:inner complex, kappa>rho}. Therefore, the number of iterations is to produce $\proxx_\cnt$ is $S_{2 \weakcnx} \le S_{2L}$ and $\accx_\cnt$ is $T_{\smthparacvx}$. Combining this with Theorem~\ref{theo:outerloop-cvx-basic} for the outer complexity, the result is shown. 
\end{proof}

\bibliography{biblio}

\begin{thebibliography}{10}

\bibitem{katyusha}
Z.~{Allen-Zhu}.
\newblock Katyusha: The first direct acceleration of stochastic gradient
  methods.
\newblock In {\em Symposium on Theory of Computing (STOC)}, 2017.

\bibitem{natasha}
Z.~Allen-Zhu.
\newblock Natasha: Faster non-convex stochastic optimization via strongly
  non-convex parameter.
\newblock In {\em International conference on machine learning (ICML)}, 2017.

\bibitem{allen2016variance}
Z.~Allen-Zhu and E.~Hazan.
\newblock Variance reduction for faster non-convex optimization.
\newblock In {\em International conference on machine learning (ICML)}, 2016.

\bibitem{Acc_SVRG_1}
Z.~{Allen-Zhu} and Y.~Yuan.
\newblock {Improved SVRG for Non-Strongly-Convex or Sum-of-Non-Convex
  Objectives}.
\newblock In {\em International conference on machine learning (ICML)}, 2016.

\bibitem{fista}
A.~Beck and M.~Teboulle.
\newblock A fast iterative shrinkage-thresholding algorithm for linear inverse
  problems.
\newblock {\em SIAM Journal on Imaging Sciences}, 2(1):183--202, 2009.

\bibitem{bertsekas1999nonlinear}
D.~P. Bertsekas.
\newblock {\em Nonlinear programming}.
\newblock Athena scientific Belmont, 1999.

\bibitem{bertsekas:2015}
D.~P. Bertsekas.
\newblock {\em Convex Optimization Algorithms}.
\newblock Athena Scientific, 2015.

\bibitem{error_KL}
J.~Bolte, T.~P. Nguyen, J.~Peypouquet, and B.~Suter.
\newblock From error bounds to the complexity of first-order descent methods
  for convex functions.
\newblock {\em Mathematical Programming, Series A}, 165:471--507, 2016.

\bibitem{borwein:lewis:2006}
J.~M. Borwein and A.~S. Lewis.
\newblock {\em Convex analysis and nonlinear optimization: theory and
  examples}.
\newblock Springer Verlag, 2006.

\bibitem{GD_optimal}
Y.~{Carmon}, J.~C. {Duchi}, O.~{Hinder}, and A.~{Sidford}.
\newblock {Lower bounds for finding stationary points I}.
\newblock {\em preprint arXiv:1710.11606}, 2017.

\bibitem{CDHS}
Y.~Carmon, J.~C. Duchi, O.~Hinder, and A.~Sidford.
\newblock Accelerated methods for non-convex optimization.
\newblock {\em SIAM Journal on Optimization}, 28(2):1751--1772, 2018.

\bibitem{carmon2017convex}
Y.~Carmon, O.~Hinder, J.~C. Duchi, and A.~Sidford.
\newblock ``convex until proven guilty'': Dimension-free acceleration of
  gradient descent on non-convex functions.
\newblock In {\em International conference on machine learning (ICML)}, 2017.

\bibitem{Cartis2010}
C.~Cartis, N.~I.~M. Gould, and P.~L. Toint.
\newblock On the complexity of steepest descent, newton's and regularized
  newton's methods for nonconvex unconstrained optimization problems.
\newblock {\em SIAM Journal on Optimization}, 20(6):2833--2852, 2010.

\bibitem{Cartis2014}
C.~Cartis, N.I.M. Gould, and P.~L. Toint.
\newblock On the complexity of finding first-order critical points in
  constrained nonlinear optimization.
\newblock {\em Mathematical Programming, Series A}, 144:93--106, 2014.

\bibitem{prox_smooth_equiv_rock}
F.~H. Clarke, R.~J. Stern, and P.~R. Wolenski.
\newblock Proximal smoothness and the lower-{$C^2$} property.
\newblock {\em Journal of Convex Analysis}, 2(1-2):117--144, 1995.

\bibitem{saga}
A.~J. Defazio, F.~Bach, and S.~Lacoste{-}Julien.
\newblock {SAGA:} {A} fast incremental gradient method with support for
  non-strongly convex composite objectives.
\newblock In {\em Advances in Neural Information Processing Systems (NIPS)},
  2014.

\bibitem{accel_prox_comp}
D.~Drusvyatskiy and C.~Paquette.
\newblock Efficiency of minimizing compositions of convex functions and smooth
  maps.
\newblock {\em Mathematical Programming}, (Ser. A):1--56, 2018.

\bibitem{adaGrad}
J.~C. Duchi, E.~Hazan, and Y.~Singer.
\newblock Adaptive subgradient methods for online learning and stochastic
  optimization.
\newblock {\em Journal of Machine Learning Research (JMLR)}, 12:2121--2159,
  2011.

\bibitem{pos_reach}
H.~Federer.
\newblock Curvature measures.
\newblock {\em Transactions of the American Mathematical Society}, 93:418--491,
  1959.

\bibitem{GL1}
S.~Ghadimi and G.~Lan.
\newblock Accelerated gradient methods for nonconvex nonlinear and stochastic
  programming.
\newblock {\em Mathematical Programming}, 156(1-2, Ser. A):59--99, 2016.

\bibitem{GL2}
S.~Ghadimi, G.~Lan, and H.~Zhang.
\newblock Generalized uniformly optimal methods for nonlinear programming.
\newblock {\em preprint arXiv:1508.07384}, 2015.

\bibitem{guler:1991}
O.~G\"uler.
\newblock On the convergence of the proximal point algorithm for convex
  minimization.
\newblock {\em SIAM Journal on Control and Optimization}, 29(2):403--419, 1991.

\bibitem{htw:2015}
T.~Hastie, R.~Tibshirani, and M.~Wainwright.
\newblock {\em Statistical learning with sparsity: the {L}asso and
  generalizations}.
\newblock CRC Press, 2015.

\bibitem{AGD_saddle_points_Jordan}
C.~{Jin}, P.~{Netrapalli}, and M.~I. {Jordan}.
\newblock {Accelerated gradient descent escapes saddle points faster than
  gradient descent}.
\newblock In {\em Conference On Learning Theory (COLT)}, 2018.

\bibitem{JT_SVRG}
R.~Johnson and T.~Zhang.
\newblock Accelerating stochastic gradient descent using predictive variance
  reduction.
\newblock In {\em Advances in Neural Information Processing Systems (NIPS)},
  2013.

\bibitem{adam}
Diederik~P Kingma and Jimmy Ba.
\newblock Adam: A method for stochastic optimization.
\newblock {\em International Conference on Learning Representations (ICLR)},
  2015.

\bibitem{conjugategradient}
G.~Lan and Y.~Zhou.
\newblock An optimal randomized incremental gradient method.
\newblock {\em Mathematical Programming, Series A}, pages 1--38, 2017.

\bibitem{SCSG}
L.~{Lei} and M.~I. {Jordan}.
\newblock {Less than a single pass: stochastically controlled stochastic
  gradient method}.
\newblock In {\em Conference on Artificial Intelligence and Statistics
  (AISTATS)}, 2017.

\bibitem{SCSG_noncvx}
L.~{Lei}, C.~{Ju}, J.~{Chen}, and M.~I. {Jordan}.
\newblock Non-convex finite-sum optimization via { SCSG } methods.
\newblock In {\em Advances in Neural Information Processing Systems (NIPS)},
  2017.

\bibitem{NIPS2015_5728}
H.~Li and Z.~Lin.
\newblock Accelerated proximal gradient methods for nonconvex programming.
\newblock In {\em Advances in Neural Information Processing Systems (NIPS)}.
  2015.

\bibitem{catalyst}
H.~Lin, J.~Mairal, and Z.~Harchaoui.
\newblock A universal catalyst for first-order optimization.
\newblock In {\em Advances in Neural Information Processing Systems (NIPS)},
  2015.

\bibitem{catalyst_new}
H.~{Lin}, J.~{Mairal}, and Z.~{Harchaoui}.
\newblock {Catalyst Acceleration for First-order Convex Optimization: from
  Theory to Practice}.
\newblock {\em Journal of Machine Learning Research (JMLR)}, 18:1--54, 2018.

\bibitem{miso}
J.~Mairal.
\newblock Incremental majorization-minimization optimization with application
  to large-scale machine learning.
\newblock {\em SIAM Journal on Optimization}, 25(2):829--855, 2015.

\bibitem{mairal2014sparse}
J.~Mairal, F.~Bach, and J.~Ponce.
\newblock Sparse modeling for image and vision processing.
\newblock {\em Foundations and Trends in Computer Graphics and Vision},
  8(2-3):85--283, 2014.

\bibitem{mairal2010}
J.~Mairal, F.~Bach, J.~Ponce, and G.~Sapiro.
\newblock Online learning for matrix factorization and sparse coding.
\newblock {\em Journal of Machine Learning Research (JMLR)}, 11:19--60, 2010.

\bibitem{nesterov1983}
Y.~Nesterov.
\newblock A method of solving a convex programming problem with convergence
  rate {$O$(1/$k^2$)}.
\newblock {\em Soviet Mathematics Doklady}, 27(2):372--376, 1983.

\bibitem{nesterov}
Y.~Nesterov.
\newblock {\em Introductory lectures on convex optimization: a basic course}.
\newblock Springer, 2004.

\bibitem{Nesterov_CD}
Y.~Nesterov.
\newblock Efficiency of coordinate descent methods on huge-scale optimization
  problems.
\newblock {\em {SIAM} Journal on Optimization}, 22(2):341--362, 2012.

\bibitem{nest_optima}
Y.~Nesterov.
\newblock How to make the gradients small.
\newblock {\em {OPTIMA}, MPS Newsletter}, (88):10--11, 2012.

\bibitem{nesterov2013gradient}
Y.~Nesterov.
\newblock Gradient methods for minimizing composite functions.
\newblock {\em Mathematical Programming}, 140(1):125--161, 2013.

\bibitem{Schmidt_CD}
J.~Nutini, M.~Schmidt, I.~H. Laradji, M.~P. Friedlander, and H.~A. Koepke.
\newblock Coordinate descent converges faster with the gauss-southwell rule
  than random selection.
\newblock In {\em Proc. International Conference on Machine Learning (ICML)},
  2015.

\bibitem{AG_near_crit_pts}
M.~{O'Neill} and S.~J. {Wright}.
\newblock {Behavior of accelerated gradient methods near critical points of
  nonconvex problems}.
\newblock {\em Mathematical Programming}, (Ser. B):1--25, 2018.

\bibitem{Parikh13}
N.~Parikh and S.P. Boyd.
\newblock Proximal algorithms.
\newblock {\em Foundations and Trends in Optimization}, 1(3):123--231, 2014.

\bibitem{prox_reg_var_anal}
R.~A. Poliquin and R.~T. Rockafellar.
\newblock Prox-regular functions in variational analysis.
\newblock {\em Transactions of the American Mathematical Society},
  348(5):1805--1838, 1996.

\bibitem{reddi2016stochastic}
S.~J. Reddi, A.~Hefny, S.~Sra, B.~Poczos, and A.~Smola.
\newblock Stochastic variance reduction for nonconvex optimization.
\newblock In {\em International conference on machine learning (ICML)}, 2016.

\bibitem{reddi2016proximal}
S.~J. Reddi, S.~Sra, B.~Poczos, and A.~J. Smola.
\newblock Proximal stochastic methods for nonsmooth nonconvex finite-sum
  optimization.
\newblock In {\em Advances in Neural Information Processing Systems (NIPS)},
  2016.

\bibitem{Richtarik_CD}
P.~Richtarik and M.~Takac.
\newblock Iteration complexity of randomized block-coordinate descent methods
  for minimizing a composite function.
\newblock {\em Mathematical Programming}, 144(1-2):1--38, 2014.

\bibitem{rock_subsmooth}
R.~T. Rockafellar.
\newblock Favorable classes of {L}ipschitz-continuous functions in subgradient
  optimization.
\newblock In {\em Progress in nondifferentiable optimization}, volume~8 of {\em
  IIASA Collaborative Proc. Ser. CP-82}, pages 125--143. Internat. Inst. Appl.
  Systems Anal., Laxenburg, 1982.

\bibitem{rock_wets}
R.~T. Rockafellar and R.~J.-B. Wets.
\newblock {\em Variational analysis}, volume 317 of {\em Grundlehren der
  Mathematischen Wissenschaften [Fundamental Principles of Mathematical
  Sciences]}.
\newblock Springer-Verlag, Berlin, 1998.

\bibitem{sag}
M.~Schmidt, N.~Le~Roux, and F.~Bach.
\newblock Minimizing finite sums with the stochastic average gradient.
\newblock {\em Mathematical Programming}, 162(1):83--112, 2017.

\bibitem{Acc_SVRG_2}
S.~{Shalev-Shwartz}.
\newblock {SDCA without Duality, Regularization, and Individual Convexity}.
\newblock In {\em International conference on machine learning (ICML)}, 2016.

\bibitem{p_tseng}
P.~Tseng.
\newblock On accelerated proximal gradient methods for convex-concave
  optimization.
\newblock Technical report, 2008.

\bibitem{woodworth:srebro:2016}
B.~E. Woodworth and N.~Srebro.
\newblock Tight complexity bounds for optimizing composite objectives.
\newblock In {\em Advances in Neural Information Processing Systems (NIPS)}.
  2016.

\bibitem{Wright_CD}
S.~J. Wright.
\newblock Coordinate descent algorithms.
\newblock {\em Mathematical Programming}, 151(1):3--34, June 2015.

\bibitem{proxsvrg}
L.~Xiao and T.~Zhang.
\newblock A proximal stochastic gradient method with progressive variance
  reduction.
\newblock {\em {SIAM} Journal on Optimization}, 24(4):2057--2075, 2014.

\bibitem{zou2005regularization}
H.~Zou and T.~Hastie.
\newblock Regularization and variable selection via the elastic net.
\newblock {\em Journal of the Royal Statistical Society: Series B (Statistical
  Methodology)}, 67(2):301--320, 2005.

\end{thebibliography}
\bibliographystyle{plain}

\end{document}